\documentclass{article}

    \usepackage[preprint]{neurips_2021}

\usepackage[utf8]{inputenc} 
\usepackage[T1]{fontenc}    
\usepackage{url}            
\usepackage{booktabs}       
\usepackage{amsfonts}       
\usepackage{nicefrac}       
\usepackage{microtype}      
\usepackage{xcolor}         
\usepackage{dsfont}
\usepackage{chemformula}   

\usepackage{color}
\definecolor{cgreen}{rgb}{0.2,0.6,0.2}
\definecolor{darkred}{rgb}{0.4,0.0,0.0}
\definecolor{darkgreen}{rgb}{0.0,0.4,0.0}
\definecolor{darkblue}{rgb}{0.0,0.0,0.4}
\usepackage[colorlinks=true, citecolor=darkblue]{hyperref}
\hypersetup{
    linkbordercolor = {darkred},
    urlcolor = {darkblue}
}

\usepackage{graphicx}
\usepackage{subcaption}
\usepackage{wrapfig}

\usepackage{multirow}
\usepackage{makecell}

\usepackage[ruled]{algorithm2e}

\usepackage{amsmath}
\usepackage{amssymb}
\usepackage{amsthm}

\usepackage{stmaryrd,scalerel}
\usepackage{nicefrac}

\newcommand{\vx}{\boldsymbol{x}}

\newcommand{\vz}{\boldsymbol{z}}

\newcommand{\vK}{\boldsymbol{K}}

\newcommand{\vtheta}{\boldsymbol{\theta}}

\newcommand{\cA}{\mathcal{A}}

\newcommand{\cF}{\mathcal{F}}
\newcommand{\cG}{\mathcal{G}}
\newcommand{\cH}{\mathcal{H}}

\newcommand{\cL}{\mathcal{L}}

\newcommand{\cO}{\mathcal{O}}

\newcommand{\bE}{\mathbb{E}}

\DeclareMathOperator*{\argmax}{arg\,max}
\DeclareMathOperator*{\argmin}{arg\,min}

\newcommand{\RR}{\mathds{R}}
\newcommand{\Tb}{\mathbf{T}}

\newcommand{\Xsf}{\mathsf{X}}

\newcommand{\Zsf}{\mathsf{Z}}

\newcommand{\Tbpush}{\Tb_{\#}}

\newcommand{\zv}{\mathbf{z}}

\newcommand{\eg}{{e.g.}\xspace}
\newcommand{\ie}{{i.e.}\xspace}
\newcommand{\iid}{{i.i.d.}\xspace}
\newcommand{\cf}{{cf.}\xspace}

\newcommand*\rmd{\mathop{}\!\mathrm{d}}

\newtheorem{lemma}{Lemma}

\newtheorem{example}{Example}
\newtheorem{proposition}{Proposition}

\newcommand{\rmR}{\mathrm{R}}
\newcommand{\rmE}{\mathrm{E}}
\usepackage{bbm}
\usepackage{cleveref}

\title{Learning Equivariant Energy Based Models \\ with Equivariant Stein Variational Gradient Descent}

\author{%
  Priyank Jaini$^{*}$ \\
  Bosch-Delta Lab\\
  University of Amsterdam \\
   \And
  Lars Holdijk$^{*}$ \\
  Bosch-Delta Lab \\
  University of Amsterdam \\
  \And
  Max Welling \\
  Bosch-Delta Lab \\
  University of Amsterdam
}

\begin{document}

\maketitle

\begin{abstract}
We focus on the problem of efficient sampling and learning of probability densities by incorporating symmetries in probabilistic models. We first introduce \emph{Equivariant Stein Variational Gradient Descent} algorithm -- an equivariant sampling method based on Stein's identity for sampling from densities with symmetries. Equivariant SVGD explicitly incorporates symmetry information in a density through equivariant kernels which makes the resultant sampler efficient both in terms of sample complexity and the quality of generated samples. Subsequently, we define \emph{equivariant} energy based models to model invariant densities that are learned using contrastive divergence. By utilizing our equivariant SVGD for training equivariant EBMs, we propose new ways of improving and scaling up training of energy based models. We apply these equivariant energy models for modelling joint densities in regression and classification tasks for image datasets, many-body particle systems and molecular structure generation.  
\end{abstract}

\section{Introduction}
\label{sec:intro}
Many real-world observations comprise symmetries and admit probabilistic models that are invariant to such symmetry transformations. Naturally, overlooking these inductive biases while encoding such domains will lead to models with inferior performance capabilities. In this paper, we focus on the problem of efficient sampling and learning of equivariant probability densities by incorporating symmetries in probabilistic models.

We accomplish this by first proposing \emph{equivariant Stein variational descent algorithm} in \S\ref{sec:esvgd} for sampling from invariant densities. Stein Variational Gradient Descent (SVGD) is a kernel-based inference method that constructs a set of particles iteratively along an optimal gradient path in an RKHS to approximate and sample from a target distribution. We extend SVGD for invariant densities by considering equivariant kernel functions that evolve the set of particles such that the density at each time-step is invariant to the same symmetry transformations as encoded in the kernel. We demonstrate that equivariant SVGD is more sample efficient, produces a more diverse set of samples, and is more robust compared to regular SVGD when sampling from invariant densities.

Subsequently, in \S\ref{sec:ebm}, we build \emph{equivariant} Energy Based Models EBMs for learning invariant densities given access to \iid data by leveraging the tremendous recent advances in geometric deep learning where the energy function is equivariant neural network. We train these equivariant EBMs through contrastive divergence by generating samples using equivariant SVGD. We show that incorporating the symmetries present in the data into the energy model as well as the sampler provides an efficient learning paradigm to train equivariant EBMs that generalize well beyond training data.

We empirically demonstrate the performance of equivariant EBMs using equivariant SVGD in 
\S\ref{sec:exp}. We consider real-world applications comprising of problems from many-body particle systems, molecular structure generation and, classification and generation for image datasets. 
\section{Preliminaries and Setup}
\label{sec:prelim}

In this section we set-up our main problem, introduce key definitions and notations and formulate an approach to incorporate symmetries in particle variational inference optimization methods through Stein variational gradient descent. Along the way, we also discuss directly related work and relegate a detailed discussion on previous work to \Cref{app:prev}. 

Let $\cG$ be a group acting on $\RR^d$ through a representation $\mathrm{R} : \cG \to \mathrm{GL}(d)$ where $\mathrm{GL}(d)$ is the general linear group on $\RR^d$, such that $\forall g \in \cG$, $g \to \rmR_g$. Given a target random variable $\Xsf \subseteq \RR^d$ with density $\pi$, we say that $\pi$ is $\cG$-invariant if $\forall g \in \cG$ and $\vx \in \RR^d$, $\pi(\rmR_g\vx) = \pi(\vx)$. Additionally, a function $f(\cdot)$ is $\cG$-equivariant if $\forall g \in \cG$ and $\vx \in \RR^d$, $f(\rmR_g\vx) = \rmR_g f(\vx)$. We denote with $\cO(\vx)$ the orbit of an element $\vx \in \Xsf$ defined as $\cO(\vx) := \{ \vx' : \vx' = \rmR_g\vx, \forall g \in \cG \}$. We call $\pi_{|\cG}$ the factorized density of a $\cG$-invariant density $\pi$ where $\pi_{|\cG}$ has support on the set 
$\Xsf_{|\cG}$
where the elements of $\Xsf_{\cG}$ are indexing the orbits \ie if $\vx, \tilde{\vx} \in \Xsf_{\cG}$ then $\vx \neq \rmR_g \tilde{\vx}, \forall g \in \cG$. In this paper, we are interested to incorporate \emph{inductive biases} given by symmetry groups to develop efficient sampling and learning paradigms for generative modelling. Precisely, we consider the following problems: 

\paragraph{(i) Equivariant Learning:} Given access to an \iid samples $\lbag \vx_1, \ldots, \vx_n \rbag\sim \pi$ from a $\cG$-invariant density $\pi$, we want to approximate $\pi$. \cite{rezende2019equivariant} and \cite{kohler2020equivariant} addressed this by learning an equivariant normalizing flow \citep{TabakVE10, TabakTurner13, RezendeMohamed15} that transforms a simple latent $\cG$-invariant density $q_0$ to the target density $\pi$ through a series of $\cG-$equivariant diffeomorphic transformations $\Tb = (\Tb_1, \Tb_2,\cdots, \Tb_k)$ \ie $\pi := \Tbpush q_0$. They achieved this by proving (\cf \cite[Theorem 1]{kohler2020equivariant}, \cite[Lemma1]{rezende2019equivariant}) that if $q_0$ is a $\cG$-invariant density in $\RR^d$, $\cF$ is a proper sub-group of $\cG$ \ie $\cF < \cG$, and $\Tb$ is an $\cF$-equivariant diffeomorphic transformation, then $\pi := \Tbpush q_0$ is $\cF$-invariant. However, a major drawback of this formalism is that it requires $\Tb$ to not only be a $\cF$-equivariant diffeomorphism, but computation of the inverse and Jacobian must be cheap as well. This is problematic in practice. 

\cite{kohler2020equivariant} overcame this issue by using continuous normalizing flows \citep{grathwohl2018ffjord} that define a dynamical system through a time-dependent Lipschitz velocity field $\Psi : \RR^d \times \RR_{+} \to \RR^d$ with the following system of ordinary differential equations(ODEs):
\begin{align}
    \frac{\rmd \vx(t)}{\rmd t} = \Psi(\vx(t), t), \qquad \vx(0) = \vz
\end{align}
This allows to define a bijective function $\Tb_{\Psi, t}(\vz) := \vx(0) + \int_{0}^t \Psi(\vx(t), t) \rmd t$ which leads to a push-forward density $q_t$ at each time-step $t$ satisfying $\frac{\rmd \log q_t}{\rmd t} = - \mathsf{div}\big( \Psi(\vx(t), t)\big)$, which implies to the following important result:
\begin{lemma}[{\cite[Theorem 2]{kohler2020equivariant}}]
\label{lemma:1}
Let $\Psi$ be an $\cF$-equivariant vector-field on $\RR^d$. Then, the transformation $\Tb_{\Psi, t}(\vz) := \vx(0) + \int_{0}^t \Psi(\vx(t), t) \rmd t$ is $\cF$-equivariant $\forall t \in \RR_+$. Furthermore, the push-forward $q_t := \Tb_{\Psi, t, \#}q_0$ is $\cF$-invariant $\forall t$, if $q_0$ is $\cG$-invariant and $\cF < \cG$.
\vspace{-0.1cm}
\end{lemma}

\Cref{lemma:1} conveniently provides a framework to transform any $\cG$-invariant density to an $\cF$-invariant density along a path in which each intermediate density is also $\cF$-invariant. However, equivariant normalizing flows cannot be used directly to generate samples when given access to an invariant density $\pi$ since they  require \iid samples from $\pi$ to train the flow\footnote{In \Cref{app:sampling}, we discuss a way to use equivariant normalizing flow for direct sampling given access to $\pi$.}.
\paragraph{(ii) Equivariant Sampling:}In this paper, we are also interested in solving the inference problem \ie we are interested in evaluating $\bE_{\pi}[f]$, the expectation of $f$ when given access to a $\cG$-invariant density $\pi$ which typically involves generating samples $\lbag  \vx_1, \vx_2, \cdots, \vx_n\rbag \sim \pi$. Intuitively, sampling from a $\cG$-invariant density can be reduced to sampling from its corresponding factorized distribution $\pi_{| \cG}$. This is because any set of samples $\{\tilde{\vx}_i\}_{i=1}^n \sim \pi_{|\cG}$ can be used to get samples representing $\pi$ by applying group actions from $\cG$ to $\{\tilde{\vx}_i\}_{i=1}^n$. Indeed, sampling methods like Markov Chain Monte Carlo (MCMC) \citep{brooks2011handbook} or Hybrid Monte Carlo (HMC) \cite{neal2011mcmc} and their variants, in principle, can use this paradigm to sample from an invariant density $\pi$. However, MCMC methods for approximate posterior sampling are often slow and it still remains challenging to scale them up to big data settings. An alternate to MCMC methods for approximate posterior sampling is Stein Variational Gradient Descent (SVGD) \citep{liu2016stein} which is a particle optimization variational inference method that combines the paradigms of sampling and variational inference for Bayesian inference problems.  

In SVGD, a set of $n$ particles $\{\vx_i\}_{i=1}^n \in \Xsf \subseteq \RR^d$ are evolved following a dynamical system to approximate the target (posterior) density $\pi(\vx) \propto \mathsf{exp}\big(-\rmE(\vx)\big)$ where $\rmE(\cdot)$ is the energy function. This is achieved in a series of $T$ discrete steps that transform the set of particles $\{\vx_i^0\}_{i=1}^n \sim q_0(\vx)$ sampled from a base distribution $q_0$ (\eg Gaussian) at $t=0$ using the map $\vx^{t} = \Tb(\vx):= \vx^{t-1} + \varepsilon\cdot\Psi(\vx^{t-1})$ where $\varepsilon$ is the step size and $\Psi(\cdot)$ is a vector field. $\Psi(\cdot)$ is chosen such that it maximally decreases the KL divergence between the push-forward density $q_t(\vx) = \Tbpush q_{t-1}(\vx)$ and the target $\pi(\vx)$.

If $\Psi$ is restricted to the unit ball of an RKHS $\cH^d_k$ with positive definite kernel $k: \RR^d \times \RR^d \to \RR$, the direction of steepest descent that maximizes the negative gradient of the KL divergence is given by:
\begin{align}
\label{eq:phi}
    \Psi^*_{q,\pi}(\vx):= \argmax_{\Psi \in \cH^d_k} -\nabla_{\varepsilon} \mathsf{KL}\big(q || \pi\big)|_{\varepsilon \to 0} = \bE_{\vx\sim q}[\mathsf{trace}(\cA_{\pi}\Psi(\vx))],
\end{align}
where $\cA_{\pi}\Psi(\vx) = \nabla_{\vx}\log \pi(\vx)\Psi(\vx)^\top + \nabla_{\vx}\Psi(\vx)$ is the Stein operator. Thus, an iterative paradigm can be easily implemented wherein a set of particles $\{\vx^0_1, \vx^0_2, \cdots, \vx^0_n\} \sim q_0$ are transformed to approximate the target density $\pi(\cdot)$ using the optimal update $\Psi^*_{q,\pi}(\vx)\propto \bE_{\vx'\sim q}[\cA_{\pi}k(\vx',\vx)]$. Since $\cA_{\pi}\Psi(\vx) = \nabla_{\vx}[\pi(\vx) \Psi(\vx)]/\pi(\vx)$ we have that $\bE_{\vx\sim\pi}[\cA_{\pi}\Psi(\vx)]=0$ for any $\Psi$ implying convergence when $q=\pi$. Replacing the expectation in the update with a Monte Carlo sum over the current set of particles that represent $q$ we get:
\begin{align}
    \label{eq:svgd}
    \vx^{t+1}_i \leftarrow \vx^t_i + \varepsilon \tilde{\Psi}^*(\vx^t_i), ~~ \text{where,} ~~ \tilde{\Psi}^*(\vx^t_i) := \frac{1}{n} \sum_{j=1}^n \big( \underbrace{\nabla_{\vx^t_j} k(\vx_j^t, \vx_i)}_{\text{repulsive force}} - \underbrace{k(\vx_j^t, \vx_i)\cdot \nabla_{\vx_j^t} \rmE(\vx^t_j)}_{\text{attractive force}} \big)
\end{align}

Stein variational gradient descent intuitively encourages diversity among particles by exploring different modes in the target distribution $\pi$ through a combination of the second term in \Cref{eq:svgd} which attracts the particles to high density regions using the score function and the repulsive force (first term) which ensures the particles do not collapse together. In the continuous time limit, as $\varepsilon \to 0$, \Cref{eq:svgd} results in a system of ordinary differential equations describing the evolution of particles $\{\vx^0_1, \vx^0_2, \cdots, \vx^0_n\}$ according to $\frac{\rmd \vx}{\rmd t} =\tilde{\Psi}^*(\vx)$. 

Furthermore, as shown in \cite{wang2019stein}, geometric information using pre-conditioning matrices can be incorporated in \Cref{eq:svgd} by using matrix valued kernels (\cf Definition 2.3 \citep{reisert2007learning}) leading to the following generalized form of SVGD \citep{wang2019stein}:
\begin{align}
    \label{eq:mvk}
    \vx^{t+1}_i \leftarrow \vx^t_i + \frac{\varepsilon}{n} \sum_{j=1}^n \big( \nabla_{\vx^t_j} \vK(\vx_j^t, \vx_i) - \vK(\vx_j^t, \vx_i)\cdot \nabla_{\vx_j^t} \rmE(\vx^t_j) \big),
\end{align}
where $\vK(\vx, \vx')$ is a matrix valued kernel. Matrix-valued SVGD allows to flexibly incorporate preconditioning matrices yielding acceleration in the exploration of the given probability landscape. 

SVGD has gained a lot of attention over the past few years as a flexible and scalable alternative to MCMC methods for approximate Bayesian posterior sampling. Further, it is more particle efficient since it generates diverse particles due to the deterministic repulsive force induced by kernels instead of Monte Carlo randomness. A natural question to ask is: \emph{Can we incorporate symmetry information into SVGD for more efficient sampling from invariant densities?} We answer this in the affirmative in the next section by proposing \emph{equivariant Stein variational gradient descent} algorithm for sampling from invariant densities.

\vspace{-0.4em}
\section{Equivariant Stein Variational Gradient Descent}
\label{sec:esvgd}
\vspace{-0.4em}

We begin this section by presenting the main result of this section by introducing \emph{equivariant} Stein variational gradient descent (E-SVGD) by utilizing \Cref{lemma:1} and Equations (\ref{eq:svgd}) and (\ref{eq:mvk}).

\begin{proposition}
\label{prop:esvgd}
Let $\pi$ be a $\cG$-invariant density and $\lbag \vx^0_1, \vx^0_2, \cdots, \vx^0_n \rbag \sim q_0$ be a set of particles at $t=0$ with $q_0$ being $\cF$-invariant where $\cF > \cG$. Then, the iterative update given by \Cref{eq:svgd} is $\cG$-equivariant and the density $q_{t+1}$ defined by it at time $t+1$ is $\cG$-invariant if the positive definite kernel $k(\cdot, \cdot)$ is $\cG$-invariant. The same holds for \Cref{eq:mvk} if $\vK(\cdot, \cdot)$ is $\cG$-equivariant.
\end{proposition}

\begin{proof}
Since the initial distribution $q_0$ is $\cF$-invariant, following \Cref{lemma:1} the update in \Cref{eq:svgd} is $\cG$-equivariant if $\Psi$ is $\cG$-equivariant. If $k(\cdot, \cdot)$ is $\cG$-invariant then $\nabla_{\vx} k(\cdot, \vx)$ is $\cG$-equivariant. Furthermore, since $\pi = \mathsf{exp}\big(-\rmE(\vx)\big)$ is $\cG$-invariant, $\nabla_{\vx}\rmE(\vx)$ is also $\cG$-equivariant. Thus, both the terms for $\Psi$ are $\cG$-equivariant if $k(\cdot, \cdot)$ is $\cG$-equivariant making the update in \Cref{eq:svgd} $\cG$-equivariant. The result follows similarly for \Cref{eq:mvk} when $\vK(\cdot, \cdot)$ is $\cG$-equivariant.
\end{proof}

Following \Cref{prop:esvgd}, we call the updates in Equations (\ref{eq:svgd}) \& (\ref{eq:mvk}) \emph{equivariant} Stein variational gradient descent when the kernel $k(\cdot, \cdot)$ (and $\vK(\cdot, \cdot)$ respectively) is invariant (equivariant) and the initial set of particles $\lbag \vx^0_1, \cdots, \vx^0_n\rbag$ are sampled from an invariant density $q_0$. Thus, all that is required to sample from a $\cG$-invariant density $\pi$ using equivariant SVGD is to construct a positive definite kernel that is $\cG$-equivariant. Let us next give a few examples for constructing in- and equivariant positive definite kernels.

\begin{example}[Invariant scalar kernel]
\label{exp:kernel}
Let $\cG$ be a finite group acting on $\RR^d$ with representation $\rmR$ such that $\forall g \in \cG, g\to \rmR_g$. Then, 
\begin{align*}
    k_{\cG}(\vx, \vx') = \sum_{\vx \in \cO(\vx)} \sum_{\vx' \in \cO(\vx')} k(\vx, \vx')
\end{align*}
is $\cG$-invariant where $k(\cdot, \cdot)$ is some positive-definite kernel. While this provides a general method to construct invariant kernels for finite groups, the double summation can be computationally expensive. In practice, usually simple kernels like RBF kernel (for rotation symmetries) or uniform kernel suffice as more practical alternatives.
\end{example}
\Cref{exp:kernel} is only restricted to finite groups and does not directly apply to continuous symmetry groups. We can construct kernels for continuous groups following \Cref{exp:kernel} by either using a Monte Carlo approximation or using a transformation that performs computations in the factorized space $\Xsf_{|\cG}$ as we show in the next example.  

\begin{example}[Continuous Symmetry Groups]
\label{exp:cont}
Let $\pi(\vx)$ be $\mathsf{SO}(2)$-invariant (\cf \Cref{fig:sinv-circ} for an example) where $\vx \in \RR^2$ \ie $\cO(\vx) := \{ \vx' : \|\vx\| = \|\vx'\|  \}$. We can either construct an invariant kernel for sampling from $\pi$ using a Monte Carlo approximation by sampling random rotations on a unit sphere \ie
\begin{align*}
    k_{\cG}(\vx, \vx') = \sum_{i, j=1}^n k(g_j\vx, g_i \vx'), \qquad g_i, g_j \in \cG, ~\forall (i, j) \in [n]\times [n]
\end{align*}
Or alternately, we can consider the function $\Phi_{\cG} : \RR^2 \to \RR$ such that $\Phi_{\cG}(\vx) = \|\vx\|$. Then, $\Phi_{\cG}(\vx)$ is $\mathsf{SO}$(2) invariant since $\Phi_{\cG}(g\vx) = \Phi_{\cG}(\vx), \forall g \in \cG$. Thus, we can now use the following kernel 
\begin{align*}
    k_{\cG}(\vx, \vx') = k\big(\Phi_{\cG}(\vx), \Phi_{\cG}(\vx')\big)
\end{align*}

\end{example}
Examples (\ref{exp:kernel}) and (\ref{exp:cont}) are both invariant scalar kernels. Let us next give an example of an equivariant matrix valued kernel for matrix valued SVGD (\cf \Cref{eq:mvk}).
\begin{wrapfigure}[12]{R}{0.5\textwidth}
  \begin{center}
    \includegraphics[width=0.48\textwidth]{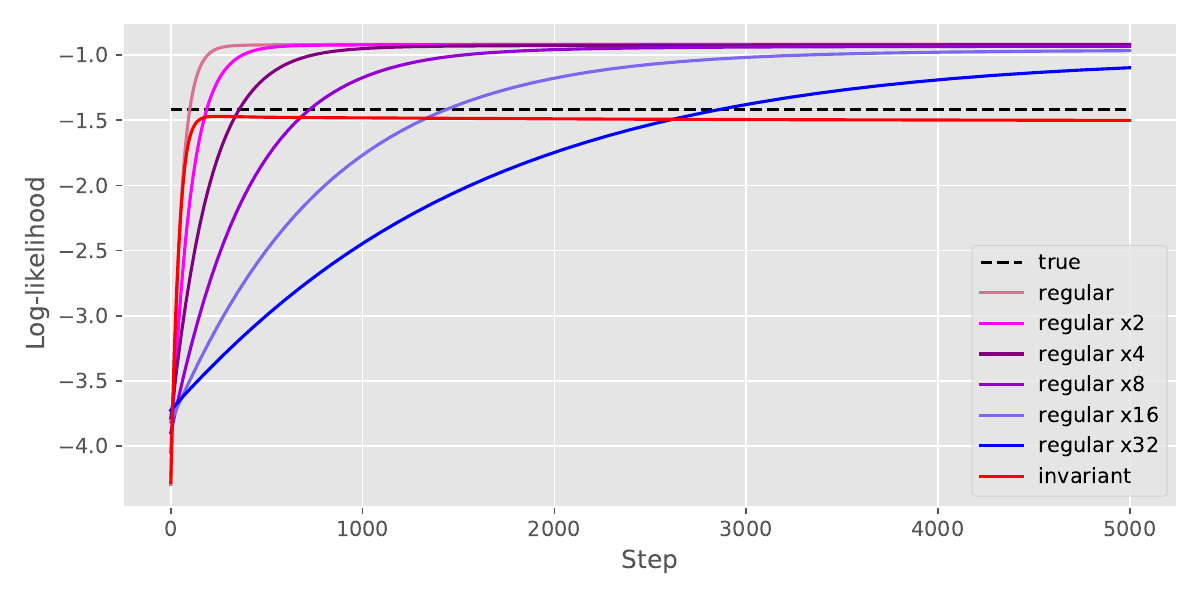}
  \end{center}
  \vspace{-0.5cm}
  \caption{Sample efficiency}
  \label{fig:sample}
\end{wrapfigure}

\begin{example}[Equivariant Matrix-Valued Kernels, \cite{reisert2007learning}]
Examples \ref{exp:kernel} and \ref{exp:cont} define an invariant scalar kernel. Following \cite{reisert2007learning}, we can also construct a $\cG$-equivariant matrix-valued kernel for the generalized update as in \Cref{eq:mvk} by defining:
\begin{align*}
    \vK(\vx, \vx') = \int_{\cG} k(\vx, g\vx') \rmR_g \rmd g
\end{align*}
where $\rmR_{g}$ is a group representation and $k(\cdot, \cdot)$ is a scalar symmetric, $\cG$-invariant function. It is easy to check that $\vK(\vx, \vx')$ is equivariant in the first argument and anti-equivariant in the second argument, leading to an equivariant $\vK(\vx, \vx')$ (\cf Proposition 2.2 \cite{reisert2007learning}).
\end{example}

\paragraph{Advantages of Equivariant Sampling:} As we discussed briefly in \Cref{sec:prelim}, SVGD works by evolving a set of particles using a dynamical system through a combination of attractive and repulsive forces among the particles that are governed by the inter-particle distance. Thus, a particle exerts these forces in a restricted neighbourhood around it. Equivariant SVGD, on the other hand, is able to model \emph{long-range interactions} among particles due to the use of equivariant kernel. Intuitively, any point $\vx$ is able to exert these forces on any other point $\vx'$ in equivariant SVGD if $\vx'$ is in the neighbourhood of any point in the orbit $\cO(\vx)$ of $\vx$. This is because for any point $\vx'$ the repulsive and attractive force terms are the same in Equations (\ref{eq:svgd}) and (\ref{eq:mvk}) for all points that are in the orbit $\cO(\vx)$. This ability to capture long-range interactions by equivariant Stein variational gradient descent subsequently makes it more efficient in sample complexity and running time with better sample quality, and makes it more robust to different initial configurations of the particles compared to vanilla SVGD. We illustrate these next with the help of the following examples: 

\textbf{(i) $C_4$-Gaussians} (\cf \Cref{fig:inv-gau} and \ref{fig:reg-gau}): This example consists of four Gaussians invariant to $C_4$ symmetry group. In this case, the group factorized distribution $\pi_{|C_4}$ is Gaussian with the original $C_4$-invariant density obtained by rotating $\pi_{|C_4}$ through the set $\{0^{\circ}, 90^{\circ}, 180^{\circ}, 270^{\circ} \}$. In \Cref{fig:inv-gau}, the first column shows the samples generated by equivariant SVGD, the second column is the projection of these samples on the group factorized space $\Xsf_{|C_4}$ and, the third column shows the samples obtained by rotating the original samples through the $C_4$-symmetry group. \Cref{fig:reg-gau} shows a similar setup for vanilla SVGD.  

\begin{wrapfigure}[19]{R}{0.3\textwidth}
  \begin{center}
    \includegraphics[width=0.3\textwidth]{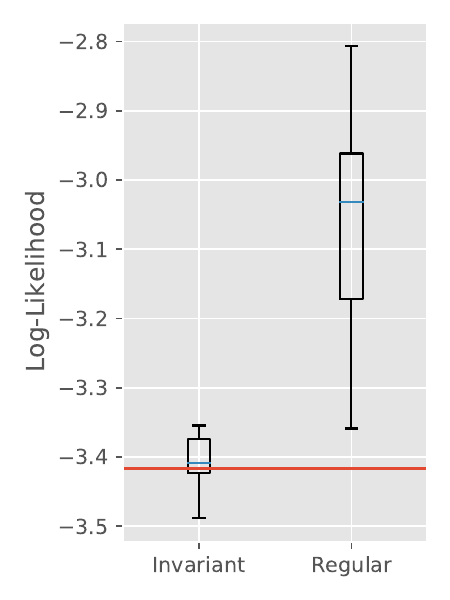}
  \end{center}
  \caption{Robustness}
  \label{fig:init}
\end{wrapfigure}

\textbf{(ii) Concentric Circles} (\cf \Cref{fig:sinv-circ} and \ref{fig:reg-cir}): This example comprises of two concentric circles invariant to the $\mathsf{SO}(2)$ symmetry group. In this case, the group factorized space is a union of two disconnected lines with length equal to the thickness of the circles. In \Cref{fig:sinv-circ}, the first column shows the samples generated by equivariant SVGD and, the second column is the projection of these samples on the group factorized space $\Xsf_{|\mathsf{SO}(2)}$. \Cref{fig:reg-cir} shows a similar setup for vanilla SVGD. 

We keep the experimental setup \ie number of particles and number iterations exactly the same for both vanilla SVGD and equivariant SVGD. For both the examples, it may seem that the original samples from the vanilla SVGD capture the target distribution better than the equivariant counterpart (first column for Figs. \ref{fig:inv-gau}-\ref{fig:reg-cir}). However, projecting the samples onto the factorized space (second column for the aforementioned figures) shows that equivariant SVGD more faithfully captures the target density compared to vanilla SVGD. Furthermore, due to its ability to model long-range interactions, we see for both examples that in the projected space of the invariant sampler, the samples are not close together whereas for vanilla SVGD most samples end up in a configuration where they reside in the same orbit. This phenomena is most evident for the concentric circles example where samples from vanilla SVGD reside on the high density region throughout the two circles resulting in all the samples being positioned on top of each other in the factorized space demonstrating its inability to capture the distribution. On the other hand, invariant SVGD prevents any sample from residing on the same orbit of another sample due to long-range repulsive force from the equivariant kernel allowing it to sample more faithfully from the invariant densities.

Secondly, we study the effect of increasing the number of particles used for vanilla SVGD for the two concentric circles example. In \Cref{fig:sample}, we plot the average log-likelihoods of the particles from vanilla SVGD and particles from invariant SVGD as a function of number of iterations and compare it to the ground-truth average log-likelihood. We run vanilla SVGD with up to 32 times more particles than invariant SVGD. As evident from the plot, invariant SVGD converges to the final configuration within the first 100 iterations with average log-likelihood closely matching the ground truth. Vanilla SVGD, on the other hand, is unable to converge to the ground truth with even 32 times more samples and 5000 iterations due to its inability to interact with particles at longer distances.

Finally, we study the effect of different configurations of the initial particles on the performance of vanilla and invariant SVGD in \Cref{fig:init} for the $C_4$-Gaussian example. As shown by \cite{zhuo2018message, zhang2020stochastic} and \cite{d2021annealed}, the particles in vanilla SVGD have a tendency to collapse to a few local modes that are closest to the initial distribution of the particles. We test the robustness of invariant SVGD to particles with initial distributions localized to different regions in the space. We plot the average log-likelihoods of the converged samples for both invariant and vanilla SVGD for all random initializations in \Cref{fig:init} and compare this to the ground truth average log-likelihood. The plot illustrates that equivariant SVGD is more robust to the initial distribution of particles than vanilla SVGD. Nevertheless, if the group-factorized space is multi-modal, equivariant SVGD might exhibit a tendency to favour one of modes. However, this can be easily alleviated by either adding some noise to the SVGD update as proposed by \cite{zhang2020stochastic} similar to SGLD \citep{welling2011bayesian} or using an annealing strategy \citep{d2021annealed}.

\begin{figure}[t]
\begin{subfigure}{.6\textwidth}
  \centering
  \includegraphics[width=.31\linewidth, keepaspectratio]{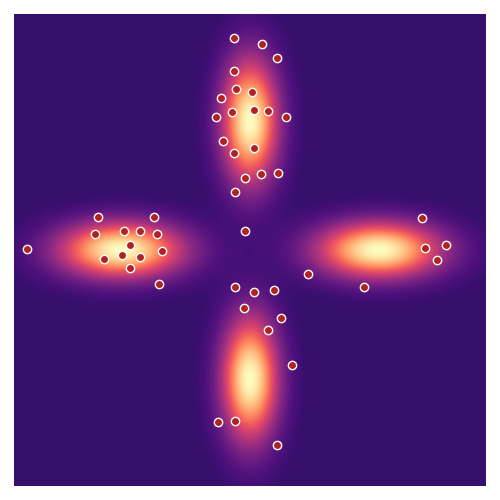}  
  \hspace{-0.3cm}
  \includegraphics[width=.31\linewidth]{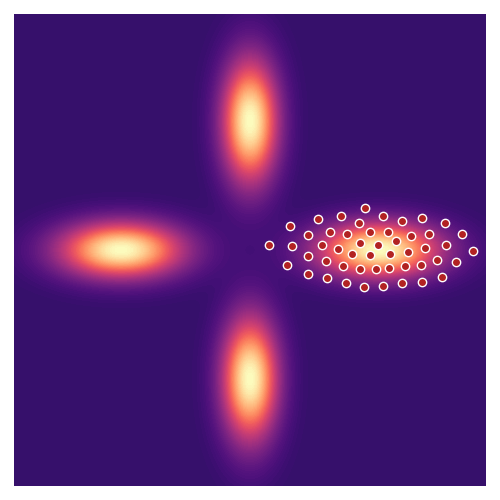} 
  \hspace{-0.25cm}
  \includegraphics[width=.31\linewidth]{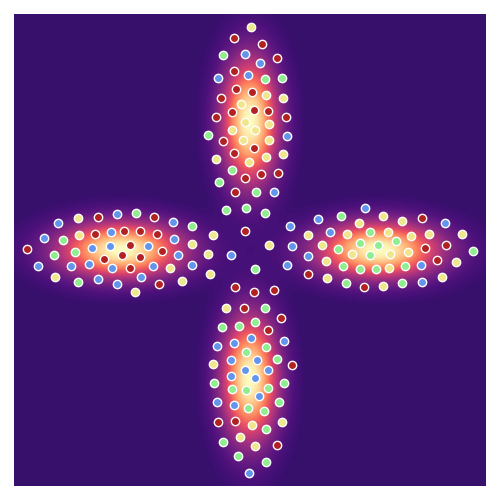} 
  \caption{$C_4$ Gaussians : Invariant SVGD sampling}
  \label{fig:inv-gau}
\end{subfigure}
\hspace{-0.8cm}
\begin{subfigure}{.4\textwidth}
  \centering
  \includegraphics[width=.47\linewidth]{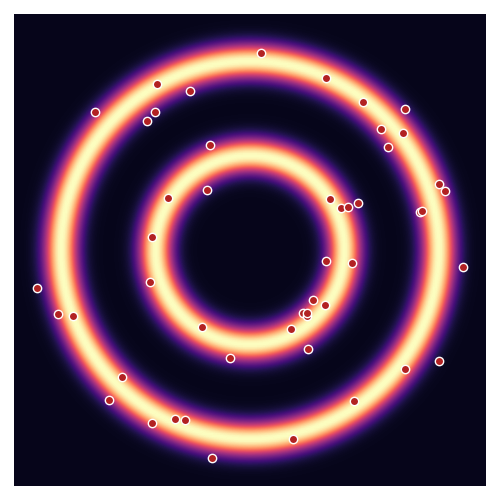}  
  \hspace{-0.3cm}
  \includegraphics[width=.47\linewidth]{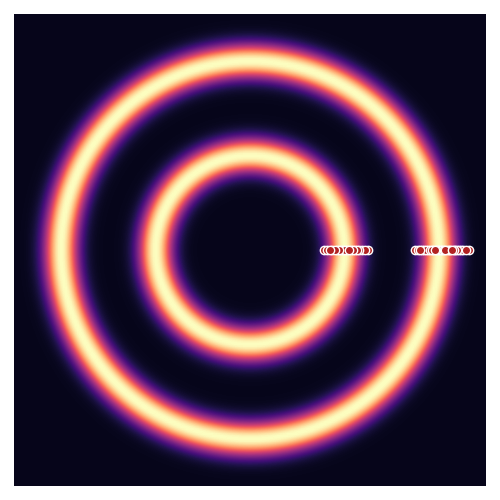}  
  \caption{Two Circles : Invariant SVGD sampling}
  \label{fig:sinv-circ}
\end{subfigure}
\newline
\begin{subfigure}{.6\textwidth}
  \centering
  \includegraphics[width=.31\linewidth]{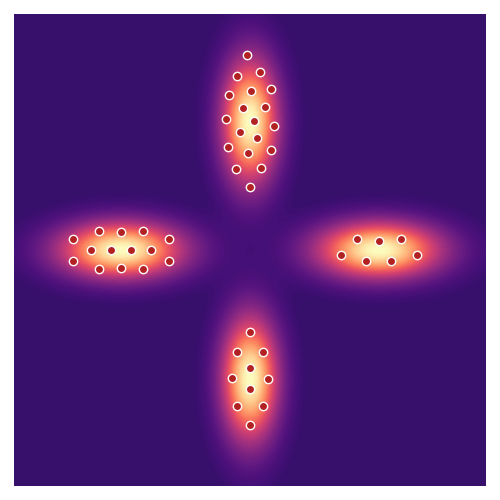} 
  \hspace{-0.3cm}
  \includegraphics[width=.31\linewidth]{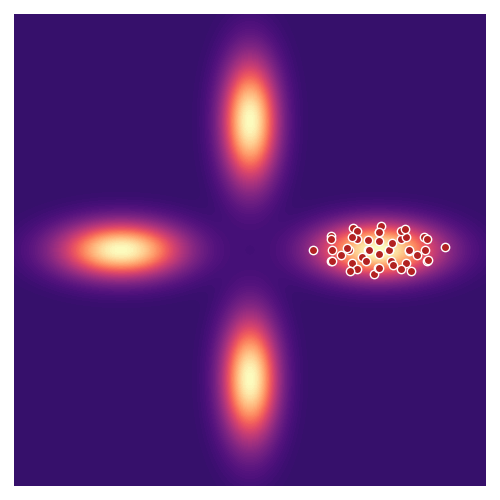} 
  \hspace{-0.25cm}
  \includegraphics[width=.31\linewidth]{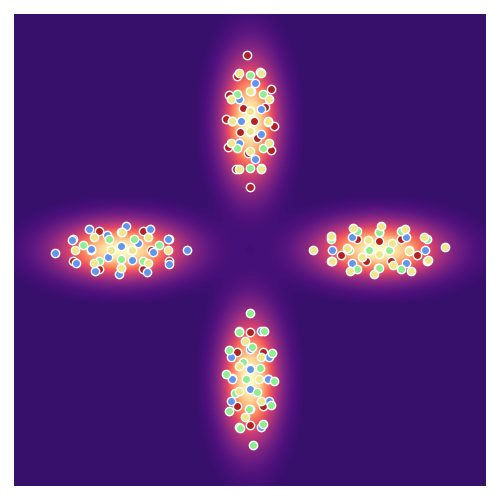} 
  \caption{$C_4$ Gaussians : Vanilla SVGD Sampling}
  \label{fig:reg-gau}
\end{subfigure}
\hspace{-0.8cm}
\begin{subfigure}{.4\textwidth}
  \centering
  \includegraphics[width=.47\linewidth]{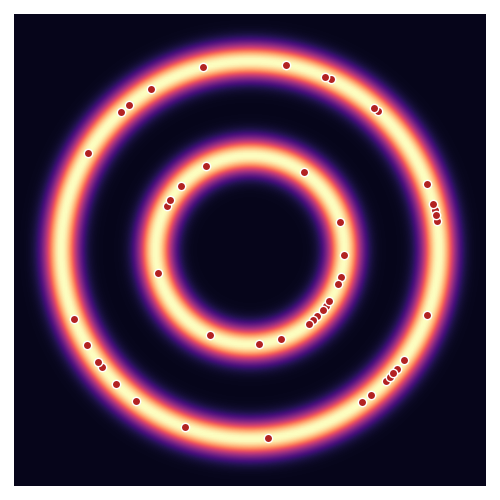}  
  \hspace{-0.3cm}
  \includegraphics[width=.47\linewidth]{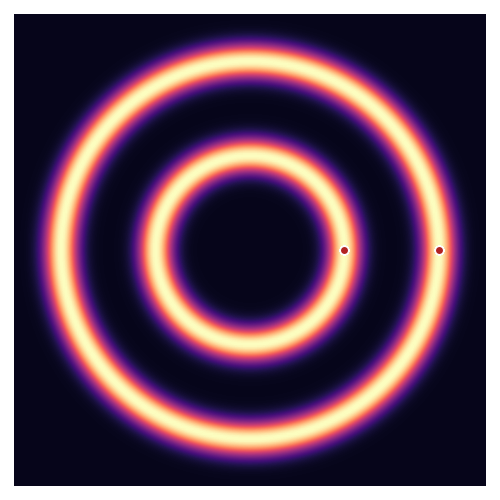} 
  \caption{Two Circles : Vanilla SVGD Sampling}
  \label{fig:reg-cir}
\end{subfigure}
\caption{\emph{Recommended to view in color}. \ref{fig:inv-gau} (Left to Right) Original Samples from E-SVGD, samples projected on to the group-factorized space and, samples obtained after applying group actions to the original samples. Yellow, Green and, Blue samples represent original samples rotated by $90^{\circ}$, $180^{\circ}$ and, $270^{\circ}$ respectively. \ref{fig:sinv-circ} (Left to Right) Original Samples from E-SVGD and, samples projected on to the group-factorized space. \ref{fig:reg-gau}-\ref{fig:reg-cir}: Same as \ref{fig:inv-gau}-\ref{fig:sinv-circ} but for vanilla SVGD. 
}
\label{fig:esvgd}
\end{figure}

\section{Equivariant Joint Energy Model}
\label{sec:ebm}
In \Cref{sec:esvgd}, we developed equivariant Stein variational gradient descent algorithm for sampling from invariant densities. In this section, we leverage the recent tremendous advances in deep geometric learning \citep{cohen2016group, dieleman2016exploiting, bronstein2021geometric} to propose \emph{equivariant energy based models} that are trained contrastively using our proposed equivariant Stein variational gradient descent algorithm to learn invariant (unnormalized) densities $\pi$ given access to \iid samples $\lbag \vx_1, \vx_2, \cdots, \vx_n \rbag \sim \pi$.


Given a set of samples $\lbag \vx_1, \vx_2, \cdots, \vx_n \rbag \subseteq \RR^d$, energy-based models \citep{YannEBM} learn an energy function $\rmE_{\vtheta}(\vx) : \RR^d \to \RR$ that defines a probability distribution $\tilde{\pi}_{\vtheta}(\vx) =  \nicefrac{\mathsf{exp}\big(-\rmE_{\vtheta}(\vx)\big)}{Z_{\vtheta}}$, where $Z_{\vtheta} = \int \mathsf{exp}\big(-\rmE_{\vtheta}(\vx)\big) \rmd \vx $ is the partition function.  Unlike other popular tractable density models like normalizing flows, EBMs are less restrictive in the parameterization of the functional form of $\tilde{\pi}_{\vtheta}(\cdot)$ since the energy function does not need to integrate to one, it can be parameterized using any nonlinear function. Conveniently, if $\pi$ is $\cG$-invariant, we can use the existing equivariant deep network architectures to parameterize $\rmE_{\theta}(\cdot)$ to encode the symmetries into the energy network. Such an equivariant energy network defines an \emph{equivariant} energy based model. EBMs are usually trained by maximizing the log-likelihood of the data under the given model:
\begin{align}
    \label{eq:mle_ebm}
    \vtheta^* := \argmin_{\vtheta} \mathcal{L}_{\mathsf{ML}}(\vtheta) = \bE_{\vx \sim \pi}\big[-\log \tilde{\pi}_{\vtheta}(\vx) \big]
\end{align}
However, evaluating $Z_{\vtheta}$ is intractable for most (useful) choices of $\rmE_{\vtheta}(\cdot)$ which makes learning EBMs via maximum likelihood estimation problematic. Contrastive divergence \citep{hinton2006unsupervised} provides a paradigm to learn EBMs using maximum likelihood estimation without needing to compute $Z_{\vtheta}$ by approximating the gradient of $\nabla_{\vtheta}\cL_{\mathsf{ML}}(\vtheta)$ in \Cref{eq:mle_ebm} as follows:
\begin{align}
    \label{eq:grad}
    \nabla_{\vtheta}\cL_{\mathsf{ML}}(\vtheta) \approx \bE_{\vx^+ \sim \pi}\big[\nabla_{\vtheta} \rmE_{\vtheta}(\vx^+) \big] - \bE_{\vx^- \sim \tilde{\pi}_{\vtheta}}\big[\nabla_{\vtheta} \rmE_{\vtheta}(\vx^-) \big]
\end{align}

Intuitively, the gradient in \Cref{eq:grad} drives the model such that it assigns higher energy to the negative samples $\vx^-$ sampled from the current model and decreases the energy of the positive samples $\vx^+$ which are the data-points from the target distribution. Since, training an EBM using MLE requires sampling from the current model $\tilde{\pi}(\vtheta)$, successful training of EBMs relies heavily on sampling strategies that lead to faster mixing. Fortuitously, since $\rmE_{\vtheta}(\cdot)$ in our present setting is $\cG$-equivariant, we propose to use our equivariant sampler for more efficient training\footnote{compared to using a regular sampler with no encoded symmetries.} of the equivariant energy based model.

\begin{wrapfigure}[11]{r}{0.59\textwidth}
\vspace{-6mm}
\begin{algorithm}[H]
\label{algo:ebm}
    \KwIn{$\lbag \vx^+_1, \vx^+_2, \cdots, \vx^+_m\rbag \sim \pi(\vx)$}
     \While{\text{not converged}}{
 $\triangleright$ \textit{Generate samples from current eqNN model} $\rmE_{\vtheta}$
 $\lbag \vx^-_1, \vx^-_2, \cdots, \vx^-_m\rbag = \mathsf{EquivariantSVGD}(\rmE_{\vtheta}$) \;
 $\triangleright$ \textit{Optimize objective} $\cL_{\mathsf{ML}}(\vtheta)$:
 $\Delta \vtheta \leftarrow  \sum_{i=1}^m \nabla_{\vtheta}\rmE_{\vtheta}(\vx_i^+) -  \nabla_{\vtheta}\rmE_{\vtheta}(\vx_i^-)$ \;
   $\triangleright$ 
   \textit{Update} $\vtheta$ \textit{using} $\Delta \vtheta$ \textit{and Adam optimizer}
}
\caption{Equivariant EBM training}
\end{algorithm}
\end{wrapfigure}

Additionally, following \cite{grathwohl2019your}, we can extend equivariant energy based models to equivariant joint energy models. Let $\{(\vx_1, y_1), (\vx_2, y_2),\cdots, (\vx_n, y_n)\} \subseteq \RR^d\times[K]$ be a set of samples with observations $\vx_i$ and labels $y_i$. Given a parametric function $f_{\vtheta}: \RR^d \to \RR^k$, a classifier uses the conditional distribution $\tilde{\pi}_{\vtheta}(y|\vx) \propto \mathsf{exp}(f_{\vtheta}(\vx)[y])$ where $f_{\vtheta}(\vx)[y]$ is the logit corresponding to the $y^{\text{th}}$ class label. As shown by \cite{grathwohl2019your}, these logits can be used to define the joint density $\tilde{\pi}_{\vtheta}(\vx, y)$ and marginal density $\tilde{\pi}_{\vtheta}(\vx)$ as follows:
\begin{align}
    \label{eq:JEM}
    \tilde{\pi}_{\vtheta}(\vx, y) = \frac{\mathsf{exp}\big(f_{\vtheta}(\vx)[y]\big)}{Z_{\vtheta}}, \quad \text{and}, \quad \tilde{\pi}_{\vtheta}(\vx) = \frac{\sum_{y}\mathsf{exp}\big(f_{\vtheta}(\vx)[y]\big)}{Z_{\vtheta}}
\end{align}

Hence, the energy function at a point $\vx$ is given by $\rmE_{\vtheta} = -\log \sum_{y} \mathsf{exp}(f_{\vtheta}(\vx)[y])$ with joint energy $\rmE_{\vtheta}(\vx, y) = -f_{\vtheta}(\vx)[y]$. In our setting,  the joint distribution $\pi(\vx, y)$ is $\cG$-invariant in the first argument \ie $\pi(\rmR_g \vx, y) = \pi(\vx, y), \forall g \in \cG$. An example of such a setting is any image data-set where the class label does not change if the image is rotated by an angle. Using \Cref{eq:JEM}, it suffices for the function $f_{\vtheta}$ to be $\cG$-equivariant to model a $\cG$-invariant density  $\tilde{\pi}_{\vtheta}(\vx, y)$. Furthermore, a $\cG$-equivariant $f_{\vtheta}$ also makes the marginal density $\tilde{\pi}_{\vtheta}(\vx)$ and conditional density $\tilde{\pi}_{\vtheta}(y|\vx)$ $\cG$-invariant in the input $\vx$. We call such an energy model where $f_{\vtheta}$ is equivariant to a symmetry transformation to be an \emph{equivariant} joint energy model. 

We can train this model by maximizing the log-likelihood of the joint distribution as follows:
\begin{align}
    \label{eq:loss-JEM}
    \mathcal{L}(\vtheta) :&= \cL_{\mathsf{ML}}(\vtheta) + \cL_{\mathsf{SL}}(\vtheta) = \log \tilde{\pi}_{\vtheta}(\vx) + \log \tilde{\pi}_{\vtheta}(y|\vx)
\end{align}
where $\cL_{\mathsf{SL}}(\vtheta)$ is the supervised loss which is the cross-entropy loss in the case of classification. Thus, an equivariant joint energy model can now be trained by applying the gradient estimator in \Cref{eq:grad} for $\log \tilde{\pi}_{\vtheta}(\vx)$ and evaluating the gradient of $\log \tilde{\pi}_{\vtheta}(y|\vx)$ through back-propagation. Conveniently, \Cref{eq:loss-JEM} can also be used for semi-supervised learning with $\cL_{\mathsf{SL}}((\vtheta))$ substituted with the appropriate supervised loss \eg MSE for regression. 

Let us end this section with an empirical example for learning a mixture of $C_4$-Gaussians (\Cref{fig:cond-ebm}) as shown in row two of the leftmost column of \Cref{fig:cond-ebm}. The innermost $C_4$-Gaussian defines the class conditional probability $\pi(\vx|y=0)$ (row 3) and the outer $C_4$-Gaussian defines $\pi(\vx|y=1)$ (row 4). We learn a non-equivariant joint EBM using vanilla SVGD (\cf \Cref{fig:cond-ebm} center column) and an equivariant joint EBM using equivariant SVGD (\cf \Cref{fig:cond-ebm} right column) keeping the number of iterations and particles the same for training. In \Cref{fig:cond-ebm}, we plot the decision boundaries learned by the model in the top row. The star marked samples in the figure are the samples generated by the underlying model. We plot the joint distribution and the class conditional distributions in row two-four respectively. The figure abundantly demonstrates the superior performance of an equivariant joint energy model trained using equivariant SVGD over its non-equivariant counterpart. A more detailed figure with comparisons to an equivariant joint energy model trained using vanilla SVGD is presented in \Cref{app:jem_gaussian}.

\newcommand{\fw}{0.31\linewidth}
\begin{wrapfigure}[32]{r}{0.49\linewidth}
    \centering
\begin{subfigure}[b]{\fw}
    \vspace{-2mm}
    \includegraphics[width=\textwidth]{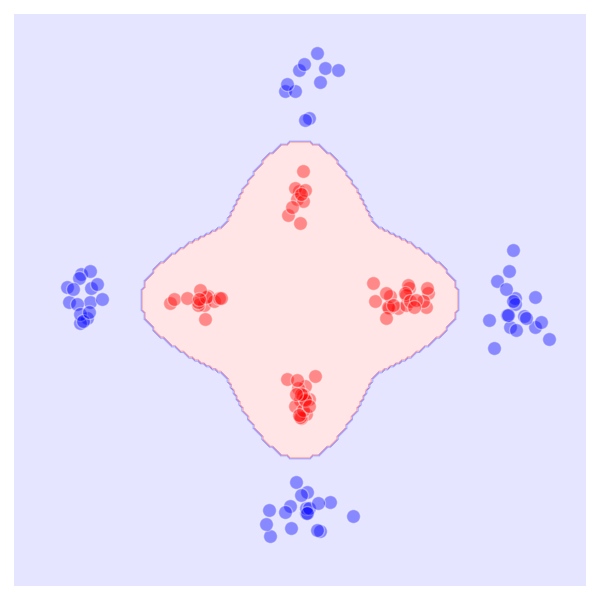}
\end{subfigure}
\begin{subfigure}[b]{\fw}
    \vspace{-2mm}
    \includegraphics[width=\textwidth]{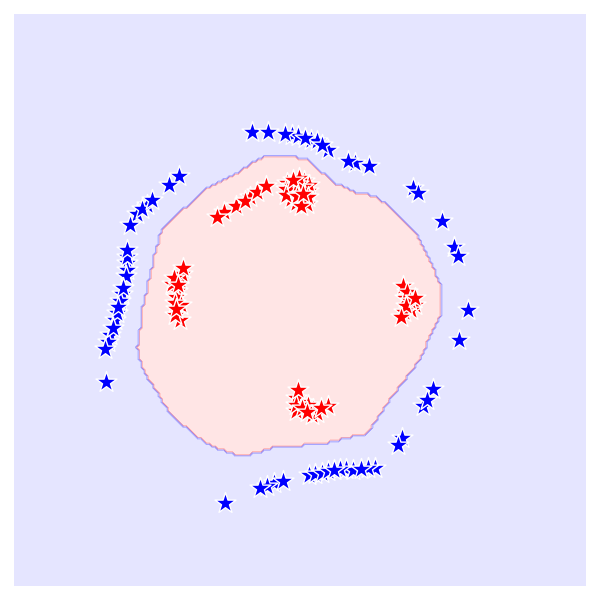}
\end{subfigure}
\begin{subfigure}[b]{\fw}
    \vspace{-2mm}
    \includegraphics[width=\textwidth]{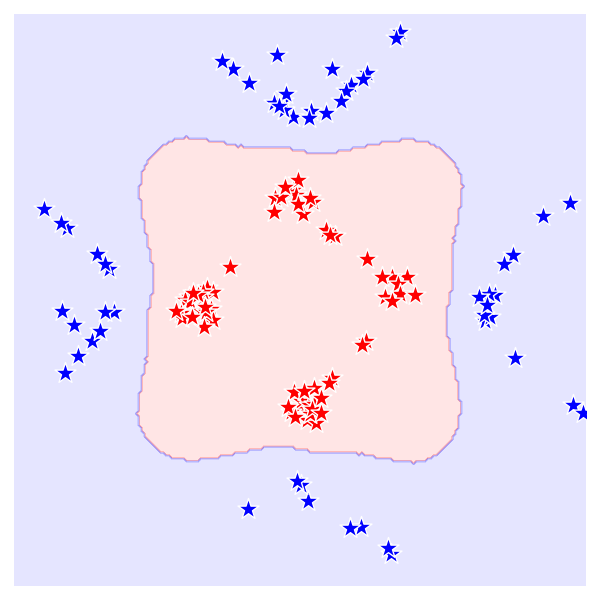}
\end{subfigure}

\begin{subfigure}[b]{\fw}
    \includegraphics[width=\textwidth]{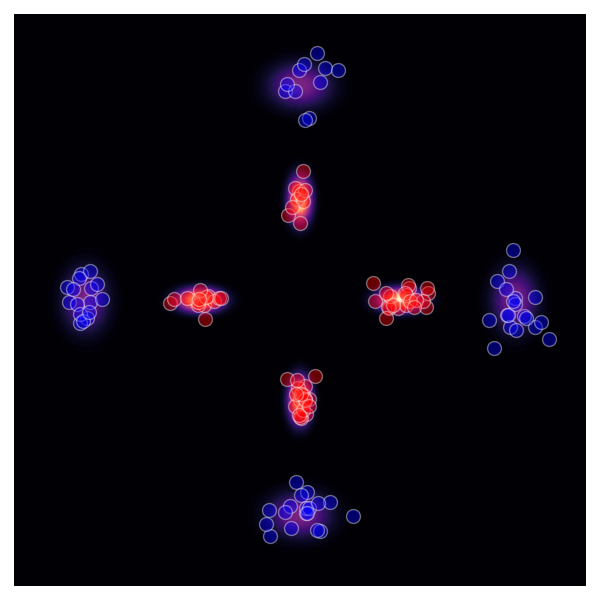}
\end{subfigure}
\begin{subfigure}[b]{\fw}
    \includegraphics[width=\textwidth]{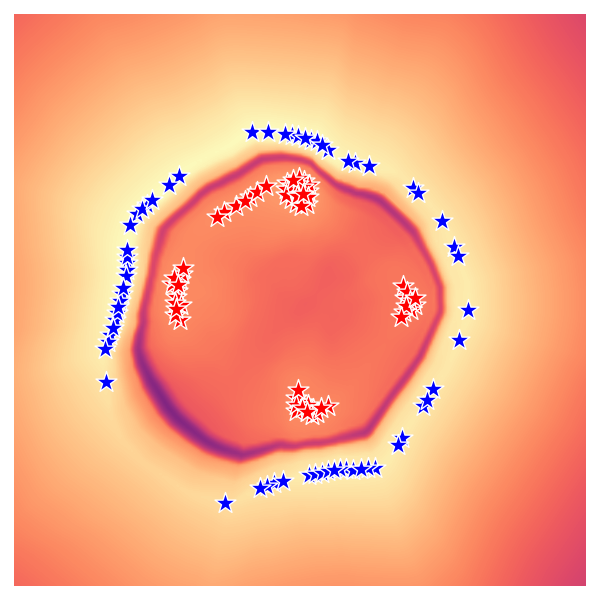}
\end{subfigure}
\begin{subfigure}[b]{\fw}
    \includegraphics[width=\textwidth]{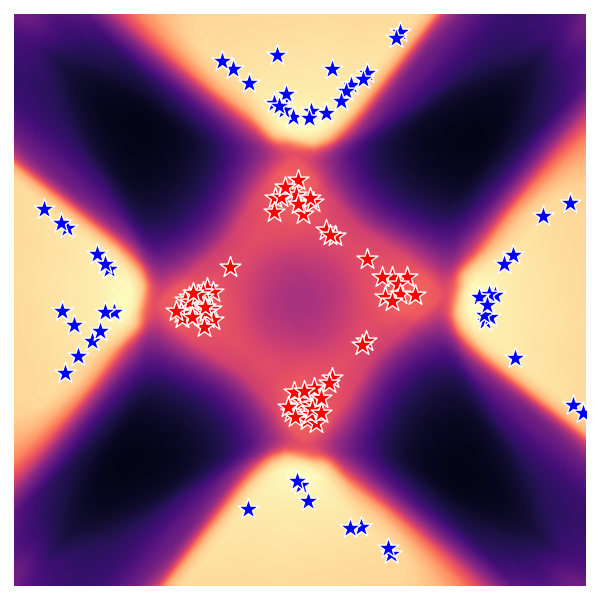}
\end{subfigure}

\begin{subfigure}[b]{\fw}
    \includegraphics[width=\textwidth]{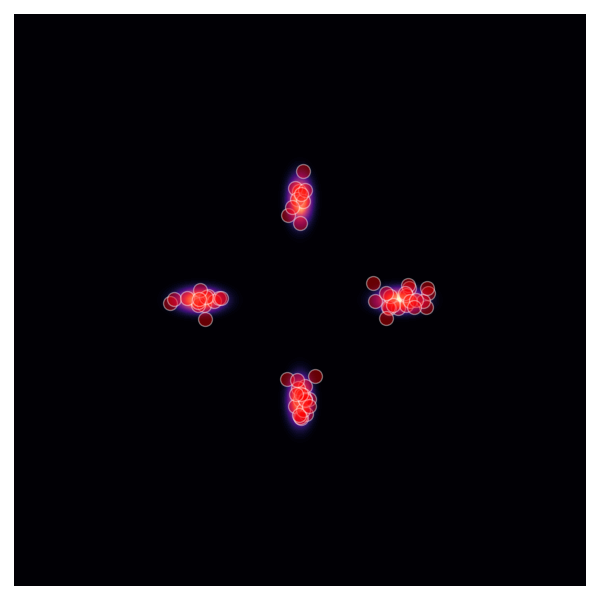}
\end{subfigure}
\begin{subfigure}[b]{\fw}
    \includegraphics[width=\textwidth]{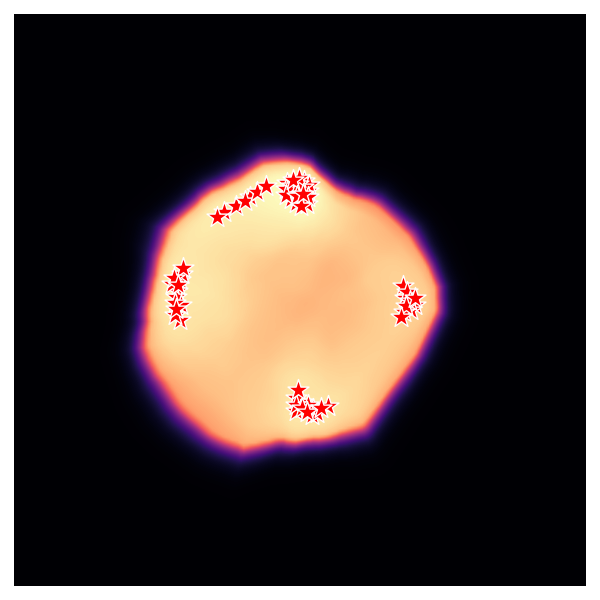}
\end{subfigure}
\begin{subfigure}[b]{\fw}
    \includegraphics[width=\textwidth]{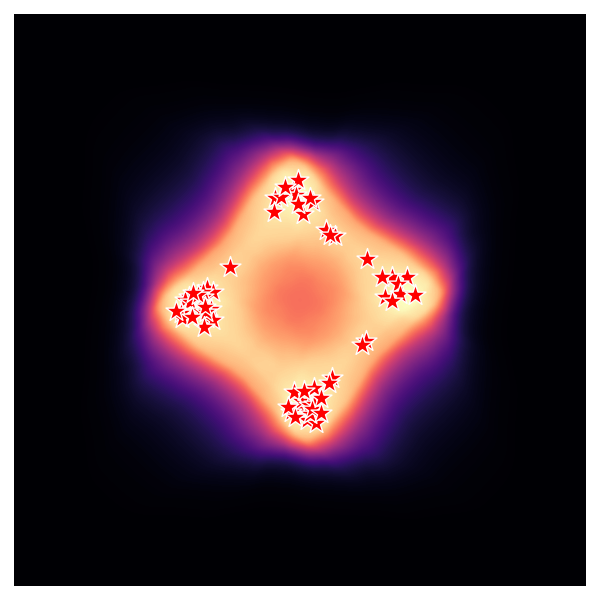}
\end{subfigure}

\begin{subfigure}[b]{\fw}
    \includegraphics[width=\textwidth]{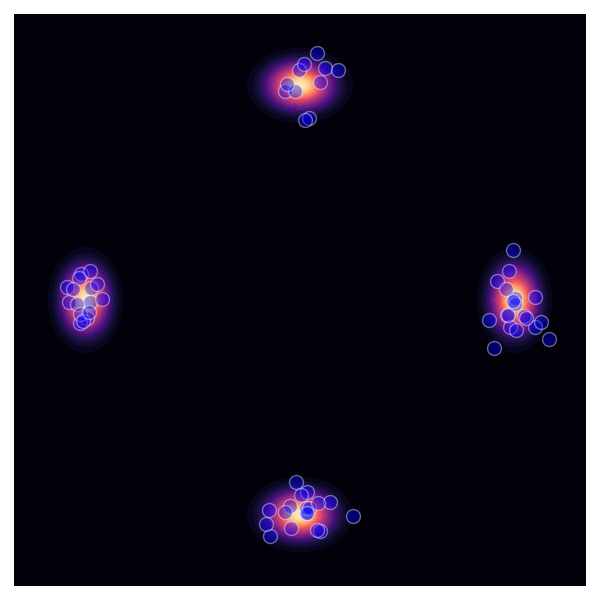}
\end{subfigure}
\begin{subfigure}[b]{\fw}
    \includegraphics[width=\textwidth]{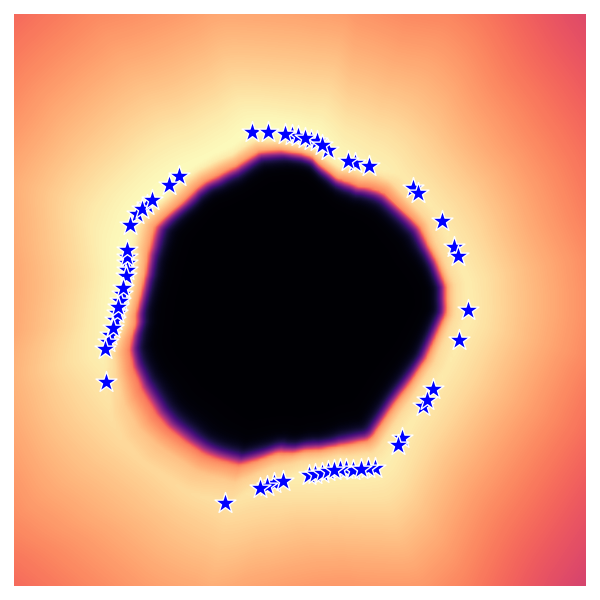}
\end{subfigure}
\begin{subfigure}[b]{\fw}
    \includegraphics[width=\textwidth]{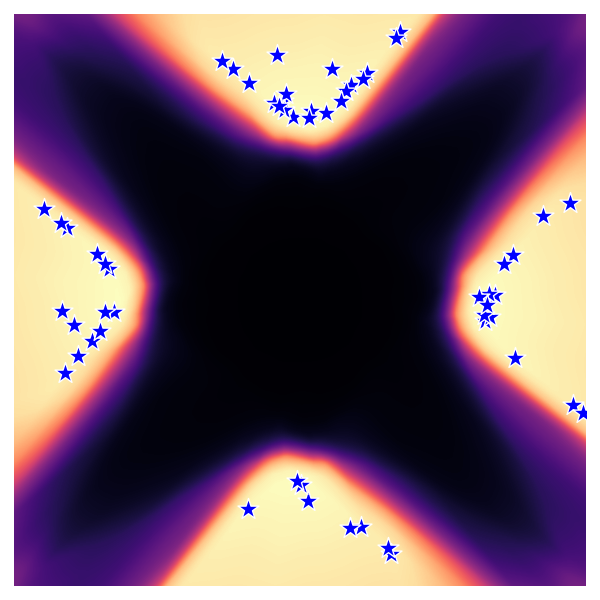}
\end{subfigure}

    \caption{$C_4$-Gaussian mixture model. \textit{Row 1:} Decision Boundary. \textit{Row 2:} Samples and energy of joint distribution $\pi(\vx, y)$. \textit{Row 3:} Samples and energy of conditional distribution $\pi(\vx|y=0)$. \textit{Row 4:} Samples and energy of conditional distribution $\pi(\vx|y=1)$. \textit{Left:} Target distribution. \textit{Middle:} Non-equivariant EBM trained with vanilla SVGD. \textit{Right:} E-EBM trained with E-SVGD.}
    \label{fig:cond-ebm}
\end{wrapfigure}
\section{Experiments}
\label{sec:exp}

In this section, we present empirical analysis of equivariant EBMs and E-SVGD through experiments to (i) reconstruct potential function describing a many-body particle system (DW-4) trained using limited number of meta-stable states, (ii) model a generative distribution of molecules (QM9) and generate novel samples and (iii) hybrid (generative \& discriminative) model invariant to rotations for FashionMNIST trained using dataset with no rotations. Due to space constraints, details about all the experiments as well as detailed figures are relegated to \Cref{app:exp}.

\textbf{DW-4:} In this many-body particles system, a double-well potential describes the configuration of four particles that is invariant to rotations, translations and, permutation of the particles. This system comprises five distinct metastable states which are characterized as the mimina in the potential function. In our experiment, we show that given access to only a single example of each metastable state configuration, an equivariant EBM trained with E-SVGD can recover other states with similar energy as those of in the training set. In \Cref{fig:dw4}, the first column shows the metastable states present in the training set. The second column are the states recovered by an EBM trained with vanilla SVGD which results in configurations that exactly copy the training set. The third column shows configurations generated by the equivariant model which are distinct from the training set but mimic the energies of the corresponding metastable states in the training set. Our setup is different from that of \cite{kohler2020equivariant}; we discuss this in detail in \Cref{app:dw4} and also produce similar results as \cite{kohler2020equivariant} for our model.     
\begin{figure}[!b]
    \centering
    \includegraphics[width=\textwidth]{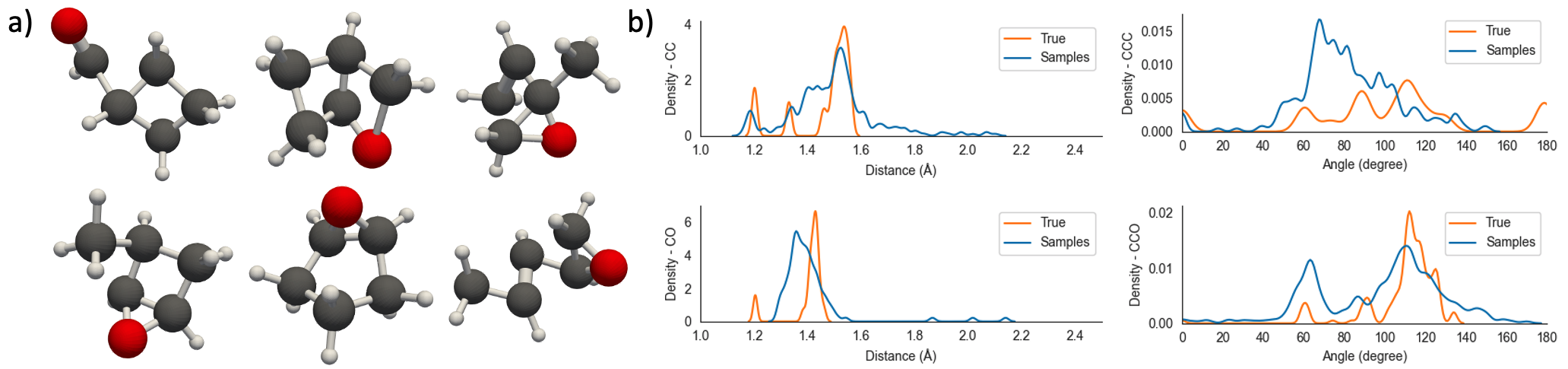}
    \caption{(a) Molecules sampled from a EBM parameterized by a E-GNN trained using E-SVGD. (b) Distribution of distance and angle between atom pairs and triplets.}
    \label{fig:qm9}
\end{figure}

\textbf{QM9:} QM9 is a molecular dataset containing over 145,000 molecules used for moleccular property prediction. However, we use this for molecular structure generation of constitutional isomers of \ch{C5H8O1}. Similar to DW-4, the molecules here are invariant to rotations, translations and, permutations of the same atoms. We encode these symmetries using E-GNN \citep{satorras2021n}, an equivariant graph neural network, to represent the energy. We trained our model via E-SVGD using \ch{C5H8O1} molecules present in the QM9 dataset and used the trained energy model to generate novel samples that are isomers of \ch{C5H8O1}. We show these novel generated molecules in Figure \ref{fig:qm9} wherein we used the relative distance between atoms as a proxy for determining the covalent bonds. Our generated molecules demonstrate the correct 3D arrangement of bonds while containing complex atom structures like aromatic rings. This is further supported by the plots comparing the radial distribution functions of the two most common heavy atom pairs to quantify our model fit to QM9 (\Cref{fig:qm9}). While, the generated molecules have a larger distributional spread, the range of values and modes -- for both angles and distances-- resemble the true distribution. We provide more details in \Cref{app:qm9}.
\begin{figure}
    \centering
\begin{subfigure}[b]{0.31\linewidth}
    \caption*{Joint distribution}
    \vspace{-2mm}
    \includegraphics[width=\textwidth]{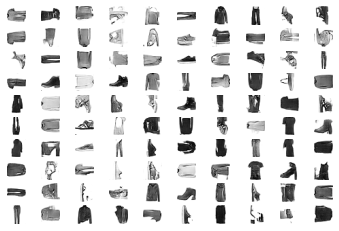}
\end{subfigure}
\hspace{0.5mm}
\begin{subfigure}[b]{0.31\linewidth}
    \caption*{Conditional distribution}
    \vspace{-2mm}
    \includegraphics[width=\textwidth]{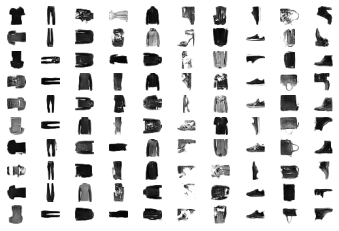}
\end{subfigure}
\begin{subfigure}[b]{0.31\linewidth}
    \caption*{Accuracy}
    \vspace{-2mm}
    \includegraphics[width=\textwidth]{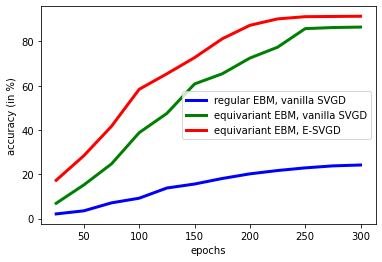}
\end{subfigure}

    \caption{\textit{Left \& Center:} Samples generated from joint and class-conditional distribution using equivariant EBM. \textit{Right:} Plot of classification accuracy vs. training iterations for equivariant and regular EBMs trained using vanilla SVGD and E-SVGD.}
    \label{fig:FMNIST}
\end{figure}

\begin{wrapfigure}[23]{r}{0.30\linewidth}
    \centering
\begin{subfigure}[b]{\fw}
    \vspace{-2mm}
    \includegraphics[width=\textwidth]{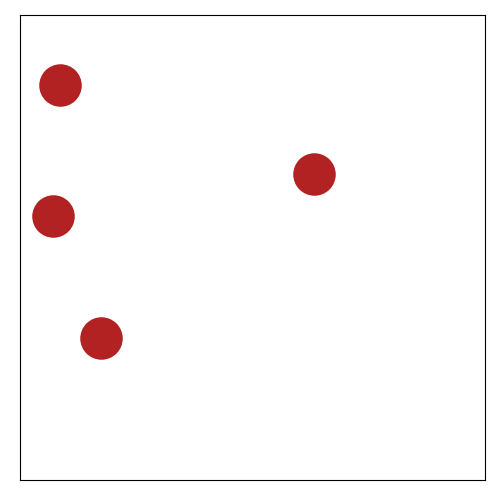}
\end{subfigure}
\begin{subfigure}[b]{\fw}
    \vspace{-2mm}
    \includegraphics[width=\textwidth]{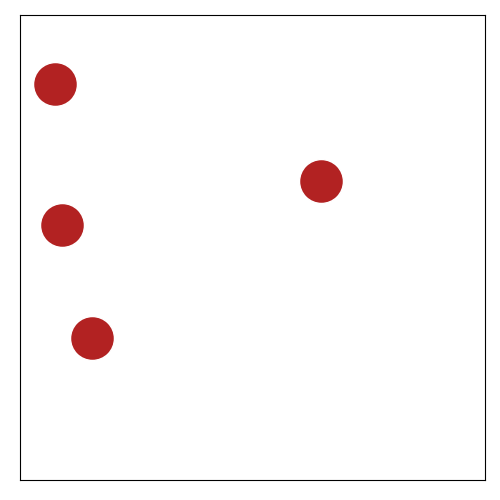}
\end{subfigure}
\begin{subfigure}[b]{\fw}
    \vspace{-2mm}
    \includegraphics[width=\textwidth]{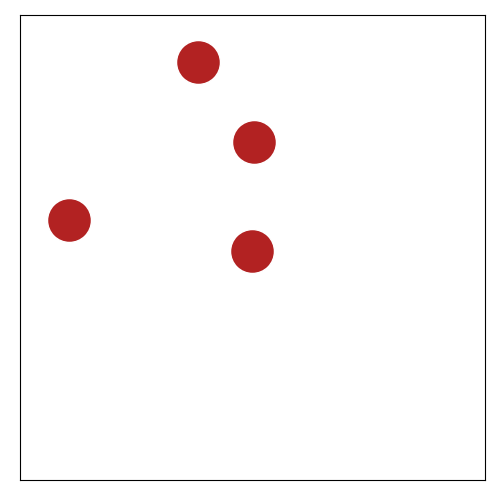}
\end{subfigure}

\begin{subfigure}[b]{\fw}
    \includegraphics[width=\textwidth]{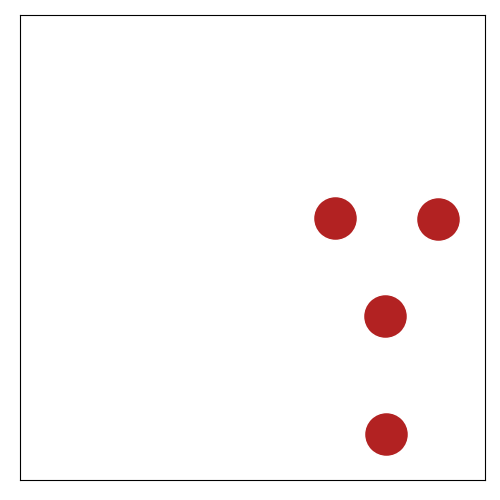}
\end{subfigure}
\begin{subfigure}[b]{\fw}
    \includegraphics[width=\textwidth]{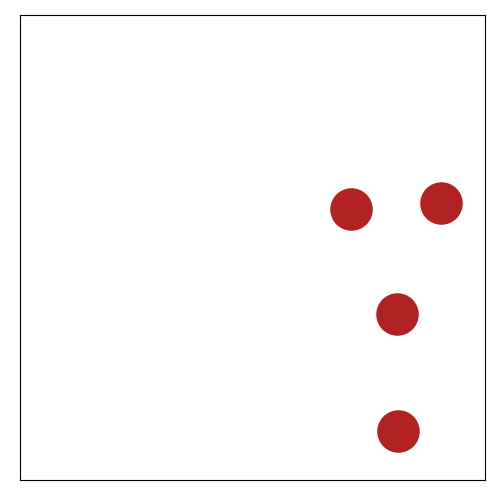}
\end{subfigure}
\begin{subfigure}[b]{\fw}
    \includegraphics[width=\textwidth]{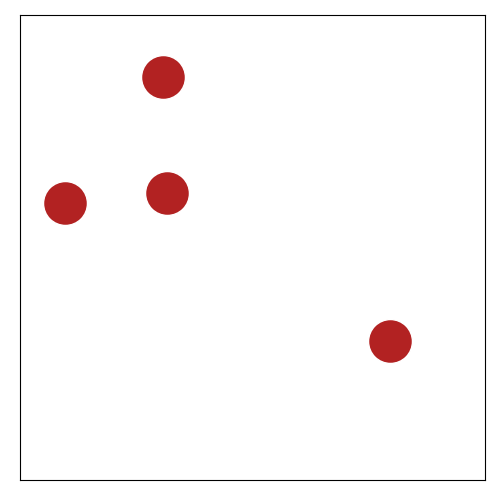}
\end{subfigure}

\begin{subfigure}[b]{\fw}
    \includegraphics[width=\textwidth]{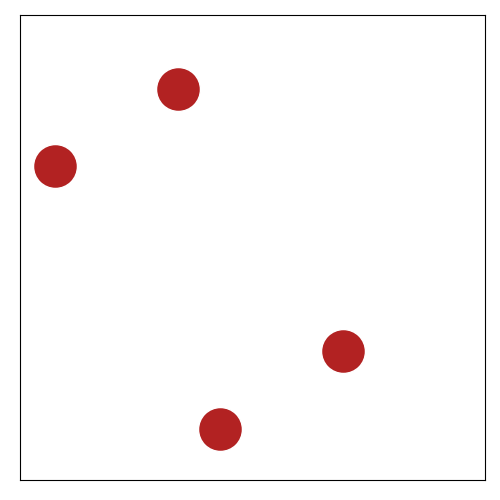}
\end{subfigure}
\begin{subfigure}[b]{\fw}
    \includegraphics[width=\textwidth]{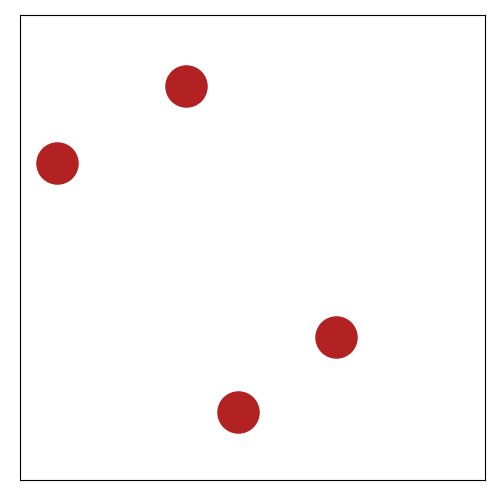}
\end{subfigure}
\begin{subfigure}[b]{\fw}
    \includegraphics[width=\textwidth]{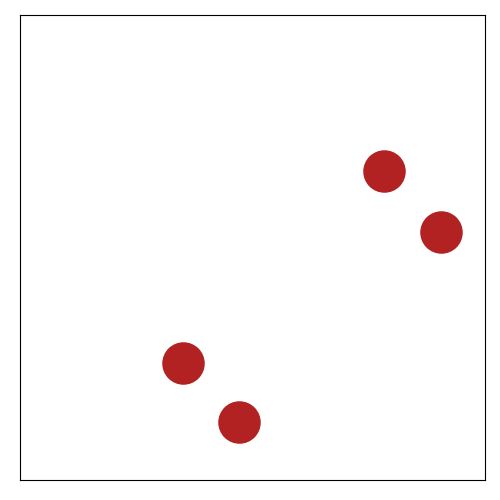}
\end{subfigure}

\begin{subfigure}[b]{\fw}
    \includegraphics[width=\textwidth]{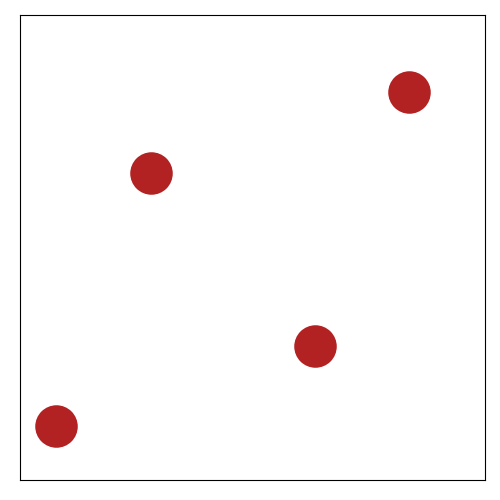}
\end{subfigure}
\begin{subfigure}[b]{\fw}
    \includegraphics[width=\textwidth]{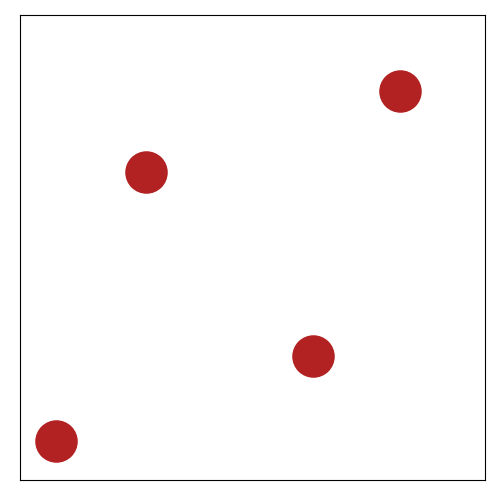}
\end{subfigure}
\begin{subfigure}[b]{\fw}
    \includegraphics[width=\textwidth]{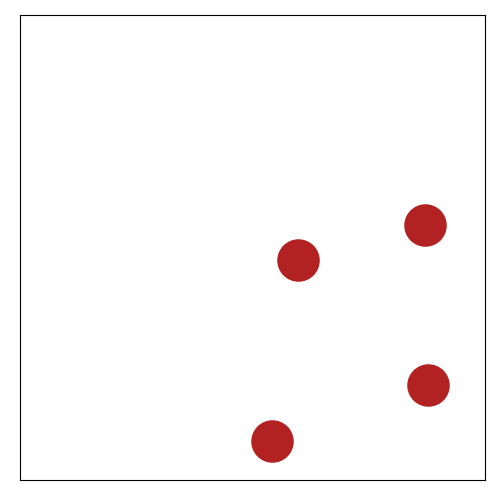}
\end{subfigure}

\begin{subfigure}[b]{\fw}
    \includegraphics[width=\textwidth]{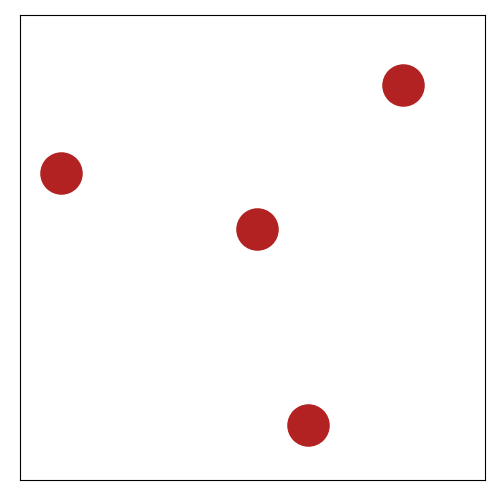}
\end{subfigure}
\begin{subfigure}[b]{\fw}
    \includegraphics[width=\textwidth]{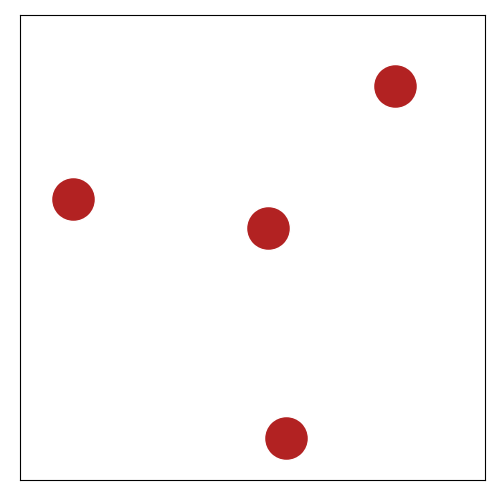}
\end{subfigure}
\begin{subfigure}[b]{\fw}
    \includegraphics[width=\textwidth]{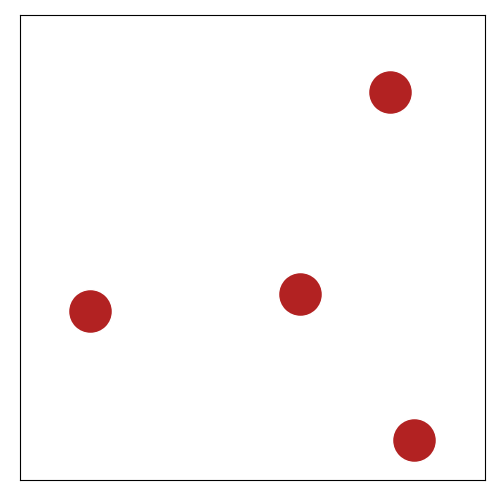}
\end{subfigure}

    \caption{\textit{Col. 1}: Samples from true potential energy. \textit{Col. 2}: Samples from EBM trained with SVGD. \textit{Col. 3}: Samples from equivariant EBM trained with E-SVGD.}
    \label{fig:dw4}
\end{wrapfigure}

\textbf{FashionMNIST:} (Details in \Cref{app:fashion}) In this experiment, we take the FashionMNIST dataset with training set consisting regular images whereas the test set is processed to contain images that are randomly rotated using the $C_4$-symmetry group. We train an equivariant energy model where the energy function is a $C_4$ steerable CNN \citep{e2cnn} with both E-SVGD and vanilla SVGD. Furthermore, we also compare to an energy model with no rotation symmetries and depict the performance in terms of classification accuracy on the held out images of these three models as a function of the number of training iterations. The plot in \Cref{fig:FMNIST} shows, albeit unsurprisingly, that an equivariant energy model performs better than a regular model. Furthermore, the results also illustrate that an equivariant model trained with E-SVGD converges faster than when trained with vanilla SVGD highlighting the benefit of using E-SVGD for training equivariant EBMs. Furthermore, in \Cref{fig:FMNIST}, we show samples generated by E-SVGD using the trained equivariant EBM from the joint and the class-conditional distribution.

\vspace{-0.4em}
\section{Discussion and Conclusion}
\label{sec:con}
\vspace{-0.4em}
In this paper, we focused on incorporating \emph{inductive bias} in the form of symmetry transformations using equivariant functions for sampling and learning invariant densities. We first proposed \emph{equivariant} Stein variational gradient descent algorithm for sampling from invariant densities by using equivariant kernels which affords many benefits in terms of efficiency due to its ability to model long-range interactions between particles. However, a major limitation of Stein variational gradient descent algorithm in general is its sensitivity to the kernel hyper-parameters. An interesting future work might be to develop strategies to either adapt or learn these hyper-parameters while running the SVGD dynamics.

\looseness=-1

Subsequently, we proposed \emph{equivariant} energy based models wherein the energy function is parameterized by an equivariant network. In our experiments, we leveraged the recent advances in geometric deep learning to model EBMs using steerable CNNs \citep{e2cnn} for images, equivariant graph networks \citep{satorras2021n} for representing molecules, and group equivariant networks \citep{cohen2016group} for many-body particle systems. We used equivariant SVGD to train these equivariant energy based models for modelling invariant densities and demonstrated that incorporating symmetries in the energy model as well as the sampler leads to efficient training. However, as discussed in previous works \citep{grathwohl2019your}, training EBMs using contrastive divergence and short sampling chains is often unstable and challenging. These issues remain with equivariant samplers and have to be addressed to be able to train large-scale energy based models.

\newpage

\bibliographystyle{apalike}
\bibliography{refs}
\newpage

\newpage
\appendix

\section{Sampling using Equivariant Flows}
\label{app:sampling}
Neural transport augmented sampling, first introduced by \cite{ParnoMarzouk18}, is a general method for using normalizing flows to sample from a given density $\pi$.   Informally, the method proceeds by learning a diffeomorphic map $\Tb : \Zsf \to \Theta$ such that $\tilde{p}(\vz) = \pi(\vtheta)\cdot |\Tb'(\vz)|$ where $\vz = \Tb^{-1}(\vtheta)$ such that $p(\vz)$ has a simple geometry amenable to efficient MCMC sampling. Thus, samples can be generated from $\pi(\vtheta)$ by running MCMC chain in the $\Zsf$-space and pushing these samples onto the $\Theta$-space using $\Tb$. The transformation $\Tb$ can be learned by minimizing the KL-divergence between a fixed distribution with simple geometry in the z-space \eg a standard Gaussian and $\tilde{p}(\vz)$ above. The learning phase attempts to ensure that the distribution $\tilde{p}(\zv)$ is approximately close to the fixed distribution with easy geometry so that MCMC sampling is efficient. 

Neural transport augmented samplers have been subsequently extended by \cite{hoffman2019neutra} who use more powerful flow architectures and, \cite{jaini2021sampling} who extend the idea to sampling from discrete probability densities using flows with surjective transformations \citep{nielsen2020survae}. We believe these ideas and extensions for neural transport augmented samplers can also be used to sample from an invariant density $\pi$ by defining the flow transformations to be equivariant \'{a} la \cite{kohler2020equivariant}.

\section{Related Work}
\label{app:prev}

In this paper, we proposed equivariant Stein variational gradient descent algorithm for sampling from densities that are invariant to symmetry transformations. Another contribution of our work is subsequently using this equivariant sampling method to efficiently train equivariant energy based models for probabilistic modeling and inference. Perhaps the closest work to that presented in this manuscript is that of \cite{liu2017learning} who first\footnote{As far as the authors are aware.} used SVGD for training energy based models. However, their work does not consider incorporating symmetries in to either the sampler or the energy model itself. 

Separately, a major contribution of our paper is indeed extending SVGD to incorporate symmetries present in the underlying target density. Since it was introduced by \cite{liu2016stein}, Stein variational gradient descent has garnered a lot of attention as an alternative to Monte Carlo methods for sampling in Bayesian inference problems courtesy of its flexibility and accuracy obtained by combining variational inference and sampling paradigm. Stein variational gradient descent has been subsequently extended by \cite{wang2019stein} to incorporate geometry information using matrix valued kernels, and to discrete spaces in \cite{han2020stein}. While, \cite{duncan2019geometry} have studied the convergence properties of Stein variational gradient descent under mean-field convergence analysis, a thorough theoretical understanding is still lacking in finite particle limit. Several works have, however, empirically probed limitations of Stein variational gradient descent. Particularly, Stein variational gradient descent is susceptible to collapsing to a few modes depending on the initial configuration of the particles \citep{zhang2020stochastic, d2021annealed}. As we discussed towards the end of \Cref{sec:esvgd}, incorporating symmetries alleviates this problem partially when the group factorized distribution is unimodal. Furthermore, the problem of mode collapse in Stein variational gradient descent can be addressed by either adding noise to the SVGD update (\cf \Cref{eq:svgd}) or using an annealing strategy as proposed in \cite{d2021annealed}.

Another contribution of this paper is learning equivariant Energy-Based Models using equivariant Stein variational gradient descent. Energy Based Models have witnessed a revival recently. The primary difficulty in training energy based models is the need to evaluate the partition function which is often intractable,. Thus, training energy based models require methods that can approximate this partition function. One line of ideas thus is to train an auxiliary sampling network that generates samples to approximate the partition function \citep{kumar2019maximum, xie2018cooperative} making it eerily similar in essence to Generative Adversarial Networks (GANs) \citep{finn2016connection}. However, as discussed in \cite{du2019implicit}, such strategies are prone to mode collapse since the sampling network is often trained without an entropy term. 

An alternative to this is to use Markov Chain Monte Carlo Method to directly estimate the partition function providing several benefits afforded by MCMC sampling methods. This idea was first proposed by \cite{hinton2006unsupervised}, termed as Contrastive Divergence algorithm, which used gradient free MCMC chains initialized from training data to estimate the partition function. This was subsequently extended by \cite{tieleman_training_2008} who introduced \emph{persistence} in contrastive divergence where a single MCMC chain with a persistent state is employed to sample from the energy model. However, there are problems still with training EBMs using contrastive divergence. Specifically, EBMs training with contrastive divergence may not capture the target distribution faithfully since the MCMC chains used while training are truncated that lead to biased gradient updates hurting the learning dynamics \cite{nijkamp2019learning, schulz2010investigating}. Towards this end, a sampling procedure that converges quickly to sample from the energy function will potentially help to train energy based models. Thus, modelling invariant densities using equivariant energy functions that are trained using samplers that incorporate the symmetries in the energy function will potentially help with the training of such models.

Equivariant Stein variational gradient descent provides an efficient sampling procedure to train equivariant energy based models. However, modelling the equivariant energy function itself is mostly due to the tremendous advances in geometric deep learning \citep{bronstein2021geometric}. Particularly, we can leverage the various proposed architectures that that incorporate symmetries to model an equivariant energy function. In our experiments, we utilized these advances to model rotations using steerable CNNs \citep{e2cnn}, and molecules using E(n)-equivariant graph neural nets \citep{satorras2021n}.

\section{Equivariant Stein Variational Gradient Descent}
In this section, we provide the details for the toy experiments presented in \Cref{sec:esvgd} as well as additional experiments and plots.
\subsection{Additional SO(3) concentric spheres}
\label{app:spheres}
\begin{figure}[t]
\begin{subfigure}{.5\textwidth}
  \centering
  \includegraphics[width=.31\linewidth, keepaspectratio]{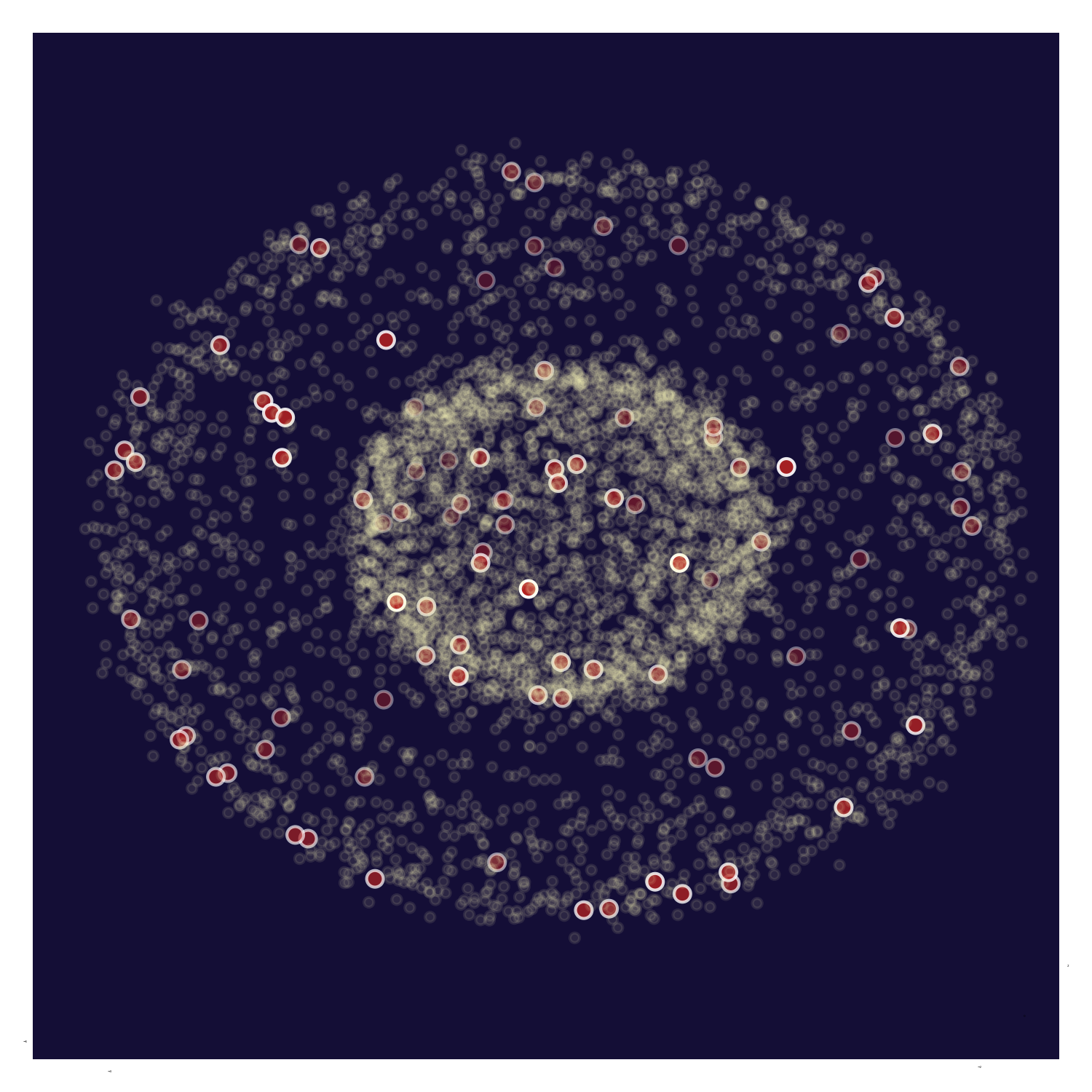}
  \includegraphics[width=.31\linewidth, keepaspectratio]{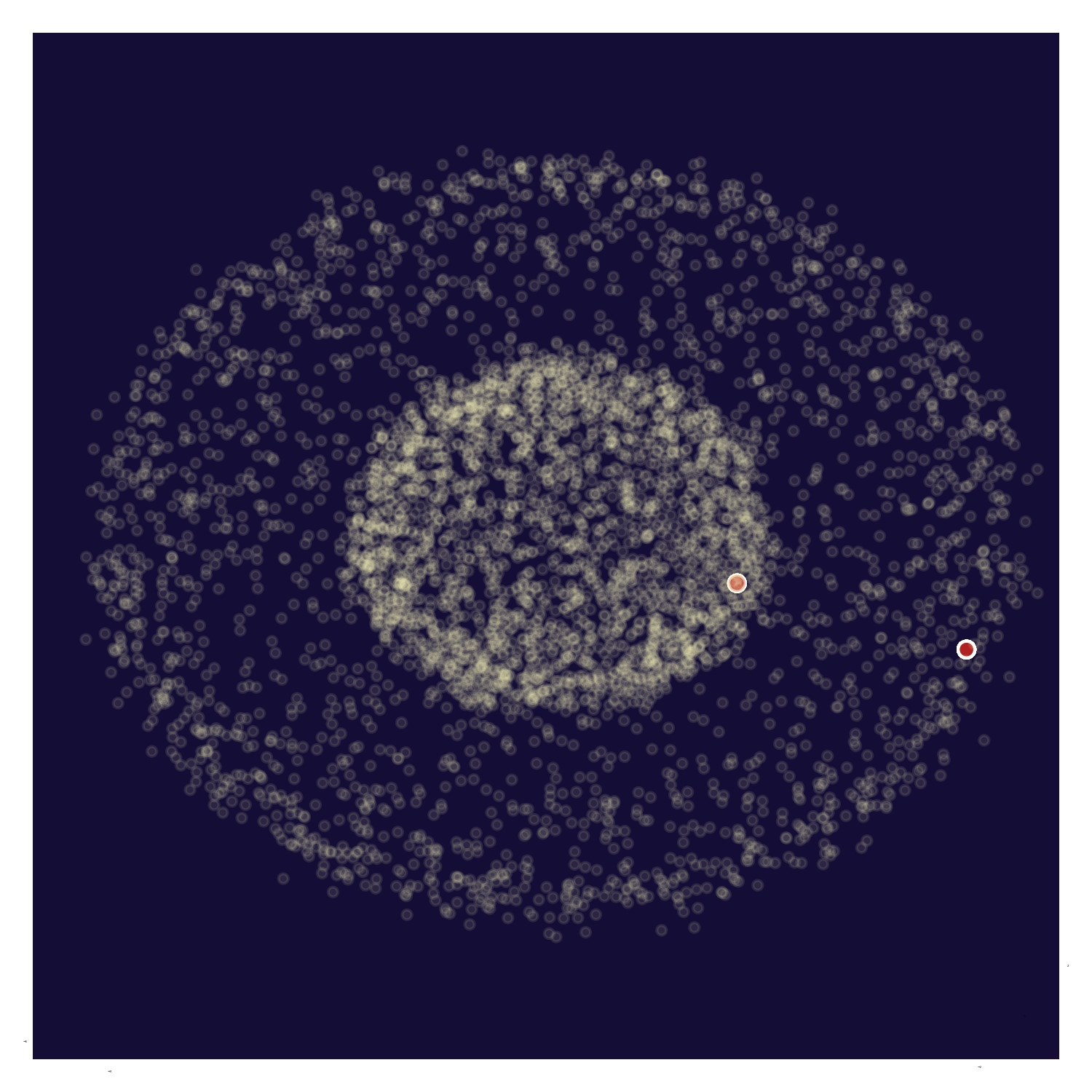}
  \caption{Regular SVGD sampling}
  \label{fig:reg-sph}
\end{subfigure}
\hspace{-0.8cm}
\begin{subfigure}{.5\textwidth}
  \centering
  \includegraphics[width=.31\linewidth, keepaspectratio]{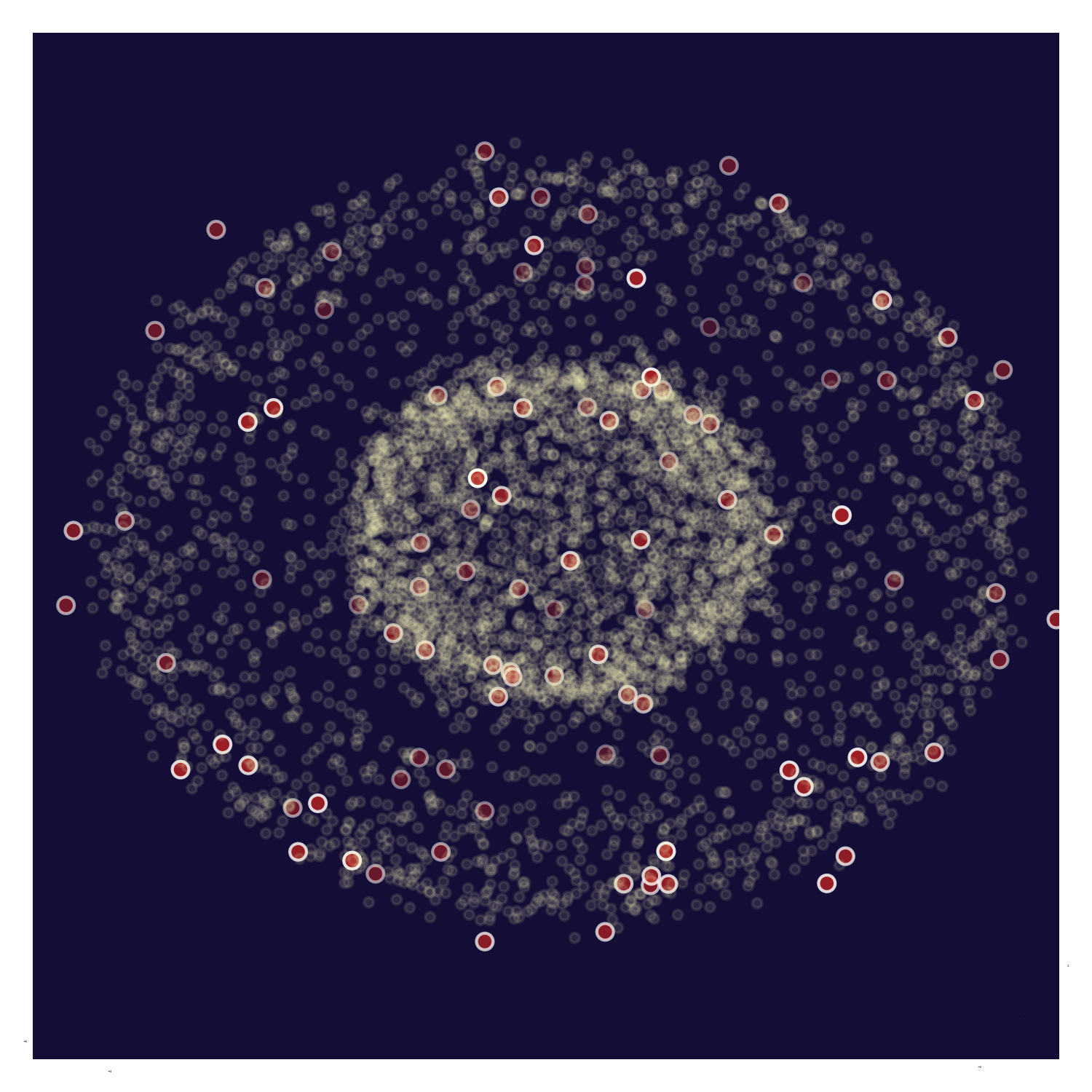}
  \includegraphics[width=.31\linewidth, keepaspectratio]{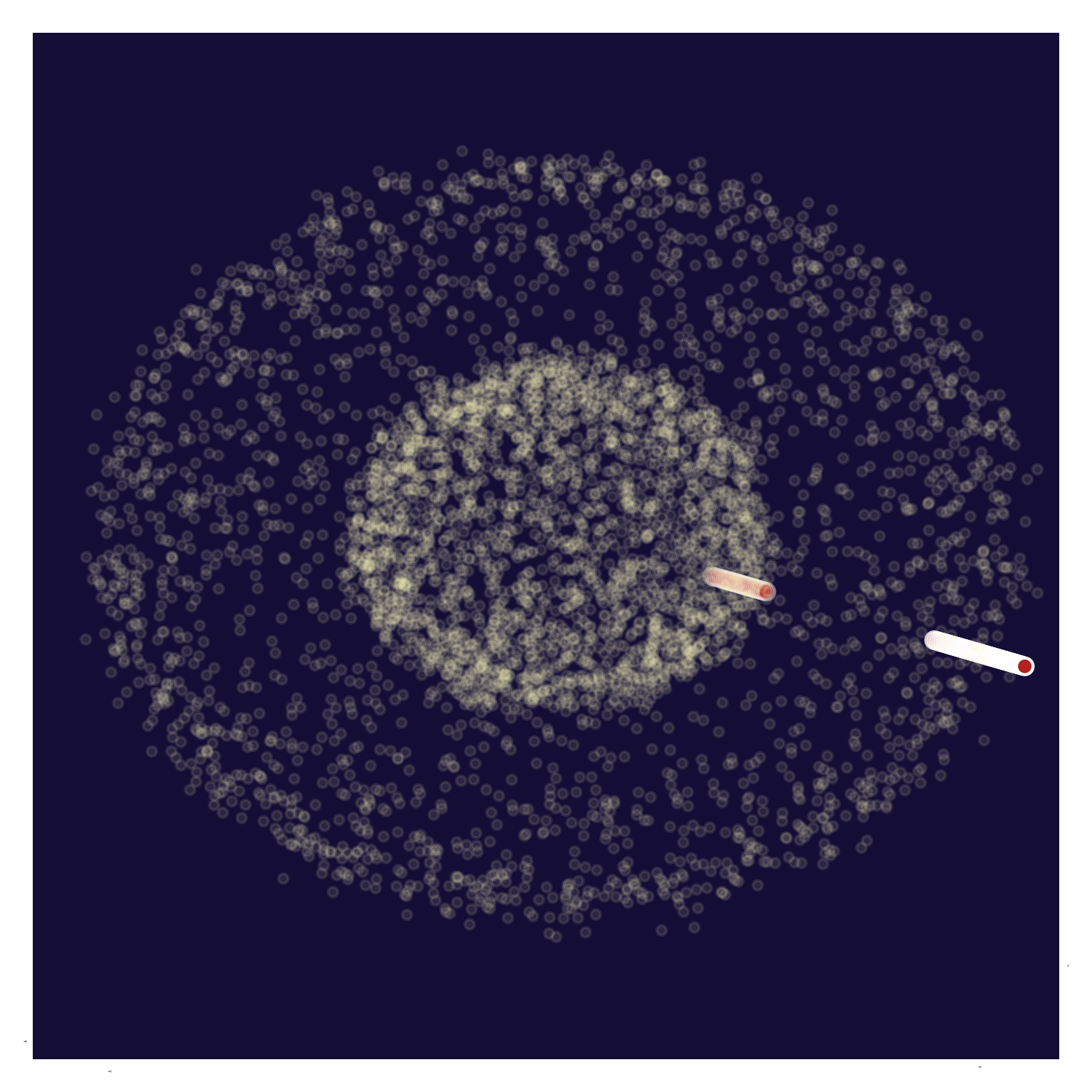}
  \caption{Invariant SVGD sampling}
  \label{fig:inv-sph}
\end{subfigure}
\caption{\emph{Recommended to view in color}. For both regular (\ref{fig:reg-sph}) and invariant (\ref{fig:inv-sph}) SVGD from left to right the original samples, and the samples projected on the group-factorized space are shown. Translucent yellow dots represent the distribution.}
\label{fig:sample-spheres}
\end{figure}
In \Cref{fig:sample-spheres}, we extend the experiments presented in \Cref{sec:esvgd} to a distribution invariant to the $\mathsf{SO}(3)$ symmetry group. The results further support the observations discussed in \Cref{sec:esvgd} that for continuous symmetry groups equivariant SVGD is able to capture the target density more faithfully due to its ability to model long range interactions whereas vanilla SVGD collapses particles to a single orbit representing the high density region in the probability landscape.

\subsection{Experimental setup}
\paragraph{$C_4$-Gaussians} The target $C_4$-Gaussians distribution for which we use equivariant SVGD to draw samples from is defined as a mixture of four Gaussians with uniform mixing coefficients. The mean of each Gaussian is located at a radius of 3 and they all have a covariance of $[1, \frac{1}{5}]$. The 50 starting samples are drawn from a two-dimensional Normal distribution with the mean at the origin and a covariance of $[2, 2]$. The samples are transformed over 25000 SVGD steps with a step size of 0.02. SVGD uses a scalar RBF kernel with a bandwidth of 0.2 for both regular and equivariant sampling.

\paragraph{Concentric circles} The inner and outer circle for the concentric circle example are located at a distance of 4 and 8 from the origin respectively. A normal distribution with variance 0.5 describes the width of the concentric circles. The starting distribution is given by the two-dimensional uniform distribution in the range $[-8, 8)$ for both dimensions. The RBF kernel used for SVGD has a bandwidth of 0.005. All other SVGD settings are kept consistent with the $C_4$-Gaussians example.

\paragraph{Concentric spheres} The concentric spheres toy example is setup in a manner very similar to concentric circle experiment. Specifically, the target distribution is parameterized by the radius of the two spheres and the variance of the Gaussian distribution for the width of the spheres. In this instance, the two spheres have a radius of 4 and 9 and the two Gaussian distributions have a variance of 0.3. Again, similarly the starting distribution is given by uniform distribution in the range $[-8, 8)$. However, this time it is a 3-dimensional distribution. The RBF bandwidth is set to 0.001 and a total of 100 samples are drawn.


\section{Equivariant Joint Energy Models}
Here, we present experimental details for the toy example presented in \Cref{sec:ebm}, additional plots for \Cref{fig:cond-ebm}, as well as an additional toy experiment for training equivariant joint energy model using equivariant Stein variational gradient descent. 
\subsection{Equivariant JEM trained with vanilla SVGD}
\label{app:jem_gaussian}
\begin{figure}
    \centering
    \begin{subfigure}{0.8\textwidth}
    \centering
    \includegraphics[width=0.24\textwidth]{ebm_conditional/true/boundary.png}
    \includegraphics[width=0.24\textwidth]{ebm_conditional/true/joint.png}
    \includegraphics[width=0.24\textwidth]{ebm_conditional/true/class_0.png}
    \includegraphics[width=0.24\textwidth]{ebm_conditional/true/class_1.png}
    \caption{Ground Truth}
    \end{subfigure}
    \begin{subfigure}{0.8\textwidth}
    \includegraphics[width=0.24\textwidth]{ebm_conditional/reg/boundary.png}
    \includegraphics[width=0.24\textwidth]{ebm_conditional/reg/joint.png}
    \includegraphics[width=0.24\textwidth]{ebm_conditional/reg/class_0.png}
    \includegraphics[width=0.24\textwidth]{ebm_conditional/reg/class_1.png}
    \caption{Regular EBM trained with Vanilla SVGD}
    \end{subfigure}
    \begin{subfigure}{0.8\textwidth}
    \includegraphics[width=0.24\textwidth]{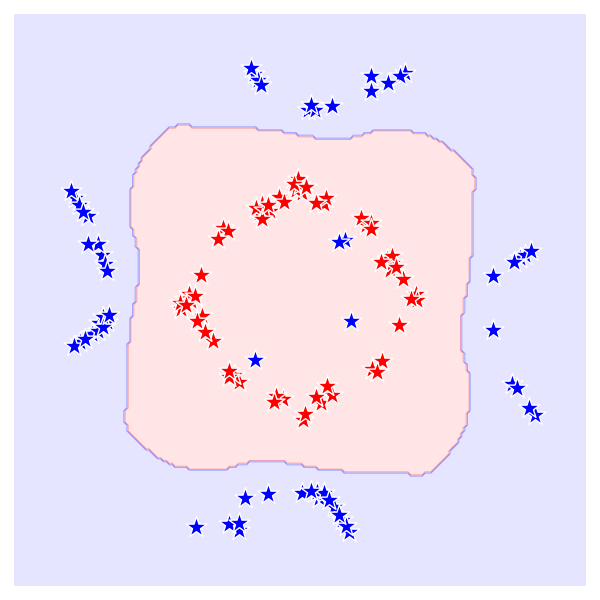}
    \includegraphics[width=0.24\textwidth]{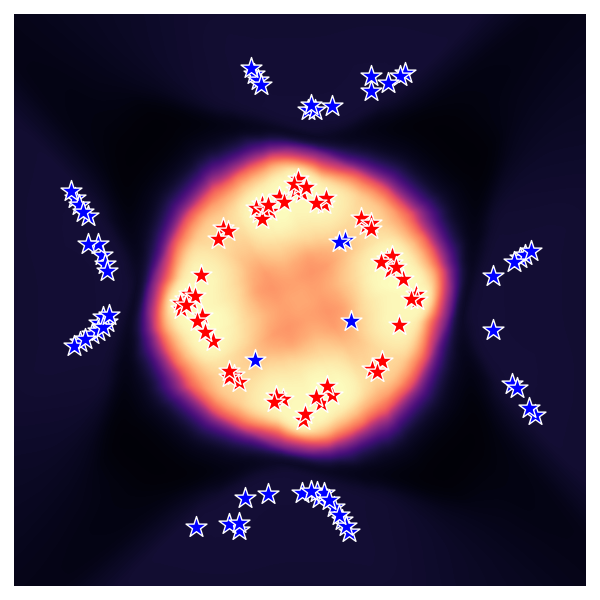}
    \includegraphics[width=0.24\textwidth]{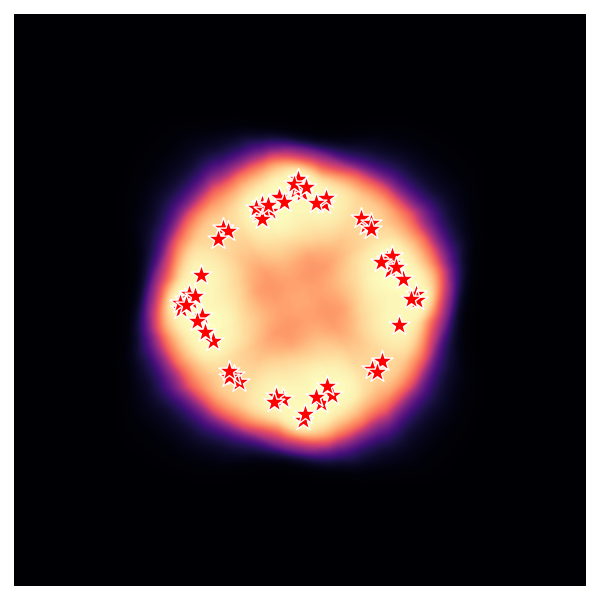}
    \includegraphics[width=0.24\textwidth]{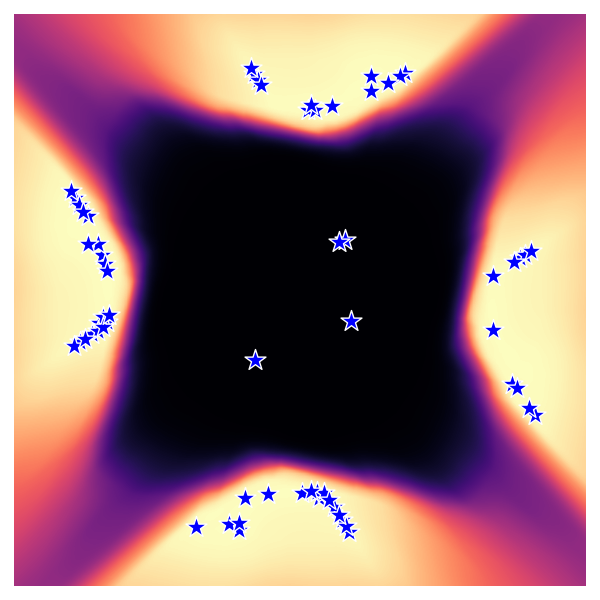}
    \caption{Equivariant EBM trained with Vanilla SVGD}
    \end{subfigure}
    \begin{subfigure}{0.8\textwidth}
    \includegraphics[width=0.24\textwidth]{ebm_conditional/inv/boundary.png}
    \includegraphics[width=0.24\textwidth]{ebm_conditional/inv/joint.png}
    \includegraphics[width=0.24\textwidth]{ebm_conditional/inv/class_0.png}
    \includegraphics[width=0.24\textwidth]{ebm_conditional/inv/class_1.png}
    \caption{Equivariant EBM trained with Equivariant SVGD}
    \end{subfigure}
    \caption{Extended version of \Cref{fig:cond-ebm} with equivariant model and regular SVGD. \textit{From left to right}: decision Boundary, samples and energy of joint distribution $\pi(\vx, y)$, samples and energy of conditional distribution $\pi(\vx|y=0)$, samples and energy of conditional distribution $\pi(\vx|y=1)$. \textit{Row 1:} Target distribution. \textit{Row 2:} Non-equivariant EBM trained with vanilla SVGD. \textit{Row 3:} E-EBM trained with E-SVGD. \textit{Row 4:} E-EBM trained with vanilla SVGD.}
    \label{fig:cond-ebm-cont}
\end{figure}
In \Cref{fig:cond-ebm-cont}, we continue the experiment presented in section \Cref{sec:ebm}. In this figure we provide the results for three models, namely: Regular energy based model trained with vanilla SVGD, equivariant energy based model trained with vanilla SVGD and, equivariant energy based model trained with equivariant SVGD. We find that in comparison with the non-equivariant EBM, an equivariant EBM (irrespective of the sampler used) better approximates the target distribution allowing the model to capture all four modes of the outer distribution while the non-equivariant EBM is unable to spread to those regions. However, compared to the equivariant-EBM trained with equivariant SVGD, we find that an equivariant-EBM trained with vanilla SVGD requires more training steps to reconstruct areas of low probability. This is due to the equivariant SVGD exploring a wider area of the landscape in the negative samples for the contrastive divergence algorithm since vanilla SVGD  generates multiple samples in the same orbit due to only being able to capture local interactions.

\subsection{Additional JEM concentric circles}
\begin{figure}[t]
    \centering
    \begin{subfigure}{0.2\textwidth}
        \includegraphics[width=0.8\textwidth]{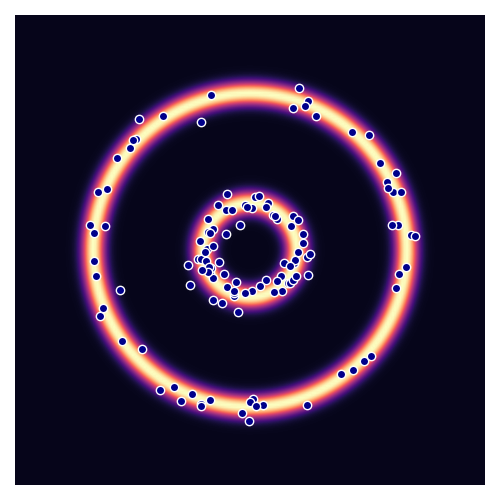}
        \includegraphics[width=0.8\textwidth]{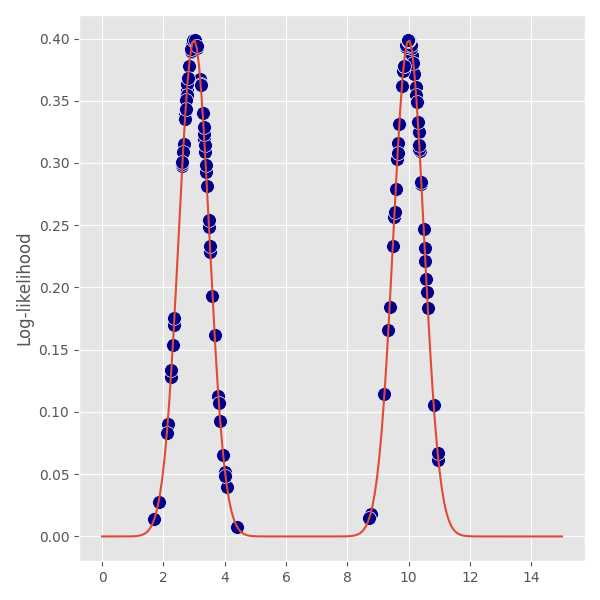}
        \label{fig:ebm_circ_target}
    \end{subfigure}
    \begin{subfigure}{0.2\textwidth}
        \includegraphics[width=0.8\textwidth]{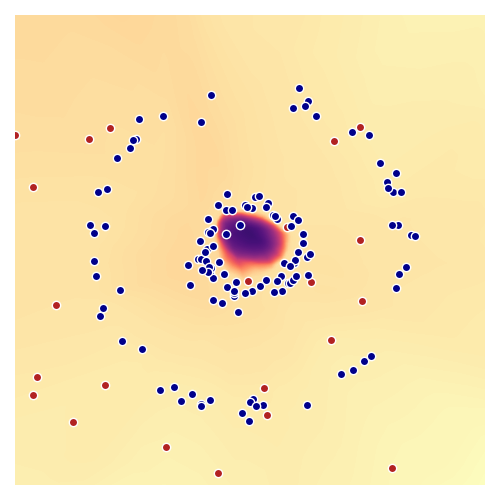}
        \includegraphics[width=0.8\textwidth]{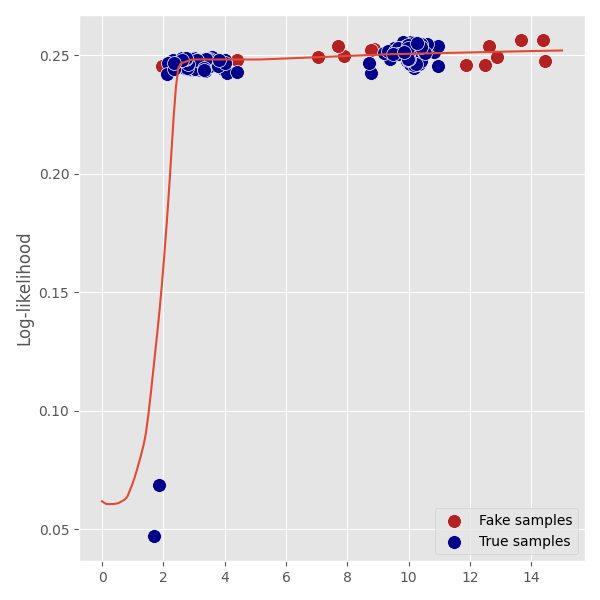}
        \label{fig:ebm_circ_reg}
    \end{subfigure}
    \begin{subfigure}{0.2\textwidth}
        \includegraphics[width=0.8\textwidth]{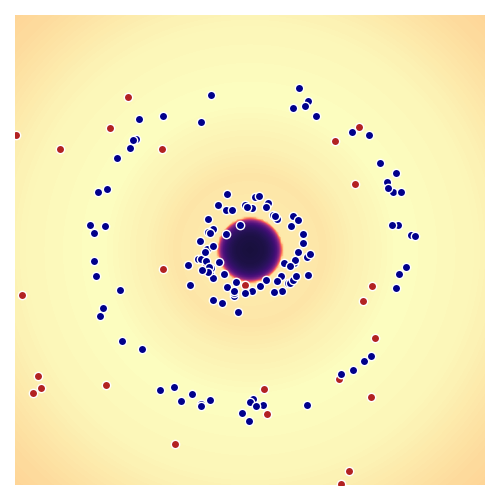}
        \includegraphics[width=0.8\textwidth]{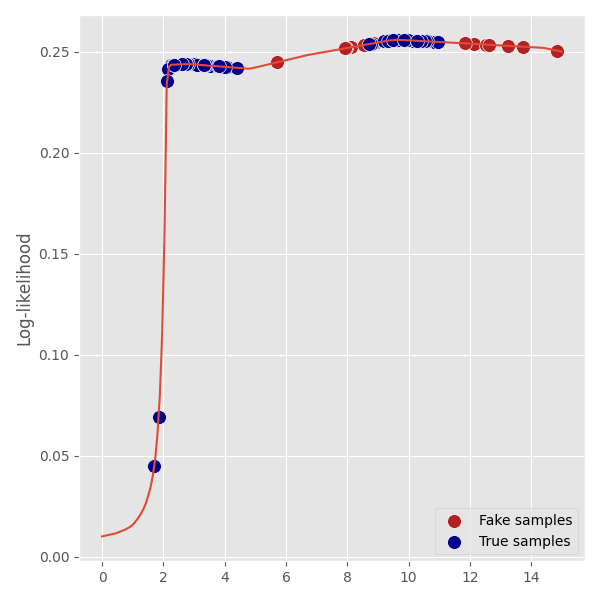}
        \label{fig:ebm_circ_reg_inv}
    \end{subfigure}
    \begin{subfigure}{0.2\textwidth}
        \includegraphics[width=0.8\textwidth]{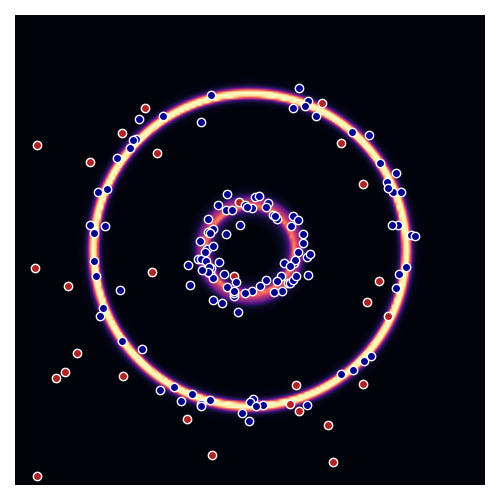}
        \includegraphics[width=0.8\textwidth]{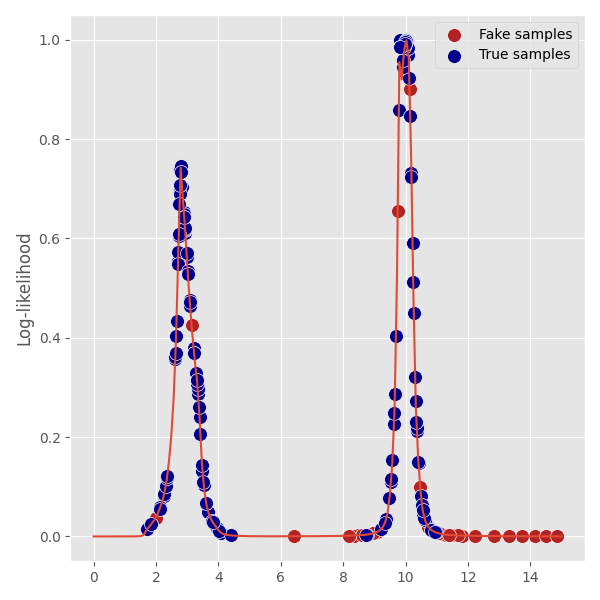}
        \label{fig:ebm_circ_inv}
    \end{subfigure}
    \caption{Visualization of (learned) distribution by EBMs trained concentric circles. Red dots represent the samples sampled during the last contrastive divergence step. Blue dots are the true samples used for training. \textit{Row 1:} Two-dimensional visualization. \textit{Row 2}: Samples projected on group-factorized spaces. \textit{From left to right}: Target distribution, non-Equivariant EBM trained with vanilla SVGD, equivariant EBM trained with vanilla SVGD, equivariant EBM trained with equivariant SVGD.}
    \label{fig:ebm_circ}
\end{figure}

In \Cref{fig:ebm_circ}, we present the results of training an EBM on samples drawn from the concentric circles toy distribution. This is done for the same three combinations of EBM and SVGD sampler as before: a non-equivariant EBM trained with regular SVGD, an equivariant EBM trained with regular SVGD and an equivariant EBM trained with equivariant SVGD. We find that the trained EBMs show the same results as observed in the JEM trained on the $C_4$-Gaussians. A non-equivariant EBM is by far the worst of the three combinations. It specifically has a hard time in reconstructing the outer-ring. Stepping up to a equivariant EBM does improve on this aspect as a slight uptick can be seen at the location of the outer-ring. However, to also fully capture the areas of low probability, the regular SVGD has to be replaced by our equivariant version. 

\subsection{Experimental setup}
The experimental setup for the experiments with regular EBMs trained using vanilla SVGD consists of three parts: 1) defining the target distribution and the construction of the training dataset, 2) defining the energy model and the training parameters and, 3) defining the SVGD kernel and the sampling parameters. Note that for both the $C_4$-Gaussians and the concentric rings experiment the setup is kept consistent between the three combinations of EBM and SVGD method. Constructing the equivariant representations for the equivariant EBM models does not add additional parameters to the models. 

\paragraph{$C_4$ Mixture of Gaussians} 
The target distribution is defined as two sets of $C_4$-Gaussians, one at a distance of 7 from the origin and the other at a distance of 15. The inner-set of four Gaussians represents the distribution of the first class (\ie $\pi(\vx|y=0)$) while the outer-set represents the second class (\ie $\pi(\vx|y=1)$). Using this definition, both the class-conditional probabilities as well as the joint distribution are invariant to the $C_4$ symmetry group. The dataset used for training the energy model contains of 128 samples equally divided amongst the two classes. 

The regular EBM is defined as a 6 layer MLP with ReLU activation functions. The layers have 32, 64, 64, 64, 32, and 2 output nodes respectively with an input dimension as 2. The energy model is trained over 500 epochs with a fixed learning rate of 0.001 using the Adam \citep{kingma2014adam} optimizer with a batch size of 32.

The SVGD kernel used for sampling the 32 negative samples for the contrastive divergence step uses an RBF kernel with a bandwidth of 0.1. Each sampling step does 10,000 steps of SVGD with a step size of 0.9. We consider the SVGD to have converged if the norm of the update of the samples between two consecutive SVGD steps is less then $10^{-4}$. Additionally, the sampling uses persistence \citep{tieleman_training_2008} with a 0.05 probability of resetting. When reset, new SVGD starting samples are drawn from the positive samples in the dataset. Furthermore, the positive samples in the dataset are used as additional repulsive forces by concatenating them to the batch of negative samples for calculating the update step in \Cref{eq:svgd}. 

\paragraph{Concentric circles}
The concentric circles of the target distribution are located at a distance of 3 and 10 from the origin with equal mixing coefficients. This target distribution is used to sample a dataset with 128 training samples. 

Both the energy function and most of the training setup are kept similar as for the $C_4$-Gaussians experiment. However, for this experiment a learning rate scheduler is used that reduces the learning rate at 150 and 400 epochs to 0.0005 and 0.0001 respectively. Additionally the Mean Square Error (MSE) loss is used as an additional supervision signal during training to demonstrate the use of loss given in \Cref{eq:loss-JEM}. The MSE loss is weighted by a factor of 0.5. 

Negative samples are drawn using 10,000 SVGD steps with a bandwidth of 0.05. Additionally, we use a scheduler for the step-size for the Stein variational gradient descent algorithm which reduces the step-size at epoch 250 and 400 to 0.5 and 0.1 respectively. The use of persistence, its reset, and additional supervised SVGD repulsive forces is consistent with the $C_4$-Gaussians experiment.


\section{Experiments}
\label{app:exp}

In this section, we present the details for our experiments presented in \Cref{sec:exp}. 

\subsection{Many-body Particle System (DW-4)}
\label{app:dw4}
\textbf{Experimental Setup:} The dataset for the Double-Well with 4 particles (DW-4) experiment is constructed by sampling a single example from each of the five meta-stable state configurations for the DW-4 potential. Each of these samples is then duplicated 200 times for a total of a 1000 training samples. A small amount of Gaussian noise is added to each sample to make them unique. 

The EBM used to reconstruct the potential is parameterized by a 3 layer MLP with 64, 64, and 1 output nodes respectively. The input of the MLP is 8-dimensional (one for each coordinate of the 4 particles). Except for the final layer the ReLU activation function is used after each layer. For the final layer we use the activation function $\log(1+x^2)$, where $x$ is the output of the final layer. The EBM is trained over 50 epochs using the Adam optimizer and a fixed learning rate of 0.01. The batch-size is 64.

We use scalar RBF kernels for all SVGD variant but depending on the type of SVGD used we use a different kernel bandwidth for the DW-4 experiment. Specifically, for regular SVGD we use a bandwidth of 0.1 and for equivariant SVGD we use 0.001. The influence of the RBF kernel bandwidth on the final results is discussed more in the next section. For each batch of negative samples, we evolved SVGD for 5,000 time-steps with a step size of 0.1 using the dataset samples as repulsive force. Persistence was used with a reset probability of 0.10. When reset, the starting coordinates for the DW-4 particles are independently sampled from the uniform distribution in the range $[-5, 5)$. 

\textbf{Differences with \cite{kohler2020equivariant}:} As mentioned in the main paper, the experiment performed using the DW-4 is a slight deviation from an earlier proposed experiment using the same DW-4 potential in \cite{kohler2020equivariant}. In this earlier work, the authors propose to investigate the capacity of their proposed density estimation method (equivariant flows) to recover unseen meta-stable states of the potential when given only access to a single meta-stable states. To clarify, we refer to all possible local minima of the potential that are equivalent under rotation, translation and permutation symmetry as being a single distinct meta-stable states. Using this definition there are a total of 5 distinct meta-stable states (see \Cref{fig:dw4}) for the DW-4 potential. 

While the results presented in the original paper shows great success, it is however our understanding that the presented results can not be due to the proposed equivariance constraint on the normalizing flow. Precisely, the equivariant flow density proposed by the authors is invariant with respect to permutation of particles, rotation of the system of particles around its center of mass, and translation of the entire system. Thus, given only access to one of the 5 distinct meta-stable states, the equivariant flow only learns to assign a high probability to particle configurations that belong to this same meta-stable states. The 4 other distinct meta-stable states can not be recovered from the single presented meta-stable state using the symmetry transformations that the flow is invariant to. As the density estimated by the EBM trained using our equivariant SVGD proposed in this work is also only invariant to these same symmetries it would face the same restrictions. 
Instead, we believe that the method recovers the other 4 meta-stable state primarily due to the additional spread we observed in the equivariant sampling process (see \Cref{sec:esvgd}) and the addition of the Gaussian noise to the training example. By adding this small amount of noise, the density estimator can not collapse its density into the single training example. As a result, every possible configuration of the four particles ultimately has a small non-zero probability. Given enough samples, there is herefore a non-trivial chance of sampling other meta-stable states as well. The equivariant constraint on the normalizing flow further amplifies this as the spread of non-zero probabilities not only occurs around the single training example, but also around all symmetry transformations of it. 

The hypothesis of the previous paragraph is substantiated by the experimental results presented in \Cref{fig:DW4-Kohler}. We find that if we train an equivariant EBM using equivariant SVGD on the dataset supplied in \cite{kohler2020equivariant} (see \Cref{fig:DW4-Kohler-data}) instead of the one we constructed ourselves, the sampling procedure can be forced to replicate the same results as presented by \cite{kohler2020equivariant}. If we use equivariant SVGD to sample from the trained EBM with the RBF kernel bandwidth set too high, the variance in the samples becomes too large. Thus, with a sufficiently large number of samples, searching through these samples reveals that all distinct meta-stable states have accidentally been recovered. When we significantly reduce the bandwidth, the same number of samples can be drawn but the variance will be low enough to only recover symmetry transformations of the original meta-stable state in the dataset. The later is the expected results given the set of symmetries the estimated densities are invariant too. 

\begin{figure}[t]
    \begin{subfigure}{\textwidth}
      \centering
      \includegraphics[width=.15\linewidth, keepaspectratio]{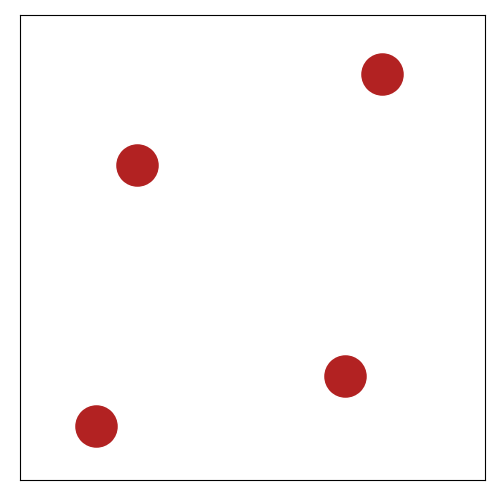}  
      \includegraphics[width=.15\linewidth, keepaspectratio]{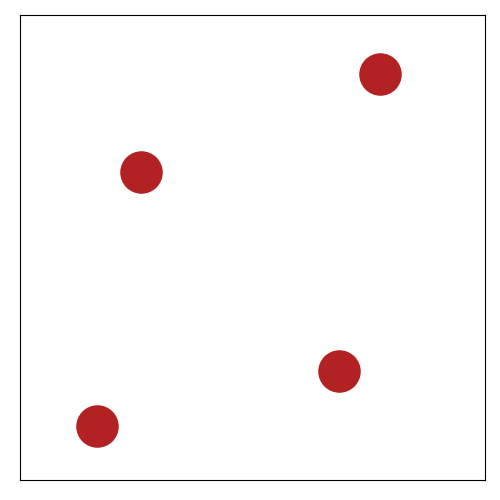}  
      \includegraphics[width=.15\linewidth, keepaspectratio]{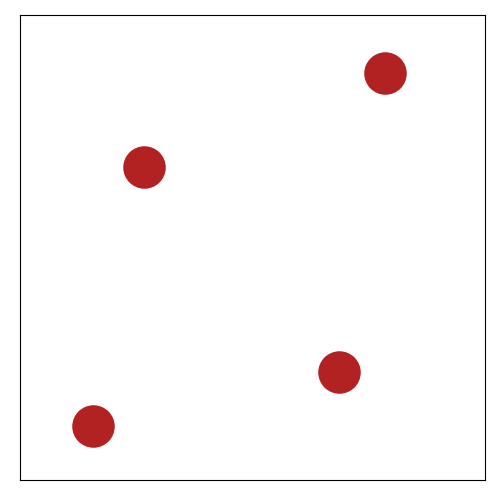}  
      \includegraphics[width=.15\linewidth, keepaspectratio]{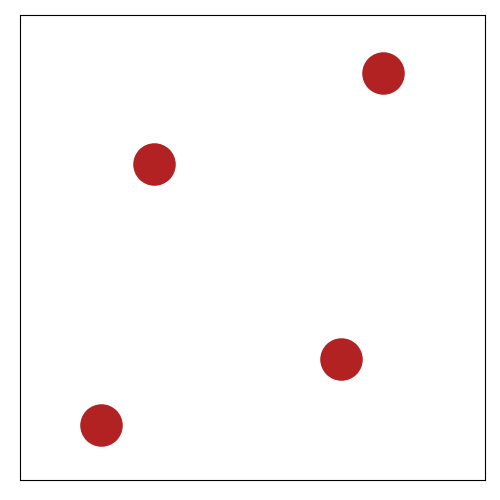}  
      \includegraphics[width=.15\linewidth, keepaspectratio]{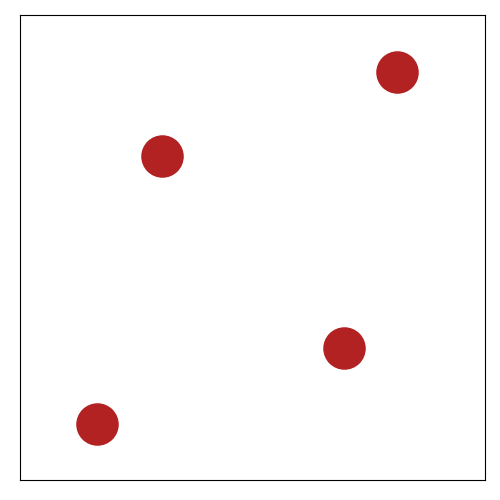}
    \end{subfigure}
  \caption{Training samples from the original dataset provided in \cite{kohler2020equivariant}. All samples represent the same meta-stable state and only differ by the addition of some Gaussian noise to the particle locations.}
    \label{fig:DW4-Kohler-data}
\end{figure}
  
\begin{figure}[t]
    \begin{subfigure}{\textwidth}
      \centering
      \includegraphics[width=.15\linewidth, keepaspectratio]{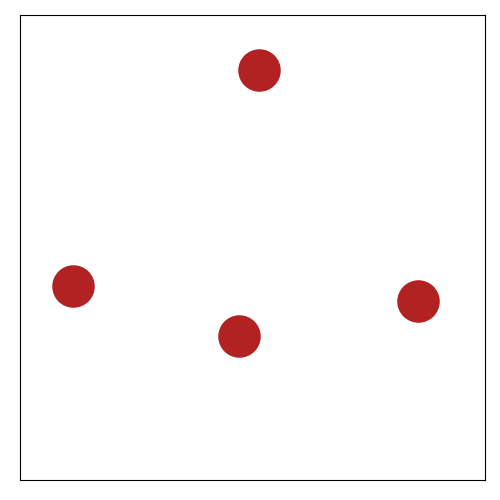}
      \includegraphics[width=.15\linewidth, keepaspectratio]{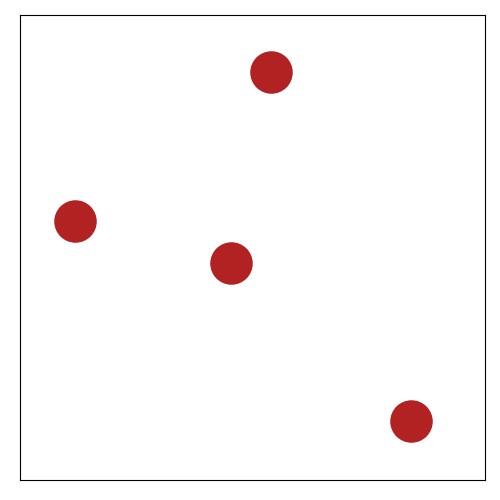}  
      \includegraphics[width=.15\linewidth, keepaspectratio]{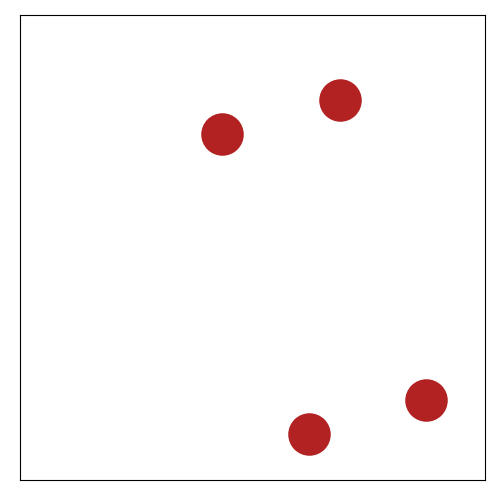}  
      \includegraphics[width=.15\linewidth, keepaspectratio]{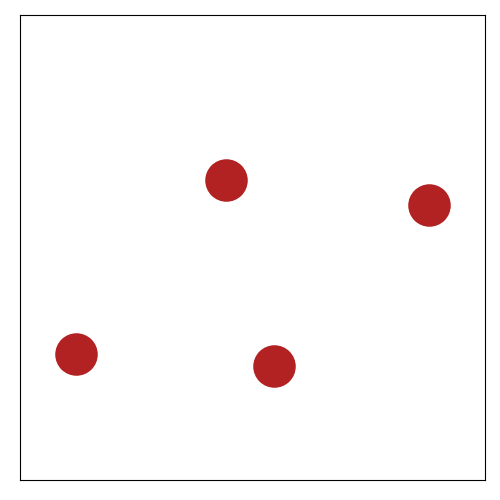}  
      \includegraphics[width=.15\linewidth, keepaspectratio]{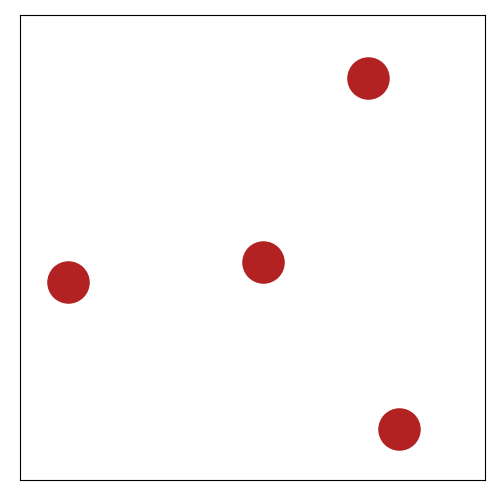}
      \caption{High bandwidth}
    \label{fig:DW4-Kohler-high}
    \end{subfigure}
    \begin{subfigure}{\textwidth}
      \centering
      \includegraphics[width=.15\linewidth, keepaspectratio]{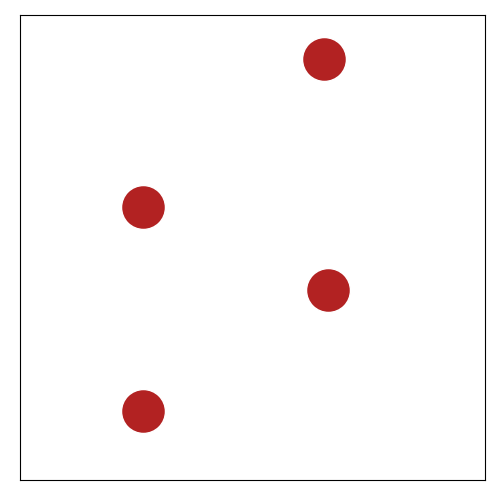}  
      \includegraphics[width=.15\linewidth, keepaspectratio]{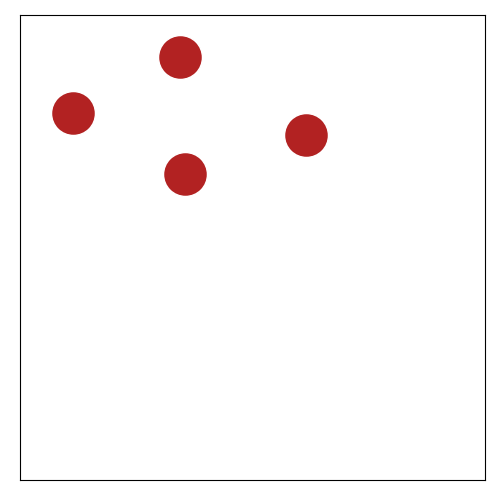}  
      \includegraphics[width=.15\linewidth, keepaspectratio]{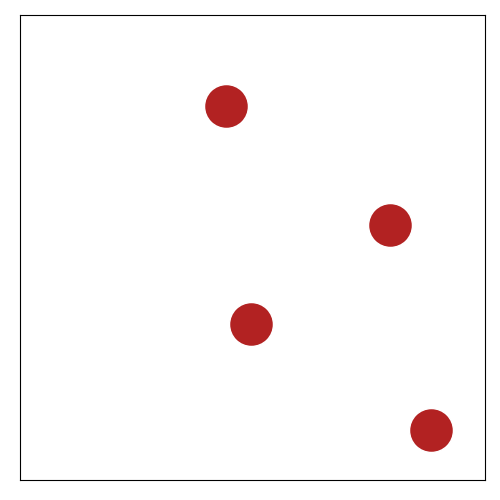}  
      \includegraphics[width=.15\linewidth, keepaspectratio]{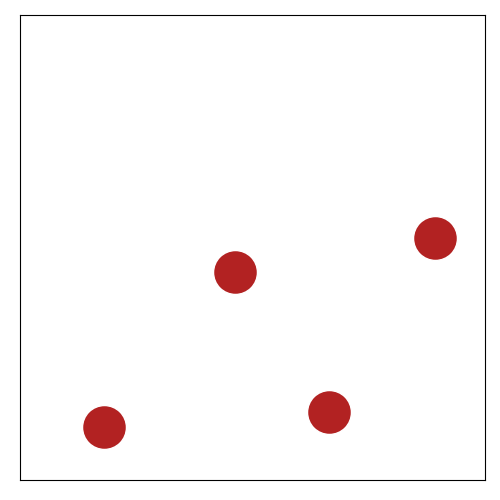}  
      \includegraphics[width=.15\linewidth, keepaspectratio]{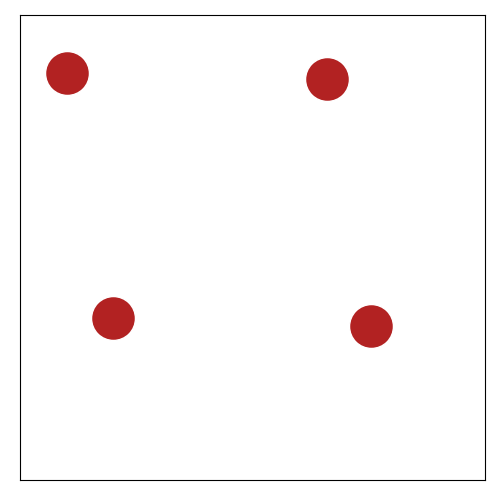}
      \caption{Low bandwidth}
    \label{fig:DW4-Kohler-low}
    \end{subfigure}
  \caption{Results of sampling using equivariant SVGD from an equivariant EBM trained with the dataset provided in \cite{kohler2020equivariant}. In \cref{fig:DW4-Kohler-low} a bandwidth of 0.0001 for the RBF-kernel is used while in \cref{fig:DW4-Kohler-high} a bandwidth of 2.0 is used. Note, sampling is done from the same EBM trained using equivariant SVGD with a bandwidth of 0.001.}
\label{fig:DW4-Kohler}
\end{figure}



\subsection{Molecular Generation using QM9}
\label{app:qm9}
For the molecular generation experiments we limit the large QM9 dataset to constitutional isomers of \ch{C5H8O1}. This molecule was selected due to the relatively high number of constitutional isomers (35) in the dataset in combination with its low atomic charge (46). This allows for sufficient variation within the dataset while keeping the molecules small enough to easily visualize and interpret. 

The equivariant EBM is created using an Equivariant Graph Neural Network E-GNN \citep{satorras2021n} with 4 Equivariant Graph Convolutional Layers. Each layer has 64 units. All other model configurations are kept consistent with those deployed in \cite{satorras2021n} for the same dataset. The model was trained over 2500 epochs with a step wise learning rate scheduler. We start with a learning rate of 0.01 and reduce it to 0.005 and 0.001 at epochs 250 and 1,000 respectively. We used an RBF kernel with a bandwidth of 1 for equivariant SVGD which was evolved for 2,500 time-steps with a step size of 0.5. The persistent samples were reset with a 0.2 probability. 
\paragraph{Bond estimation}
Before visualizing the generated structures, we first post-process them to infer the bonds between the atoms. We do this by using the distance between atoms as a proxy for their probability of being bonded. In the following, we will describe this process in detail. We will rely on graph terminology where atoms are represented by nodes and bonds as undirected edges. The three steps for this process can be roughly described as: 1) bond all heavy atoms together such that they form a connected graph, 2) bond each hydrogen atom to one of the heavy atoms and, 3) create new bonds between atoms or double-up bonds between already connected atoms until each heavy atom has the required number of bonds. 

As stated, the goal of the first step is to form a connected graph containing all the heavy atoms. We use the atom closest to the origin as the starting graph consisting of only this single node. From there on, we continuously add the atom that is closest to any of the atoms already in the connected graph. The atom newly added to the connected graph and the atom it was closest to are then connected by an edge/bond such that the graph maintains its connectivity. Note that when determining the distance to the connected graph for each atom, we only consider the distance to atoms that still have bonds available. 

In the second step, we connect all hydrogen atoms to the connected graph of heavy atoms. We order this process based on the distance of each hydrogen atom to its closest heavy atom in the connected graph. In other words, we first calculate the distance from each hydrogen atom to all heavy atoms in the connected graph. Given all these distances we then iteratively bond the hydrogen atom that is furthest away from its closest heavy atom in the graph that still has bonds available to this heavy atom. After each newly connected atom pair, we update the number of bonds available for each heavy atom before continuing. 

In the last step, we spend all remaining bonds available for the heavy atoms. To do this we repeat the following process until all bonds are spend. First, find the pair of heavy atoms that still have bonds left that are closest together. Second, find the closest pair of atoms that have bonds left and are not yet connected. If the distance between the second pair is not further than 1.1 times the distance between the first pair, then bond the second pair. Otherwise, connect the first pair. 

\begin{figure}[t]
\centering
      \includegraphics[width=.10\linewidth, keepaspectratio]{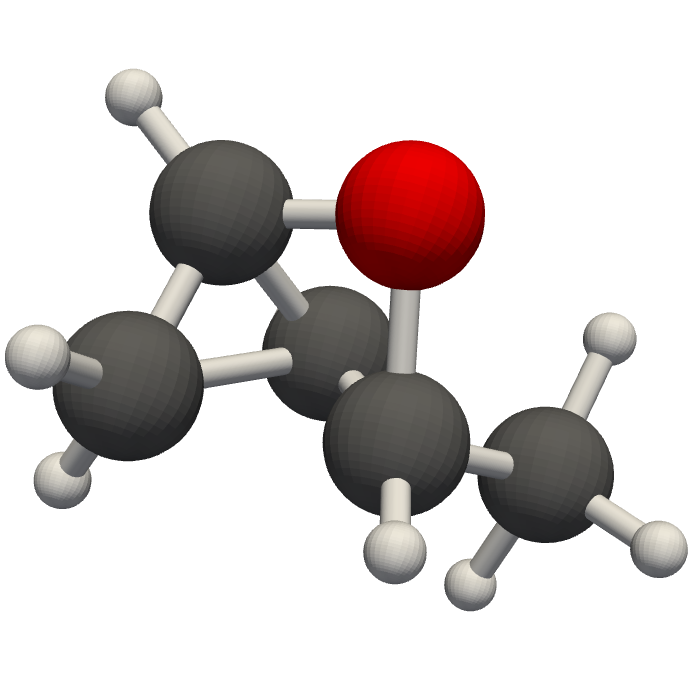}
      \includegraphics[width=.10\linewidth, keepaspectratio]{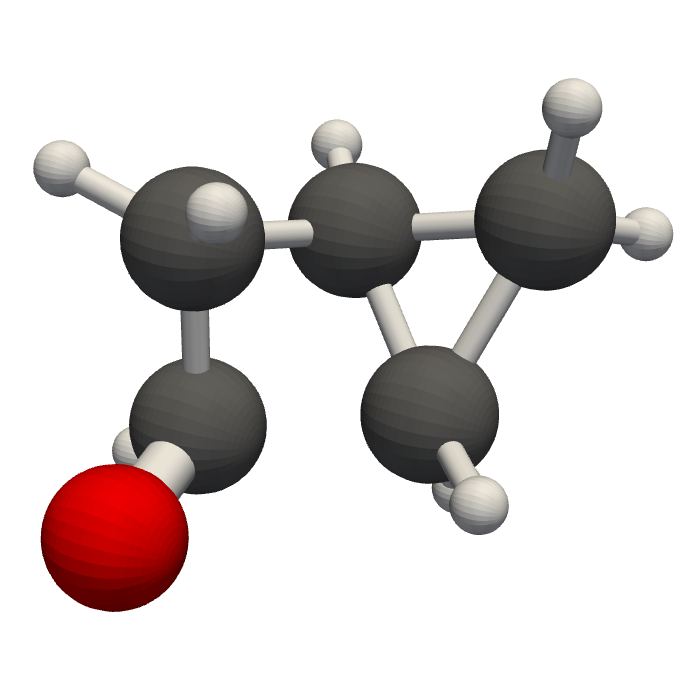}
      \includegraphics[width=.10\linewidth, keepaspectratio]{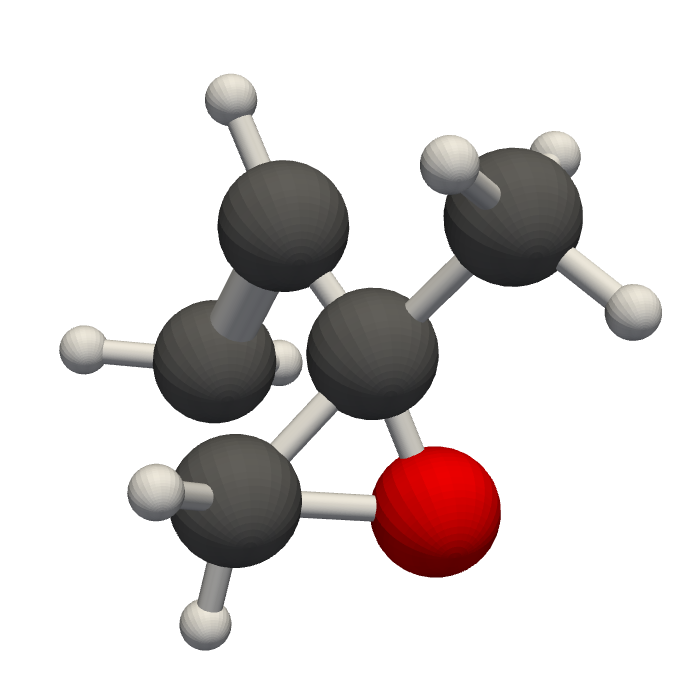}
      \includegraphics[width=.10\linewidth, keepaspectratio]{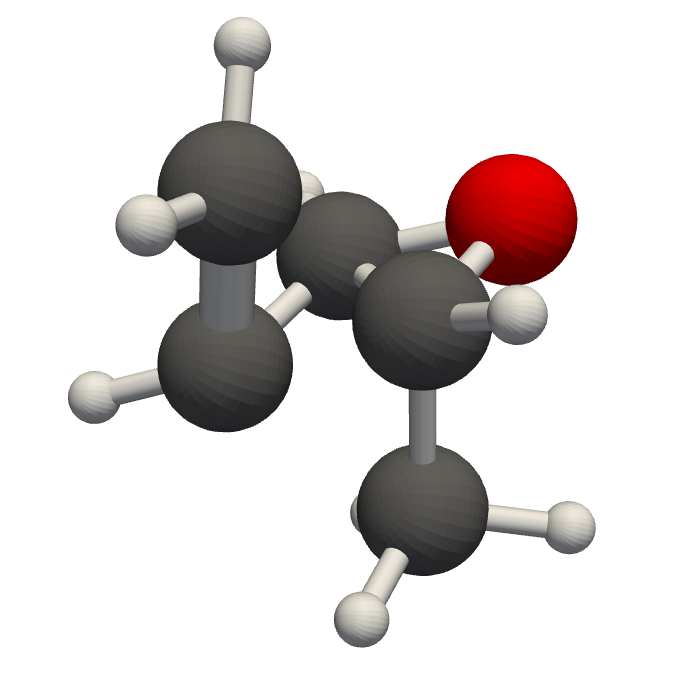}
      \includegraphics[width=.10\linewidth, keepaspectratio]{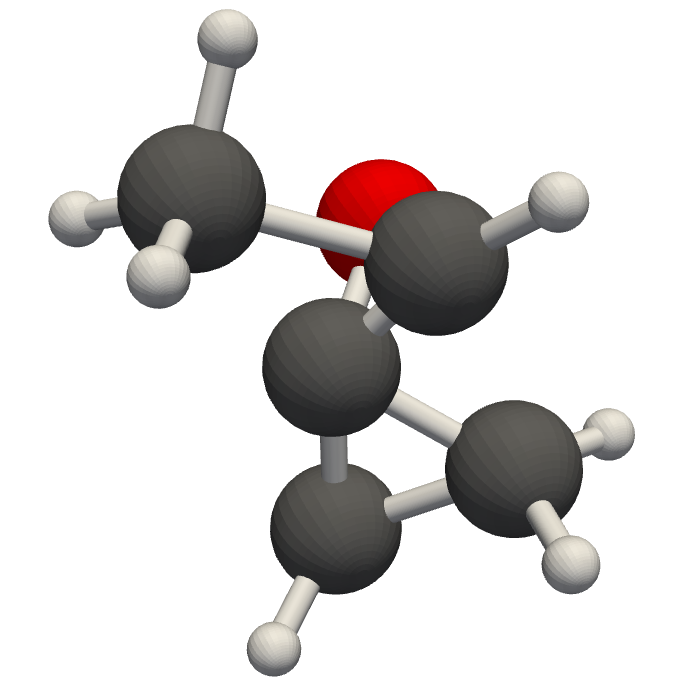}
      \includegraphics[width=.10\linewidth, keepaspectratio]{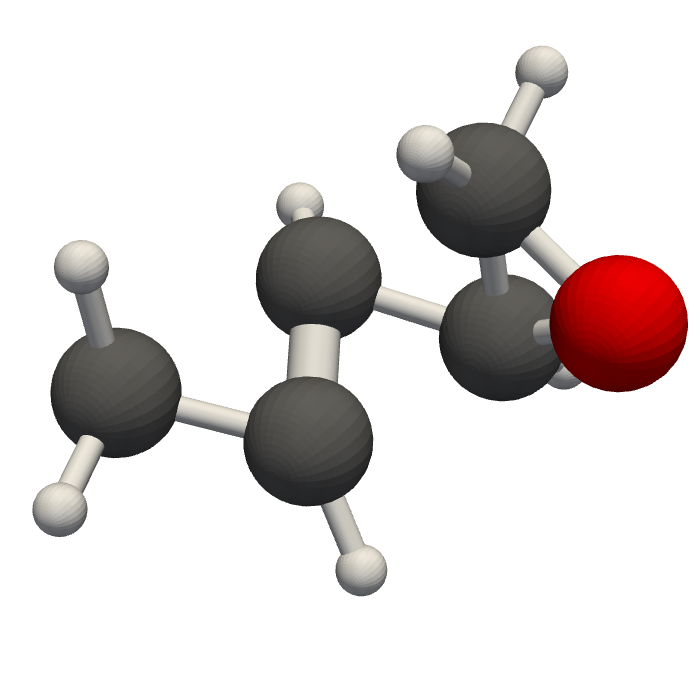}
      \includegraphics[width=.10\linewidth, keepaspectratio]{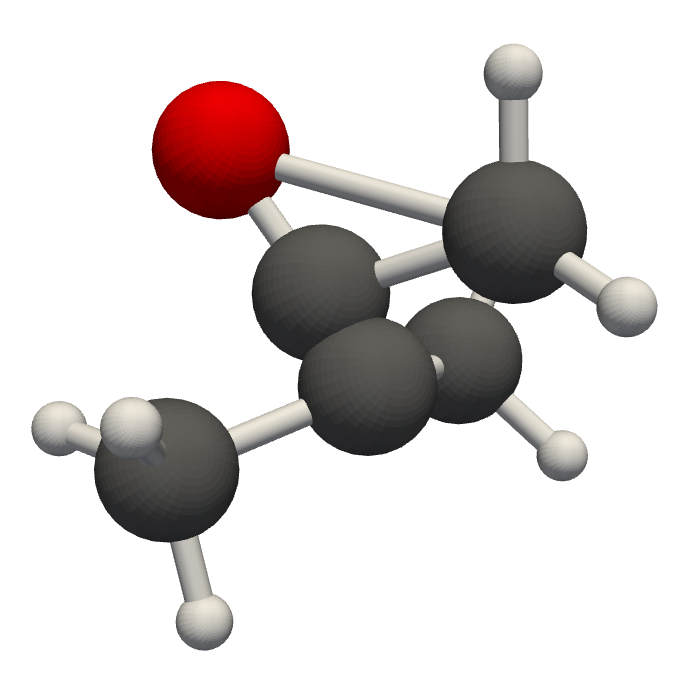}
      \includegraphics[width=.10\linewidth, keepaspectratio]{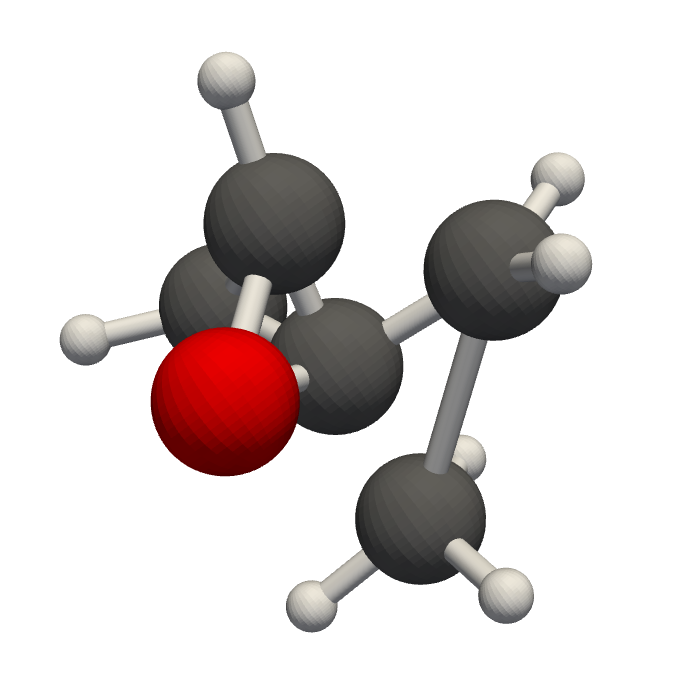}
      \newline
      \includegraphics[width=.10\linewidth, keepaspectratio]{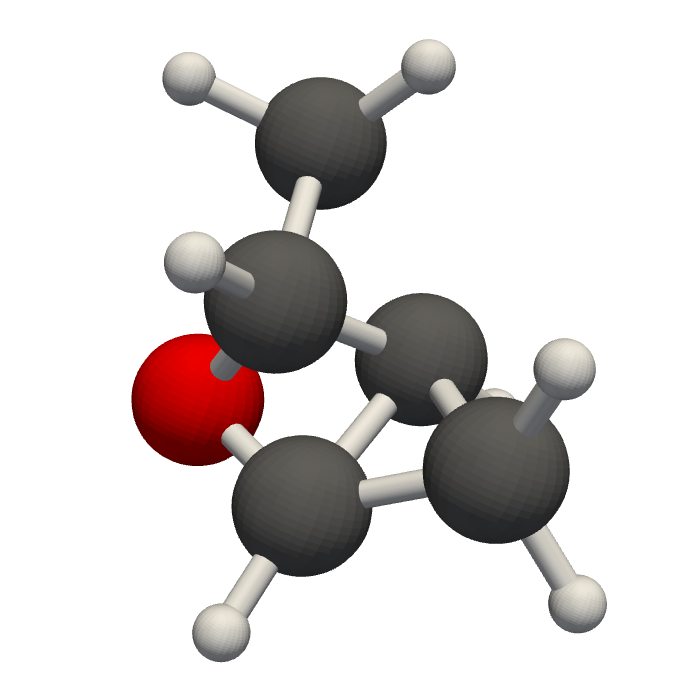}
      \includegraphics[width=.10\linewidth, keepaspectratio]{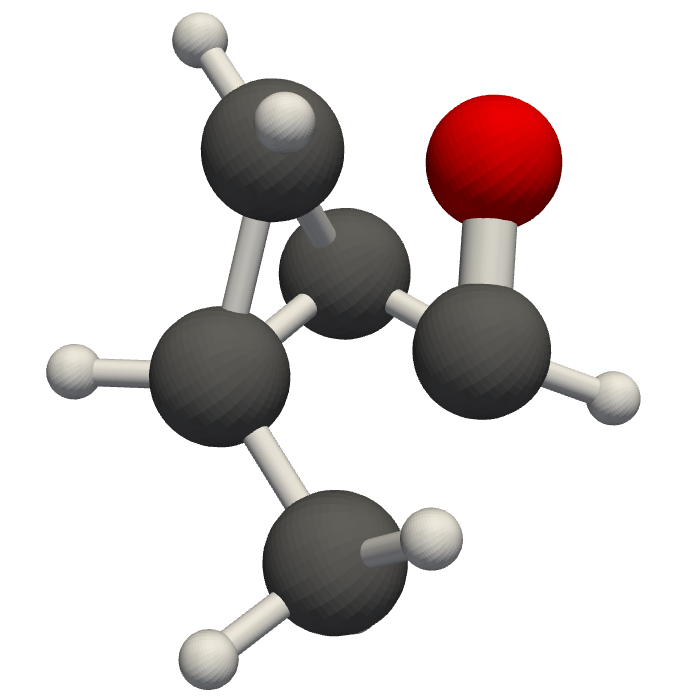}
      \includegraphics[width=.10\linewidth, keepaspectratio]{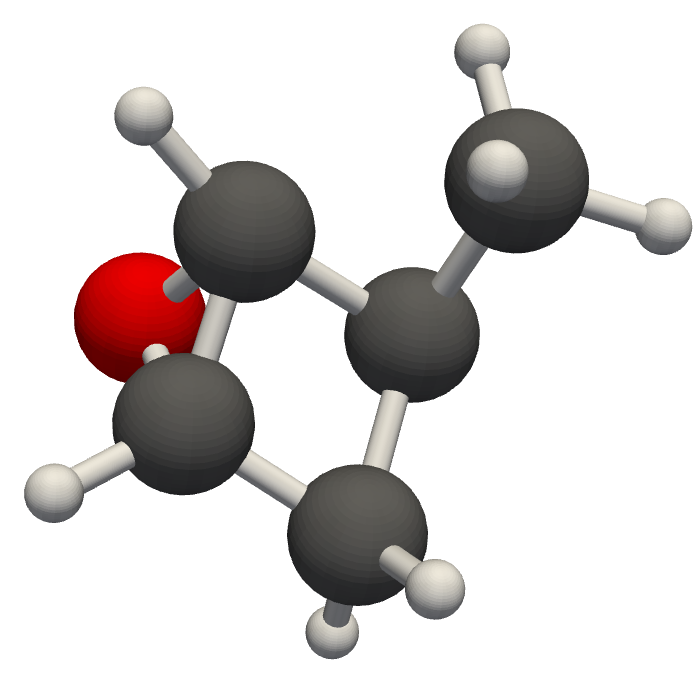}
      \includegraphics[width=.10\linewidth, keepaspectratio]{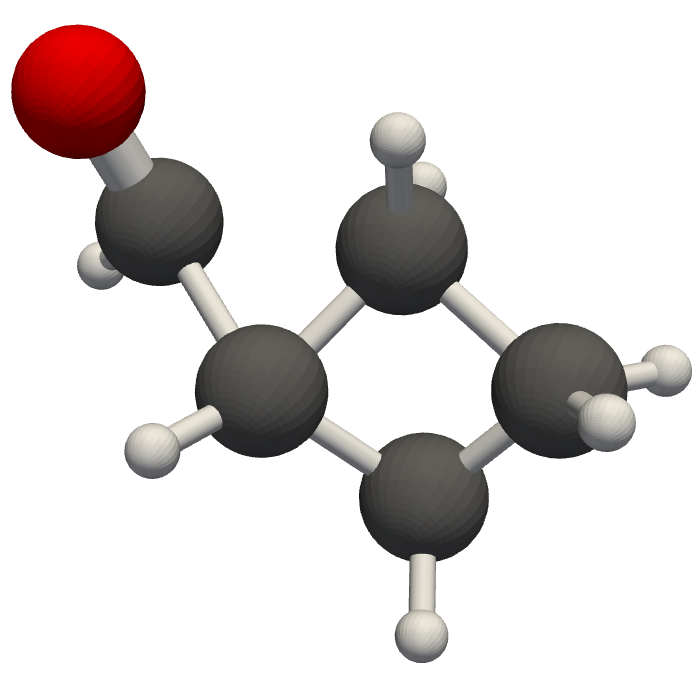}
      \includegraphics[width=.10\linewidth, keepaspectratio]{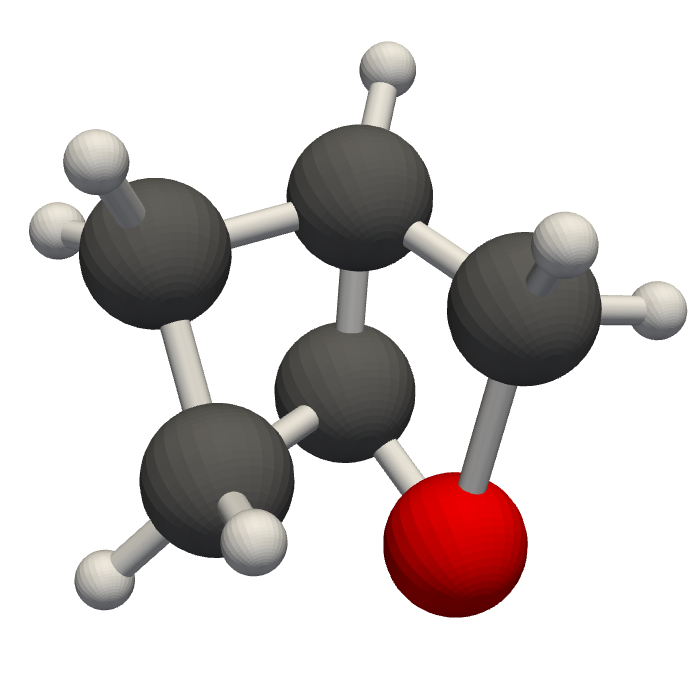}
      \includegraphics[width=.10\linewidth, keepaspectratio]{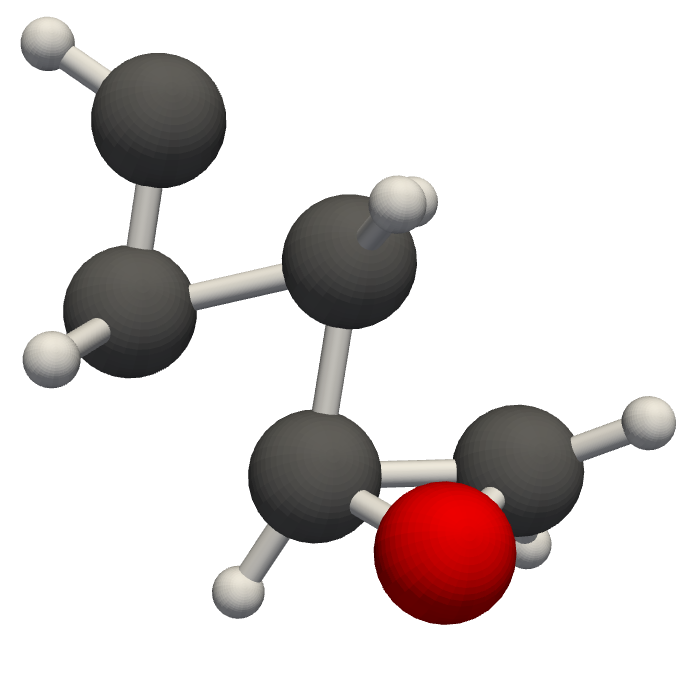}
      \includegraphics[width=.10\linewidth, keepaspectratio]{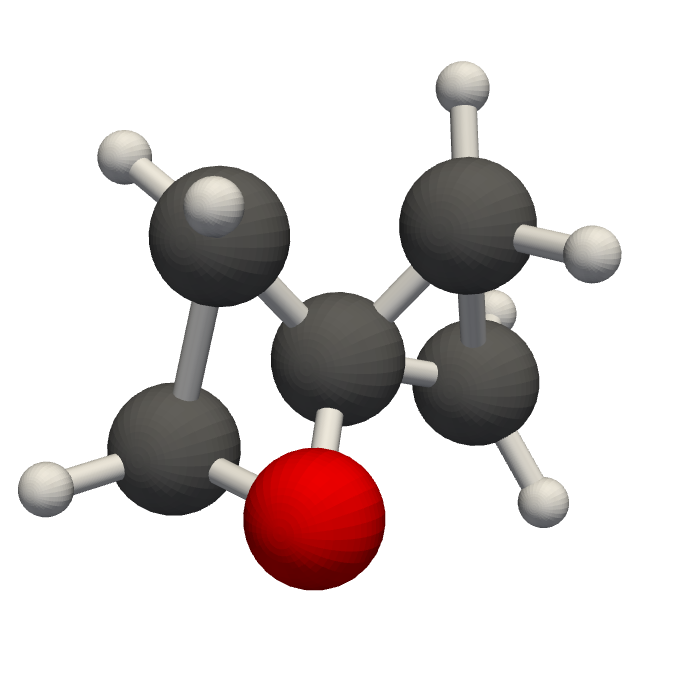}
      \includegraphics[width=.10\linewidth, keepaspectratio]{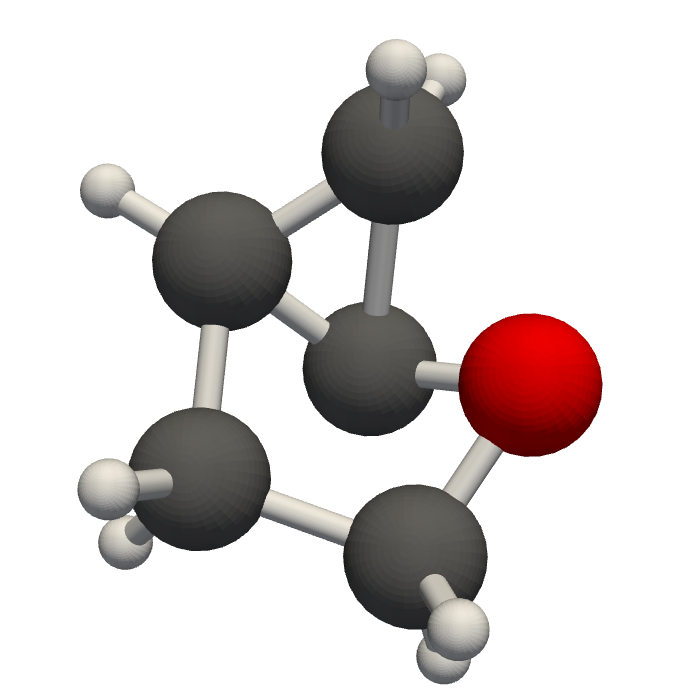}
      \newline
      \includegraphics[width=.10\linewidth, keepaspectratio]{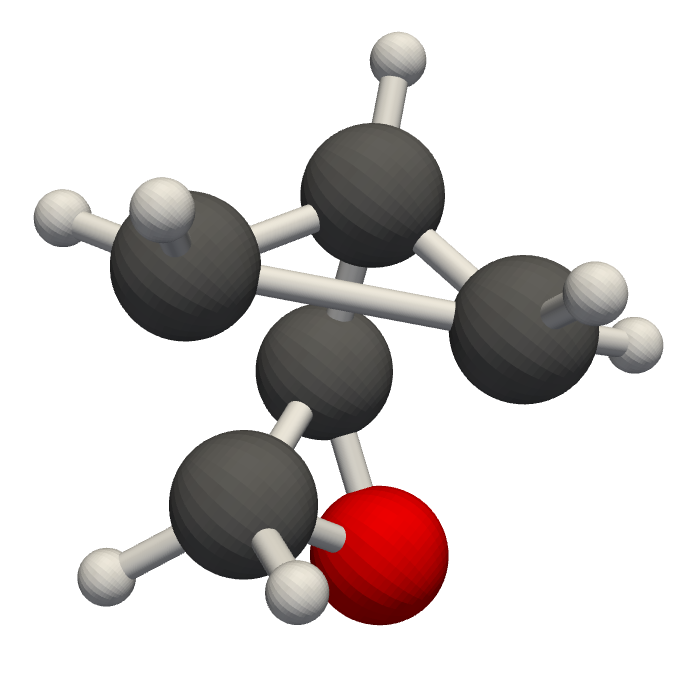}
      \includegraphics[width=.10\linewidth, keepaspectratio]{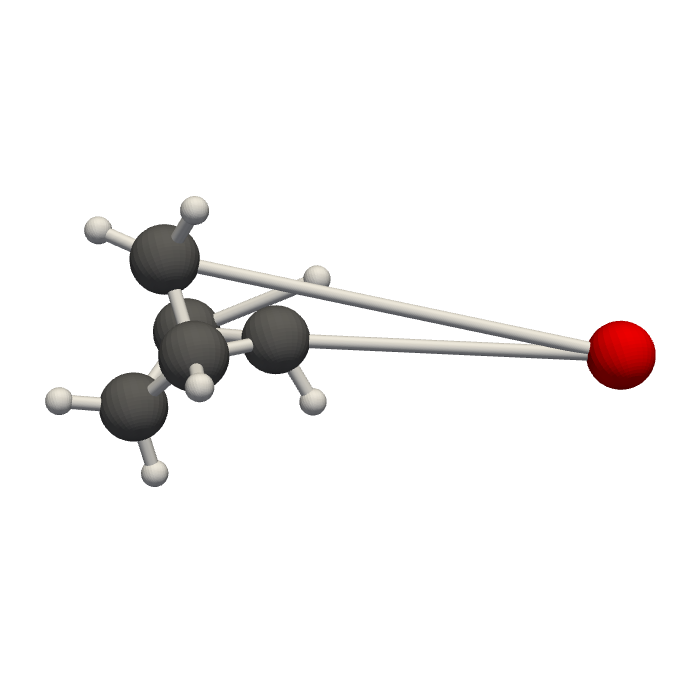}
      \includegraphics[width=.10\linewidth, keepaspectratio]{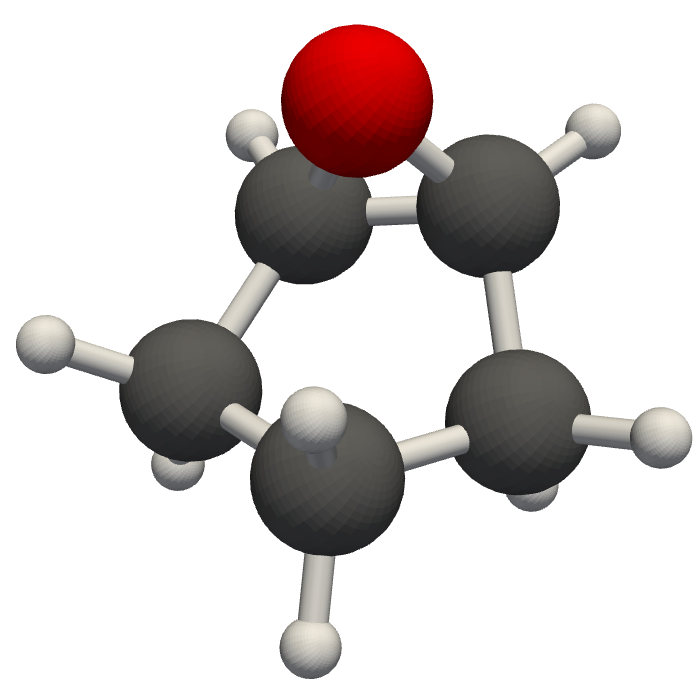}
      \includegraphics[width=.10\linewidth, keepaspectratio]{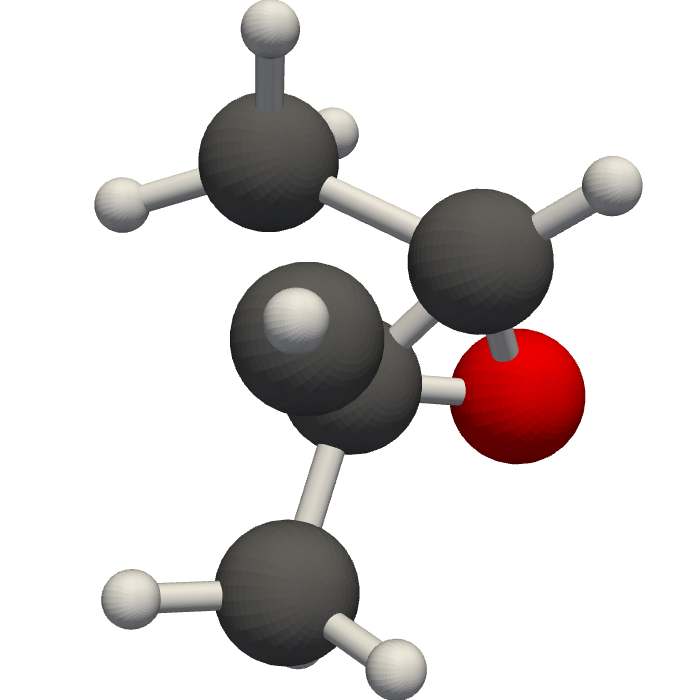}
      \includegraphics[width=.10\linewidth, keepaspectratio]{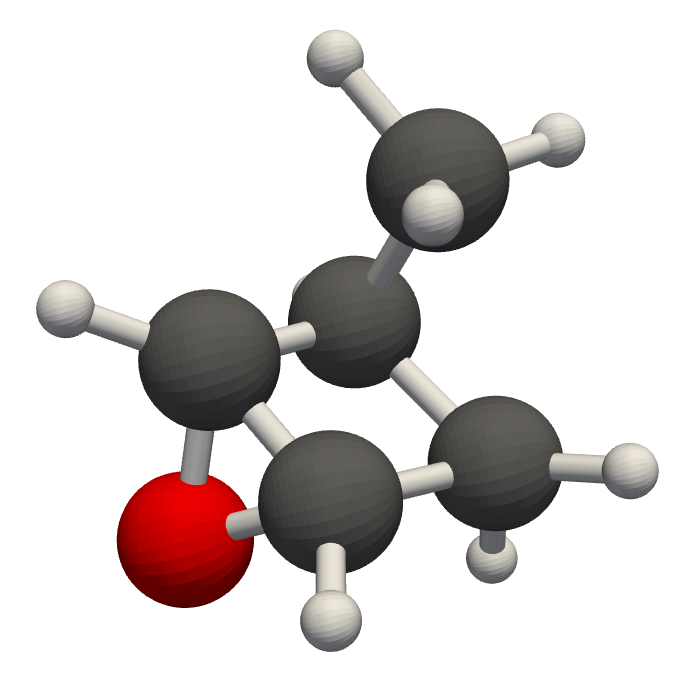}
      \includegraphics[width=.10\linewidth, keepaspectratio]{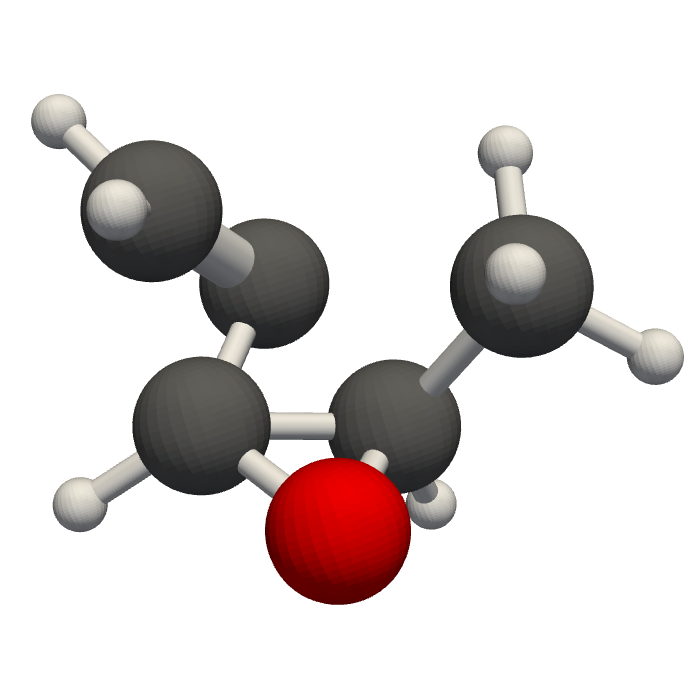}
      \includegraphics[width=.10\linewidth, keepaspectratio]{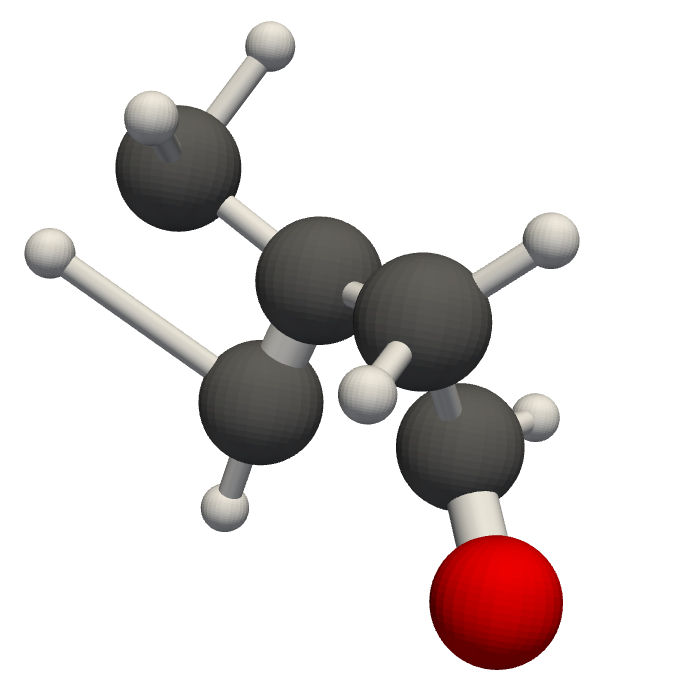}
      \includegraphics[width=.10\linewidth, keepaspectratio]{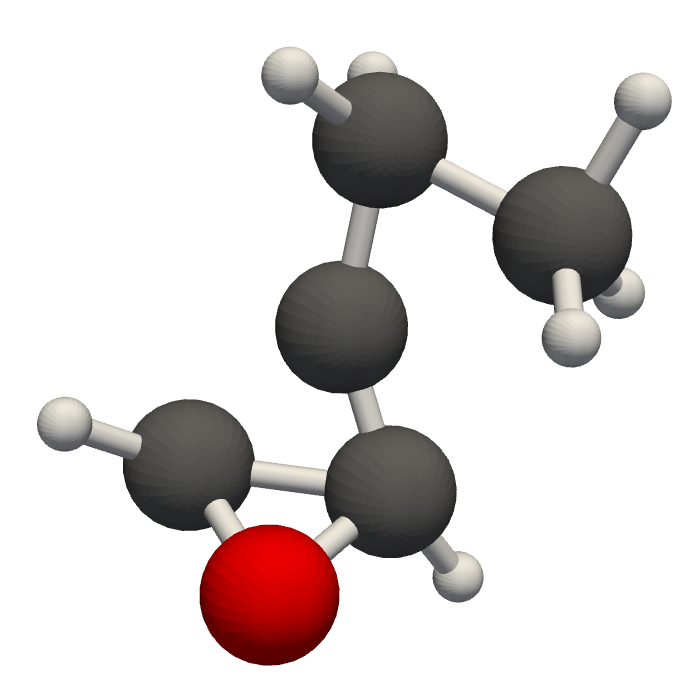}
      \newline
      \includegraphics[width=.10\linewidth, keepaspectratio]{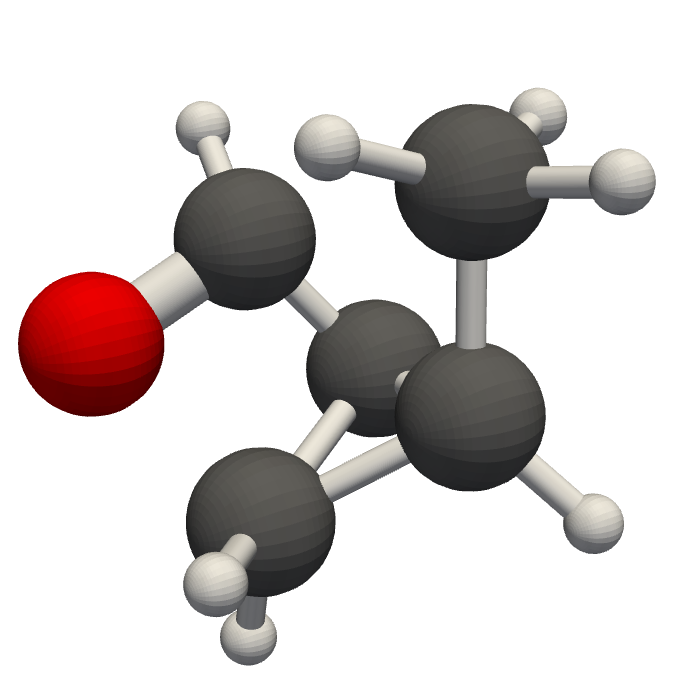}
      \includegraphics[width=.10\linewidth, keepaspectratio]{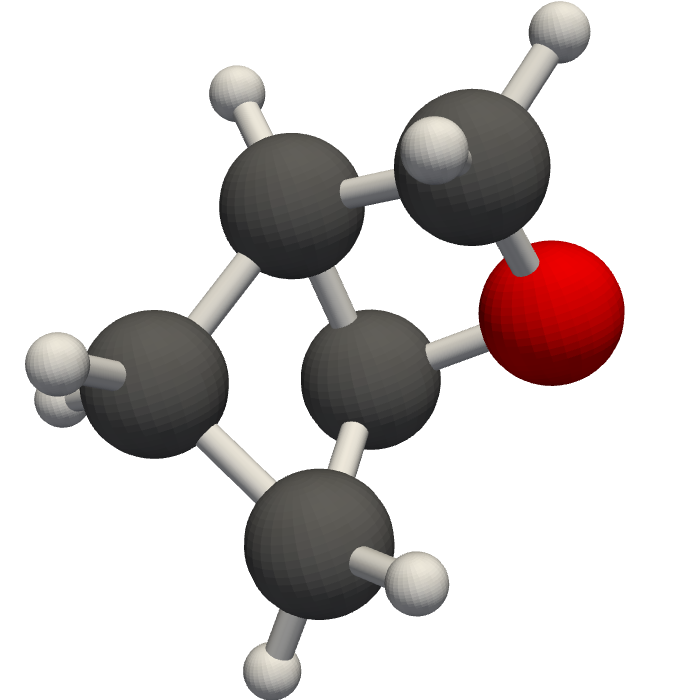}
      \includegraphics[width=.10\linewidth, keepaspectratio]{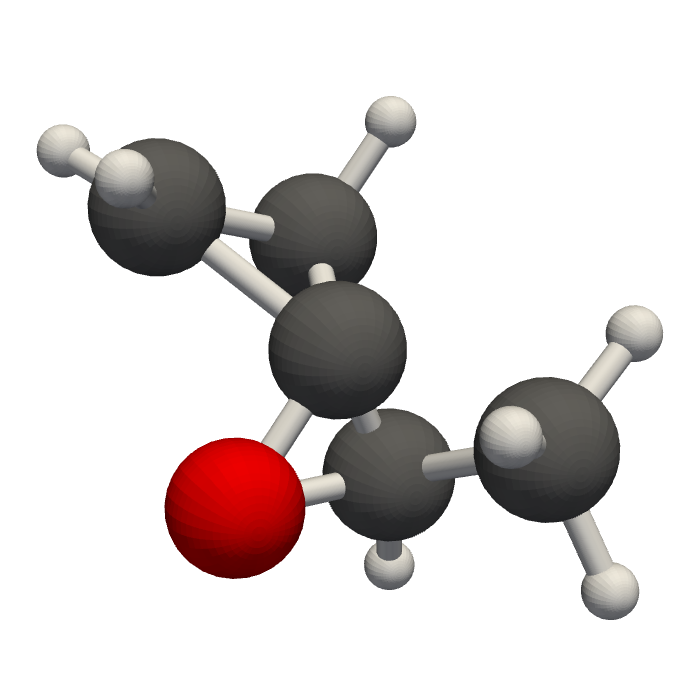}
      \includegraphics[width=.10\linewidth, keepaspectratio]{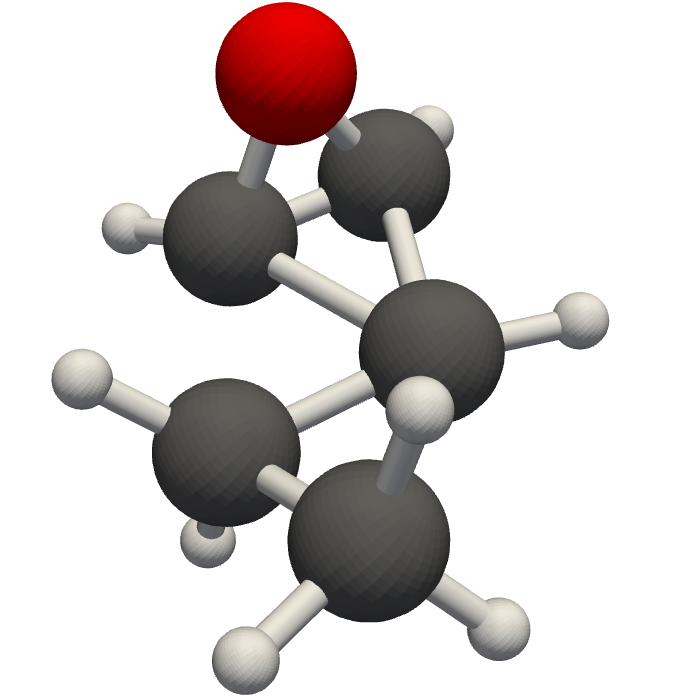}
      \includegraphics[width=.10\linewidth, keepaspectratio]{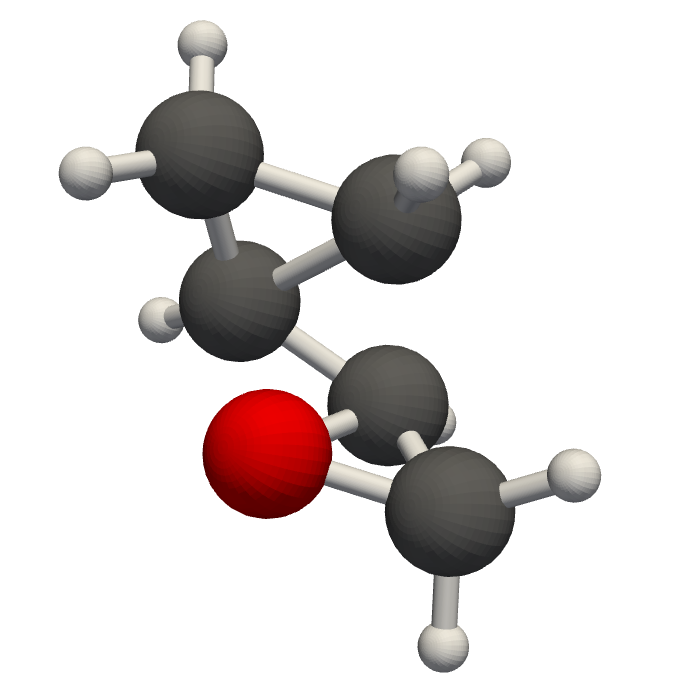}
      \includegraphics[width=.10\linewidth, keepaspectratio]{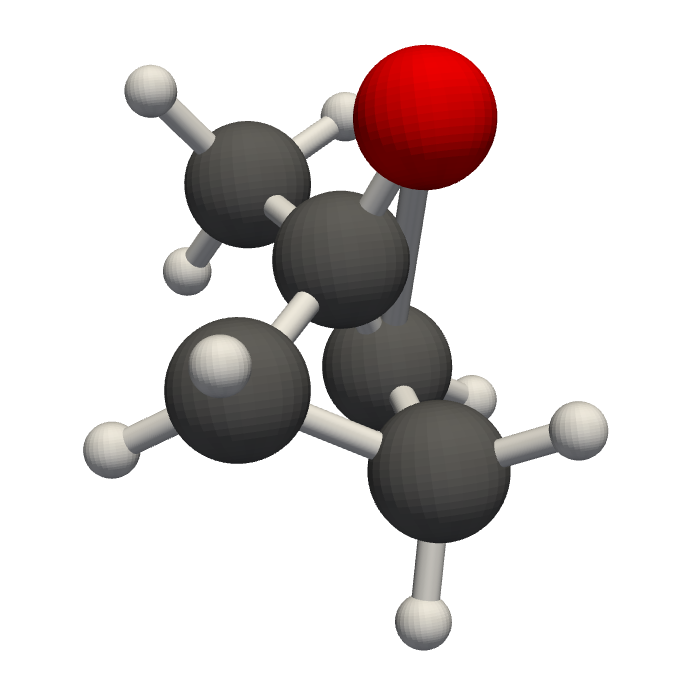}
      \includegraphics[width=.10\linewidth, keepaspectratio]{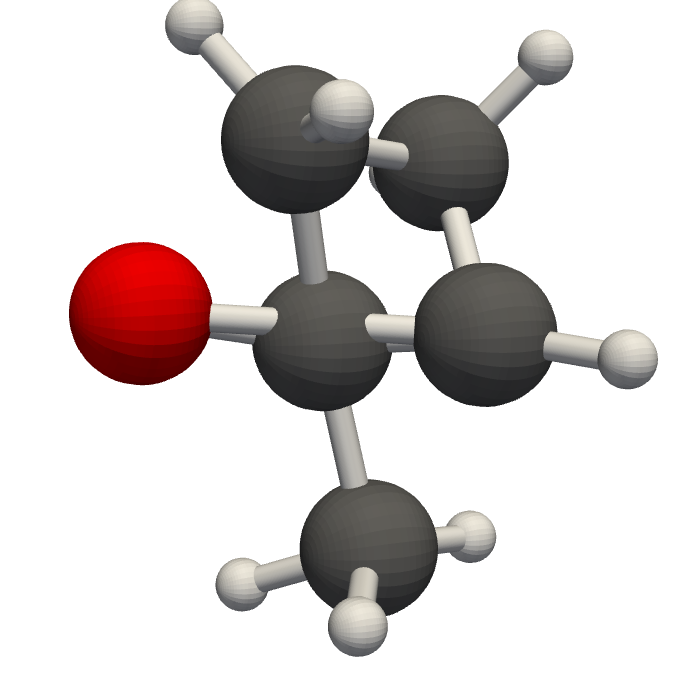}
      \includegraphics[width=.10\linewidth, keepaspectratio]{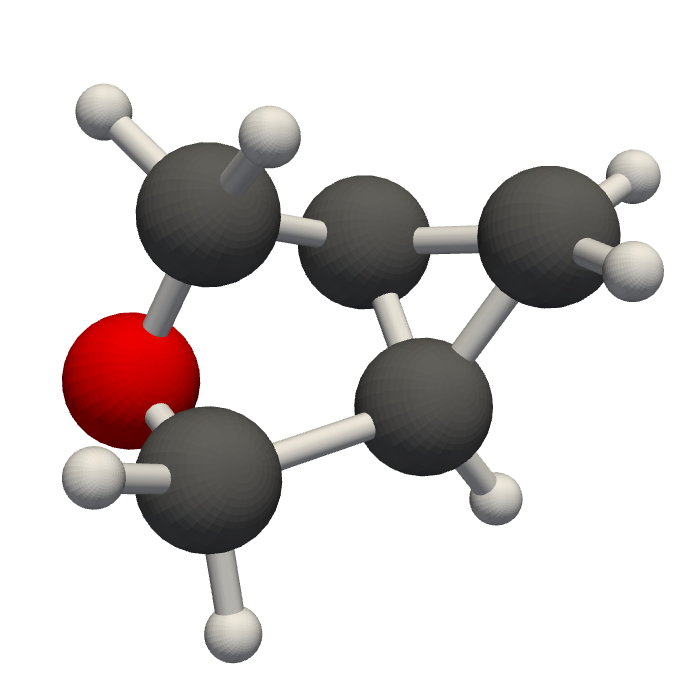}
      \newline
      \includegraphics[width=.10\linewidth, keepaspectratio]{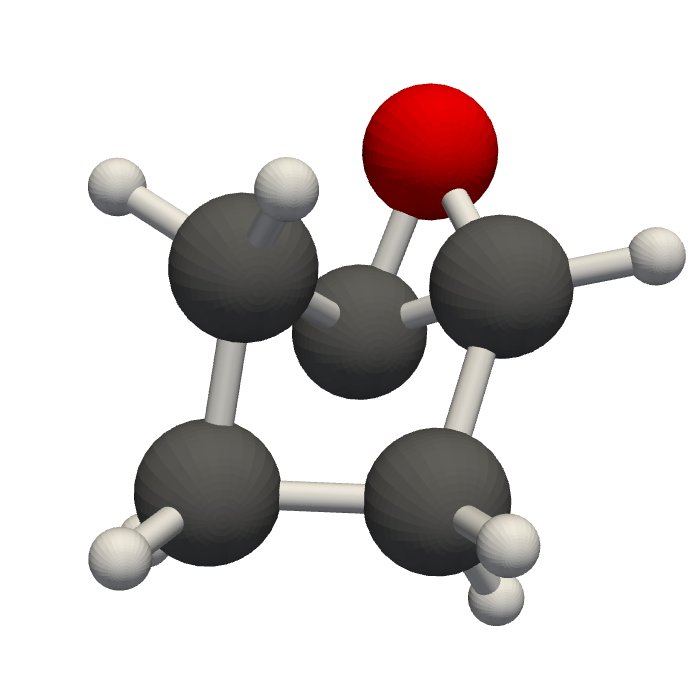}
      \includegraphics[width=.10\linewidth, keepaspectratio]{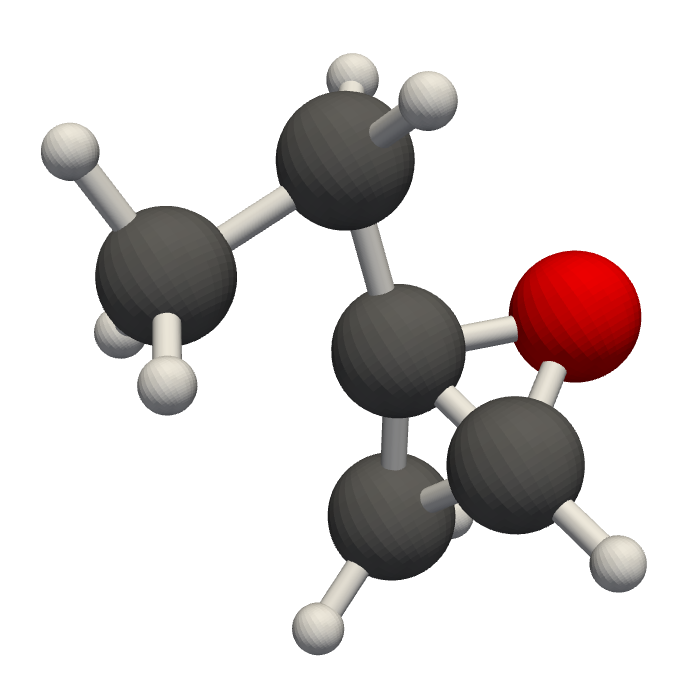}
      \includegraphics[width=.10\linewidth, keepaspectratio]{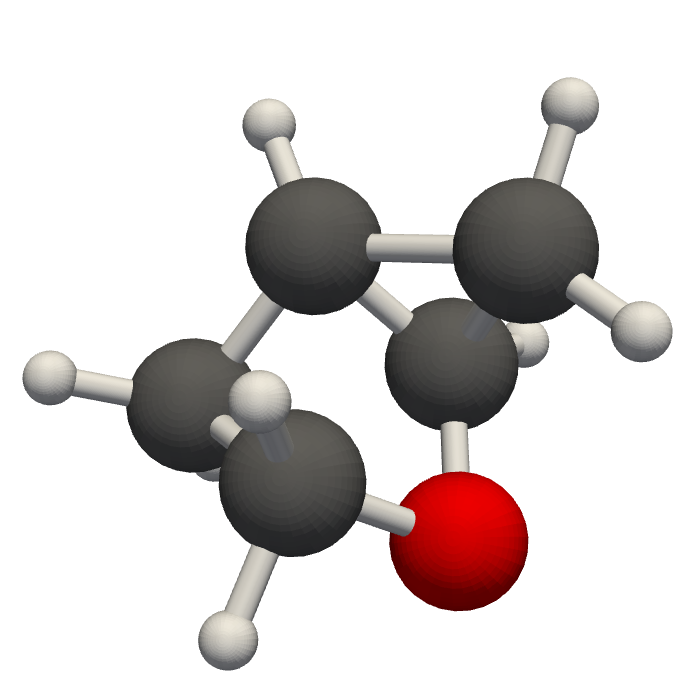}
      \includegraphics[width=.10\linewidth, keepaspectratio]{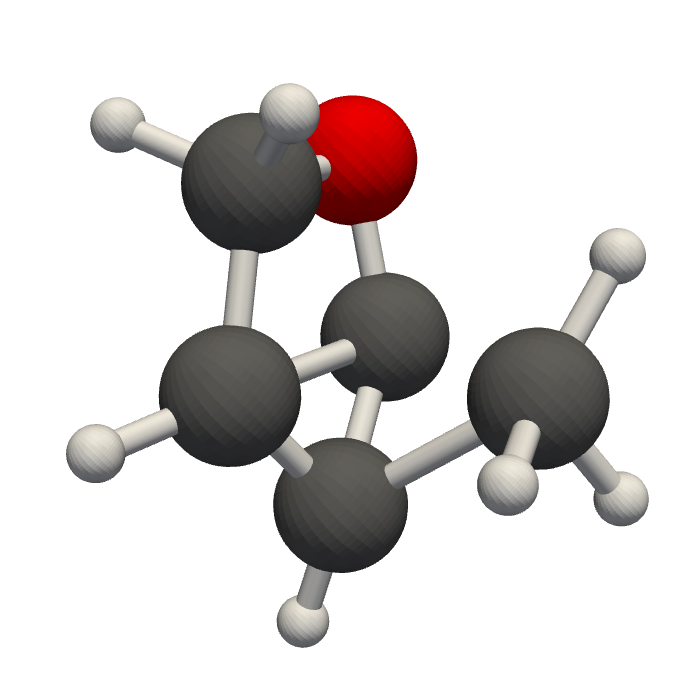}
      \includegraphics[width=.10\linewidth, keepaspectratio]{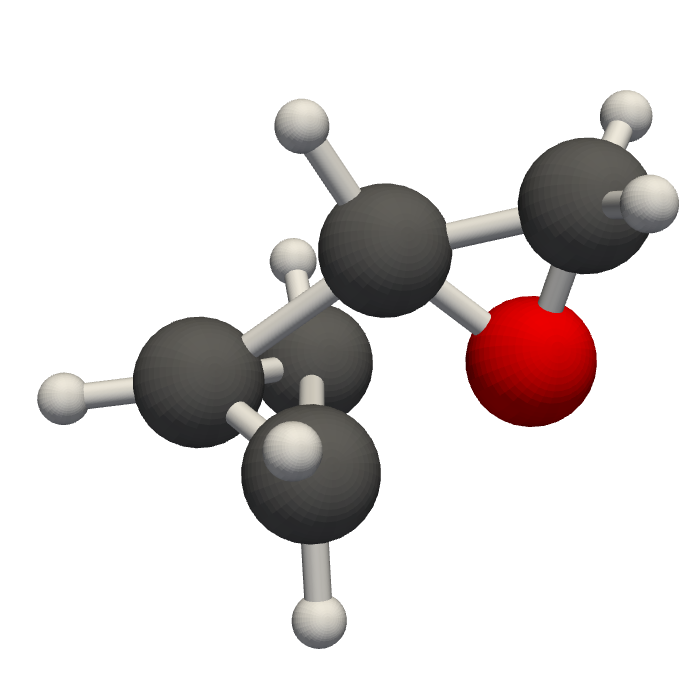}
      \includegraphics[width=.10\linewidth, keepaspectratio]{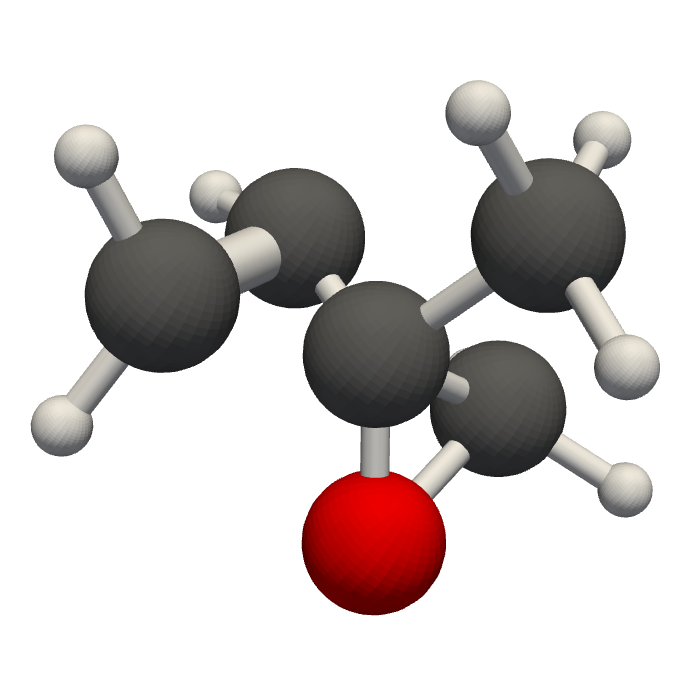}
      \includegraphics[width=.10\linewidth, keepaspectratio]{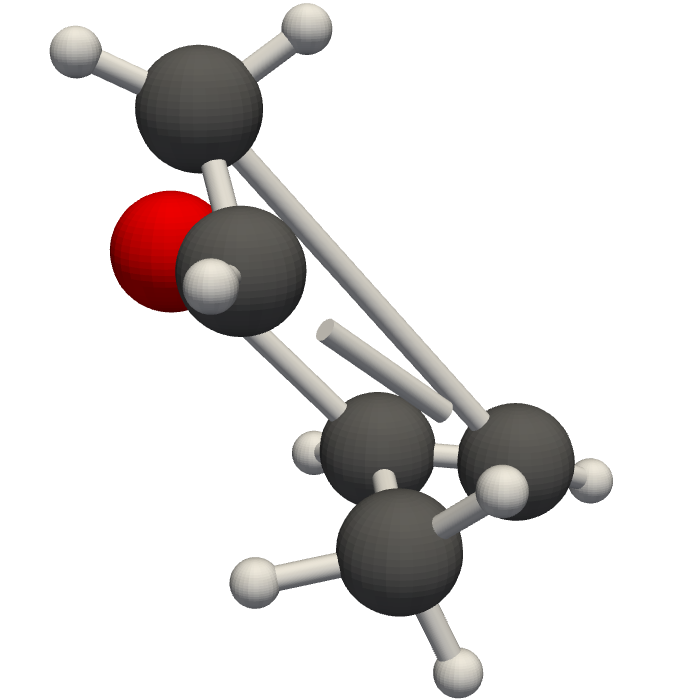}
      \includegraphics[width=.10\linewidth, keepaspectratio]{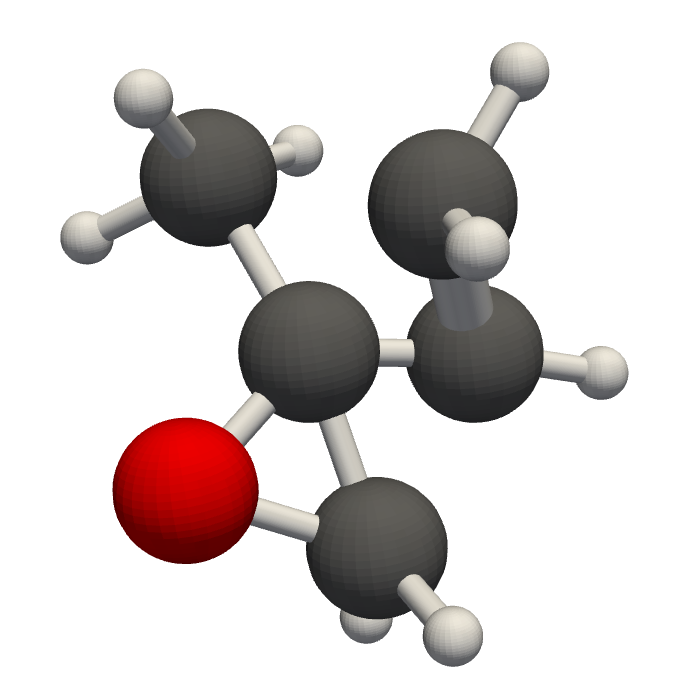}
      \newline
      \includegraphics[width=.10\linewidth, keepaspectratio]{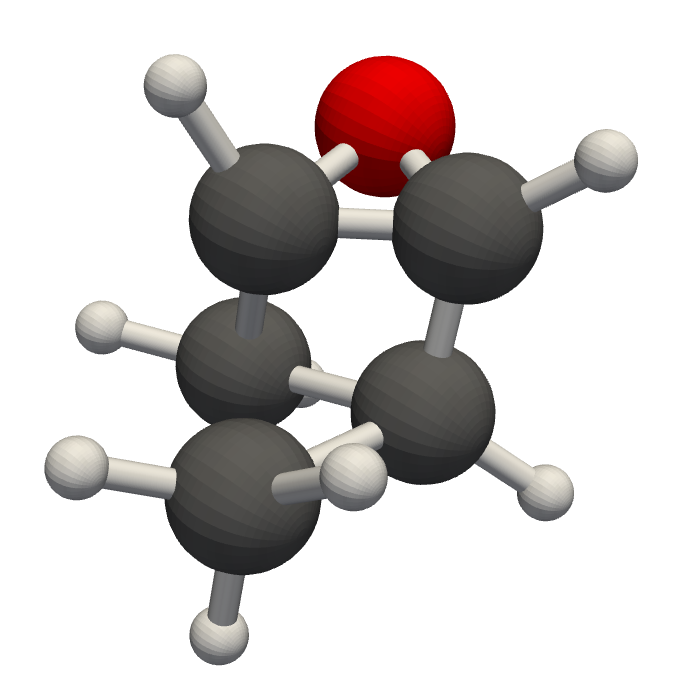}
      \includegraphics[width=.10\linewidth, keepaspectratio]{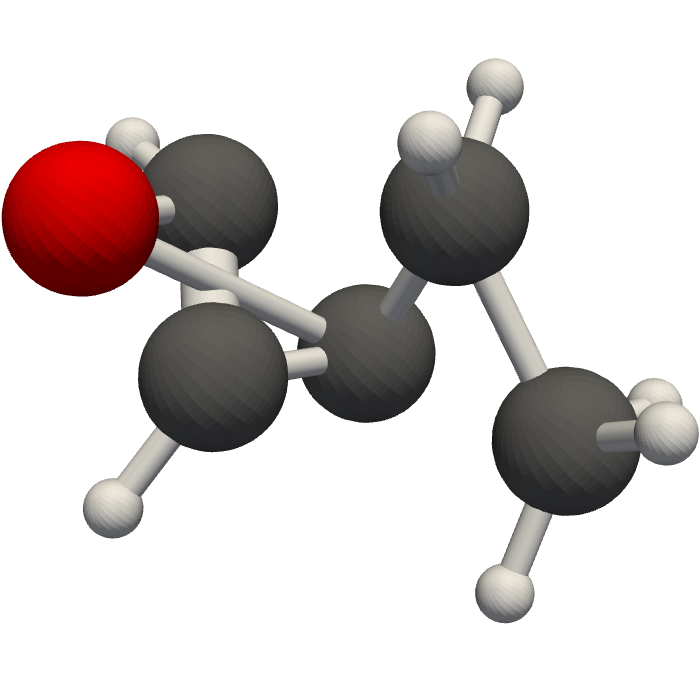}
      \includegraphics[width=.10\linewidth, keepaspectratio]{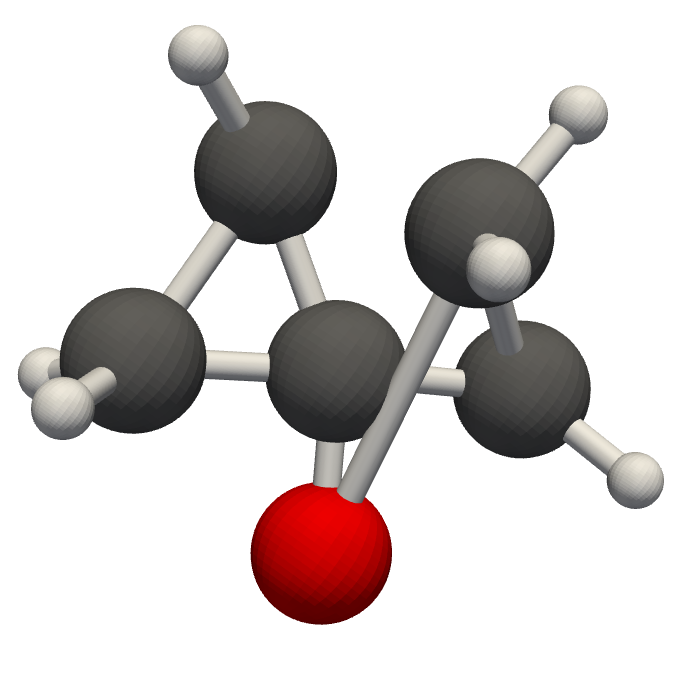}
      \includegraphics[width=.10\linewidth, keepaspectratio]{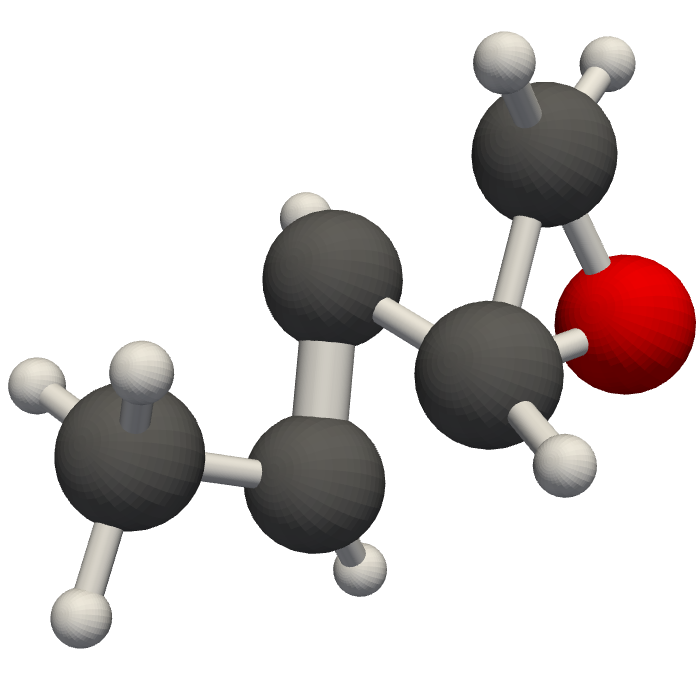}
      \includegraphics[width=.10\linewidth, keepaspectratio]{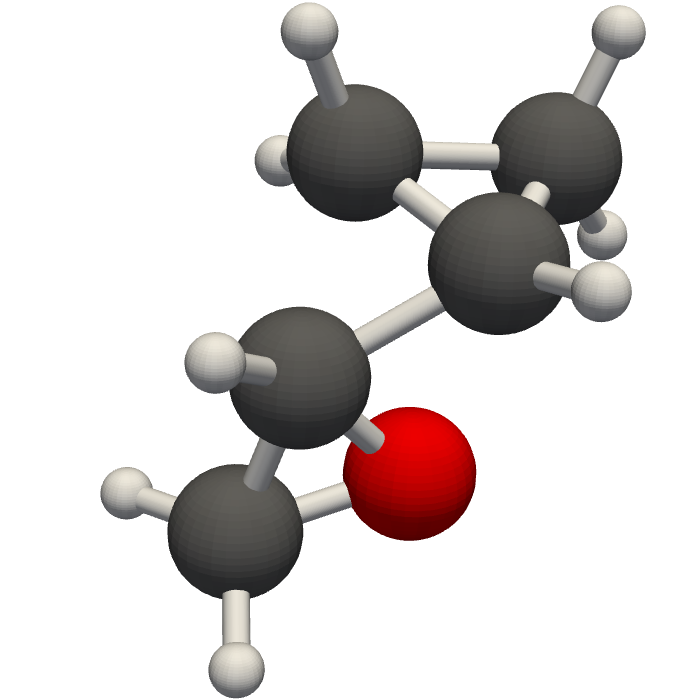}
      \includegraphics[width=.10\linewidth, keepaspectratio]{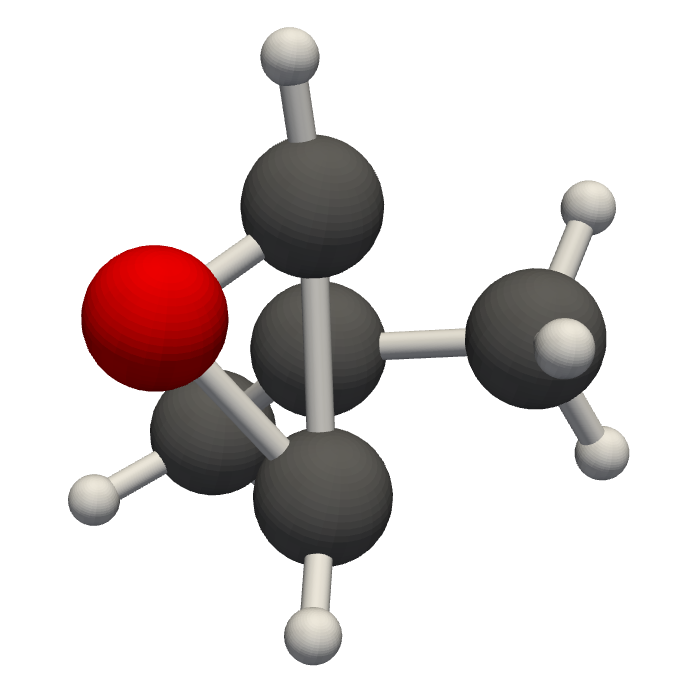}
      \includegraphics[width=.10\linewidth, keepaspectratio]{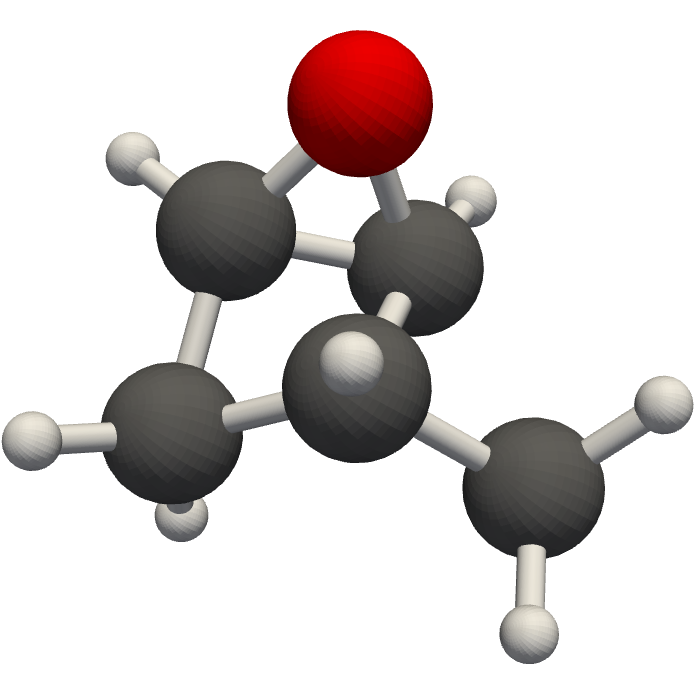}
      \includegraphics[width=.10\linewidth, keepaspectratio]{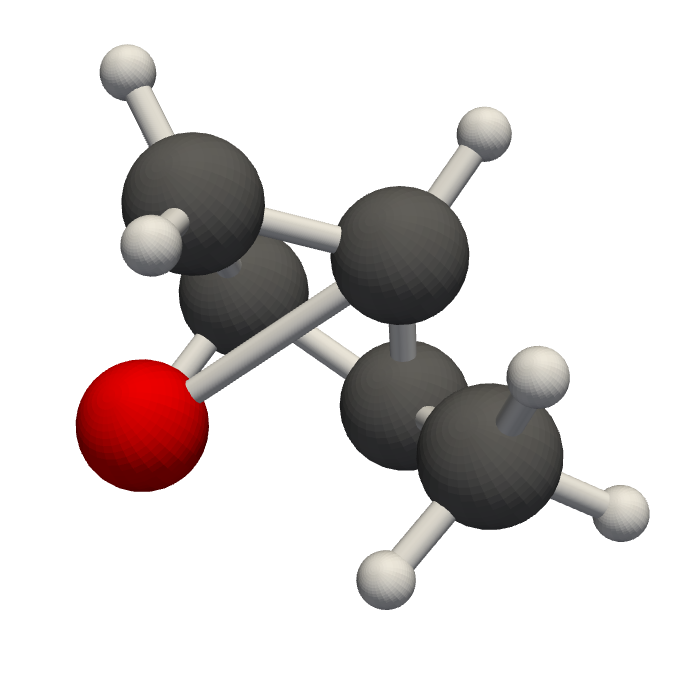}
      \newline
      \includegraphics[width=.10\linewidth, keepaspectratio]{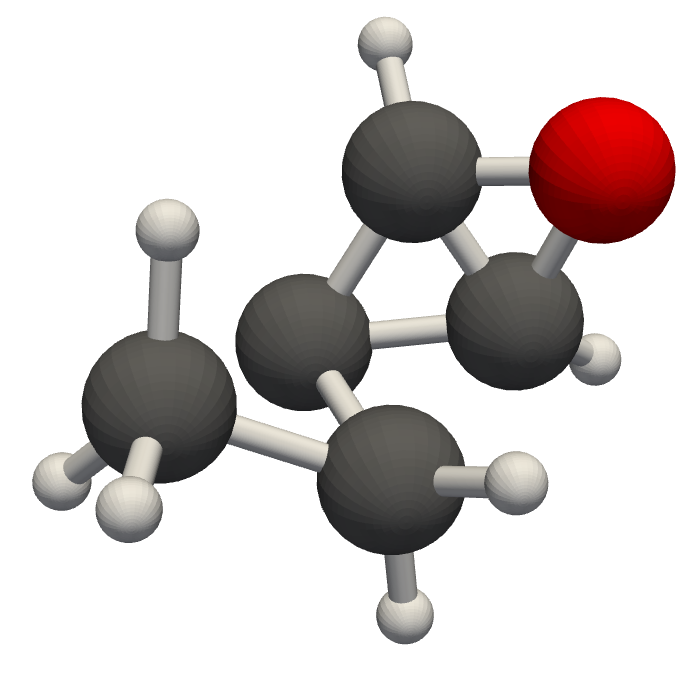}
      \includegraphics[width=.10\linewidth, keepaspectratio]{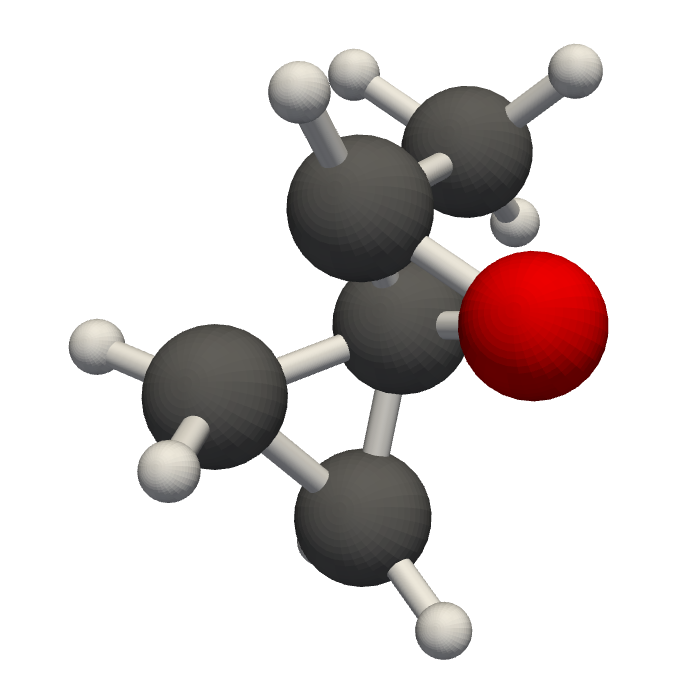}
      \includegraphics[width=.10\linewidth, keepaspectratio]{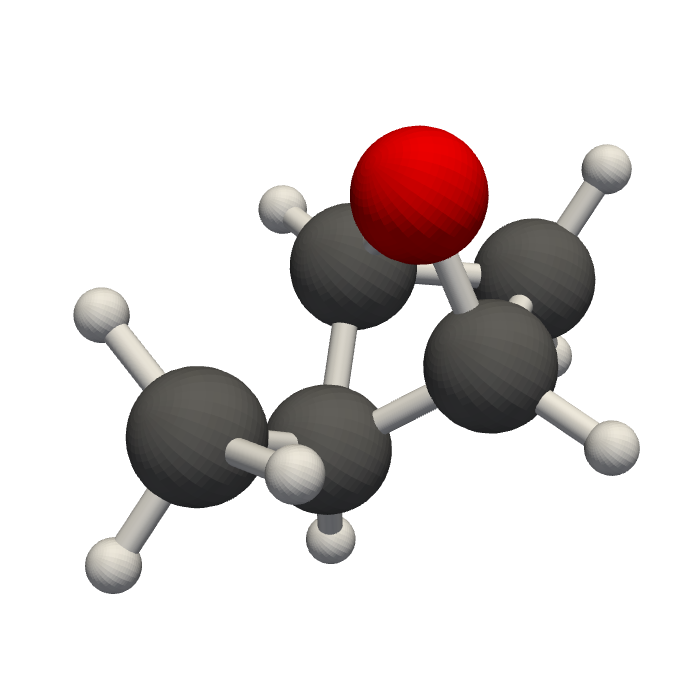}
      \includegraphics[width=.10\linewidth, keepaspectratio]{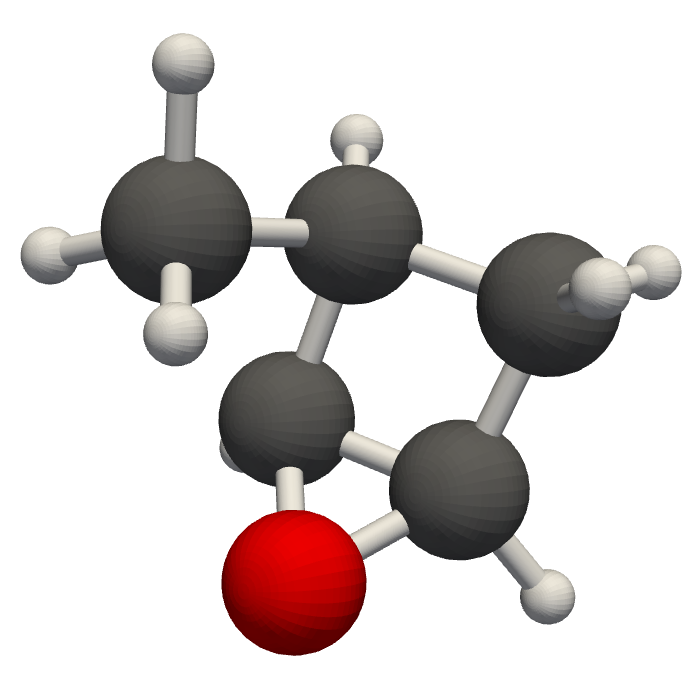}
      \includegraphics[width=.10\linewidth, keepaspectratio]{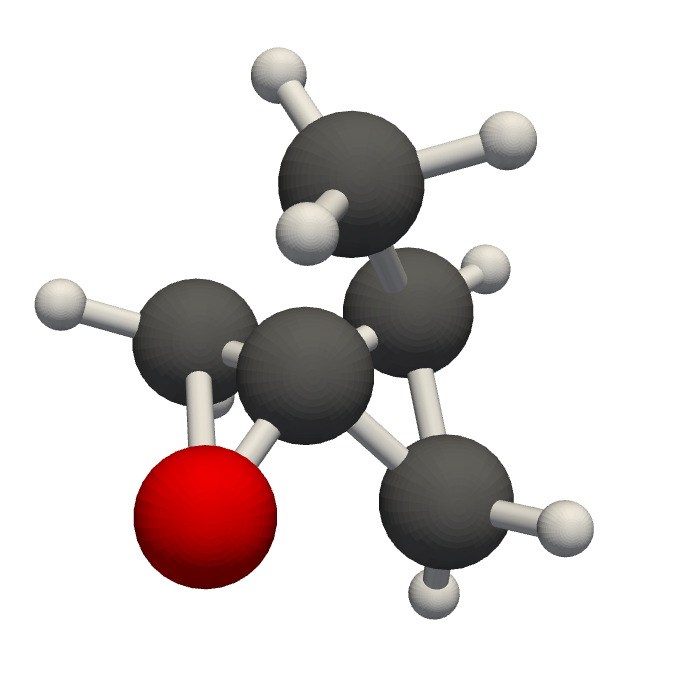}
      \includegraphics[width=.10\linewidth, keepaspectratio]{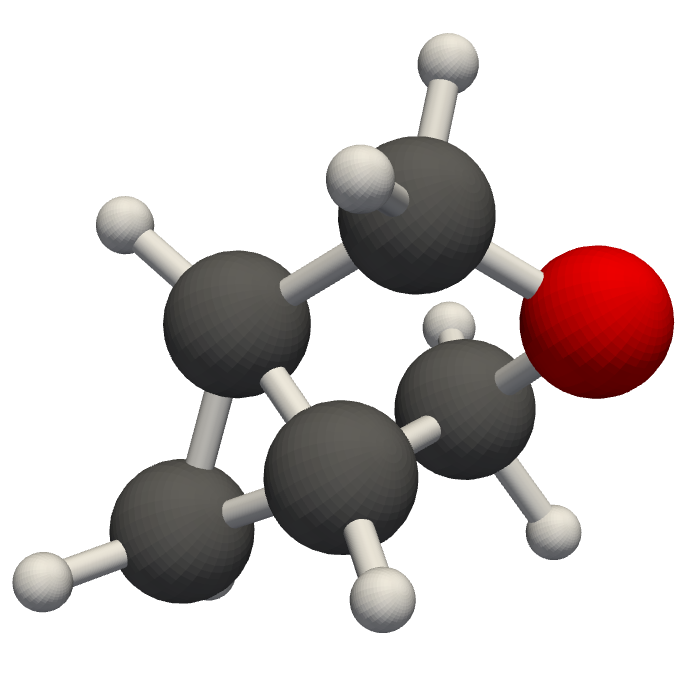}
      \includegraphics[width=.10\linewidth, keepaspectratio]{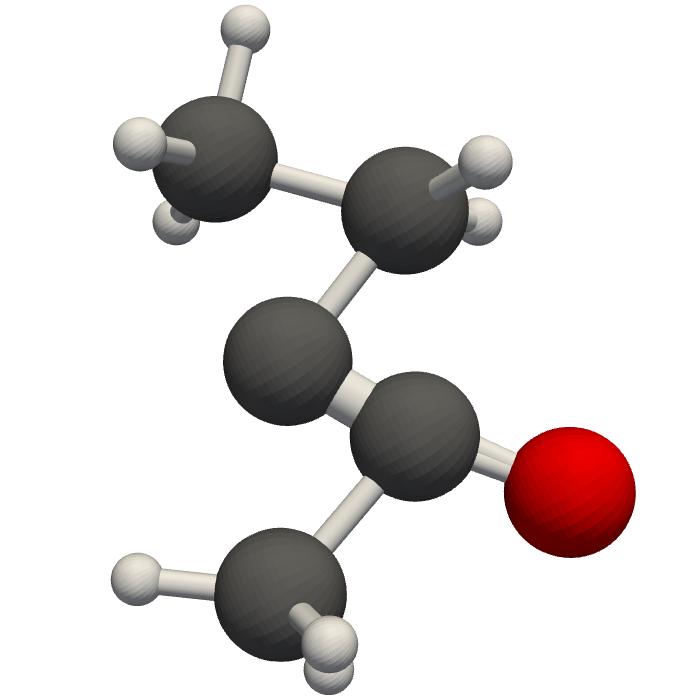}
      \includegraphics[width=.10\linewidth, keepaspectratio]{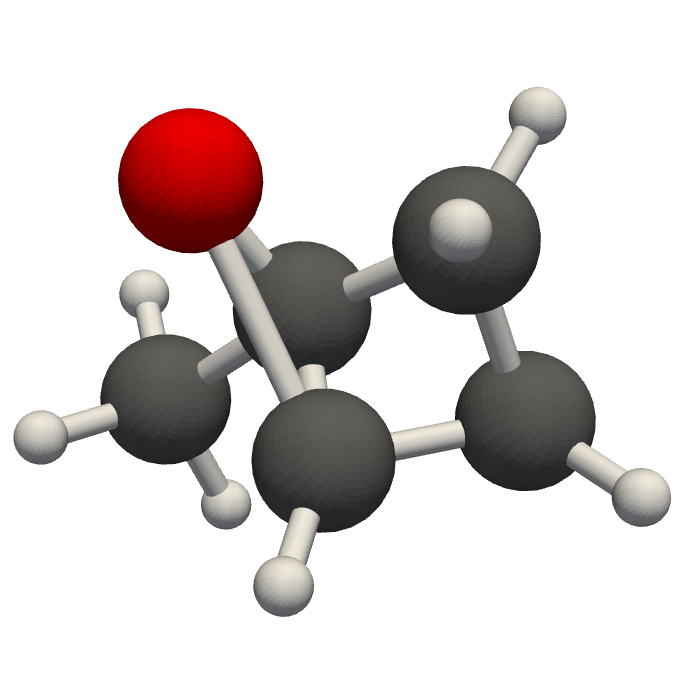}
      \newline
      \includegraphics[width=.10\linewidth, keepaspectratio]{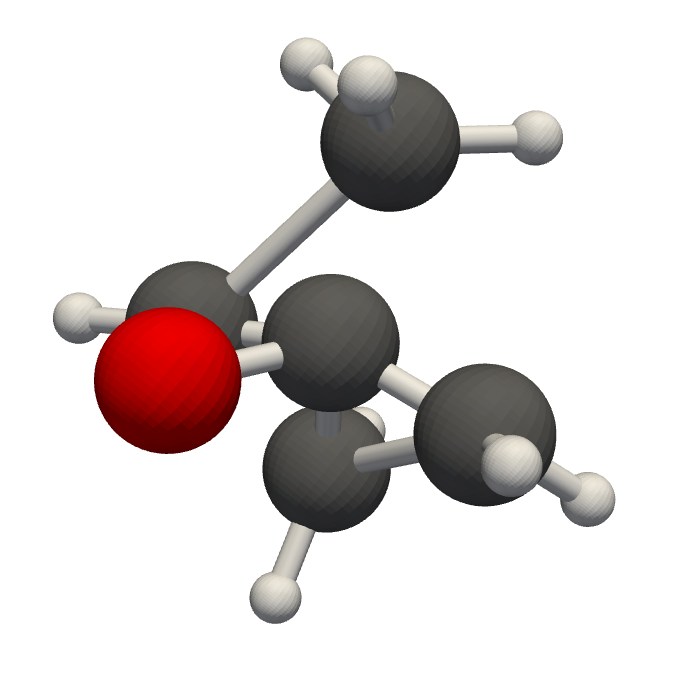}
      \includegraphics[width=.10\linewidth, keepaspectratio]{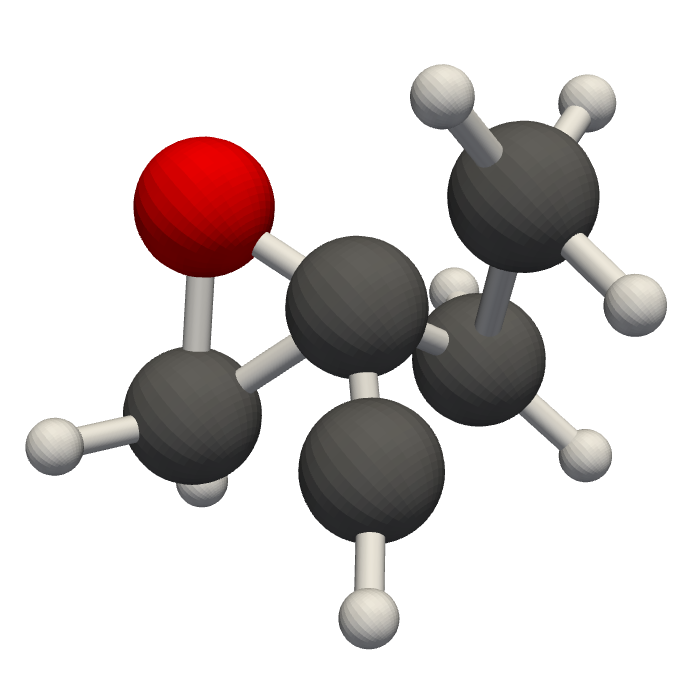}
      \includegraphics[width=.10\linewidth, keepaspectratio]{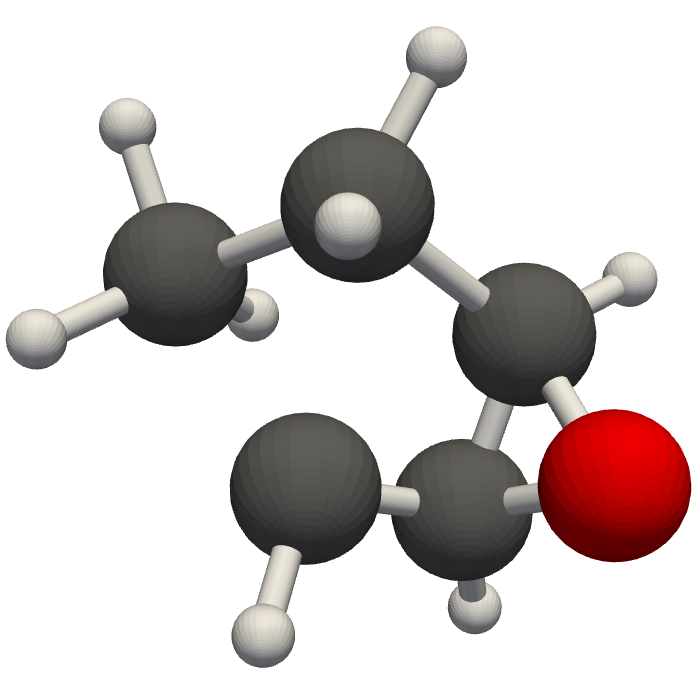}
      \includegraphics[width=.10\linewidth, keepaspectratio]{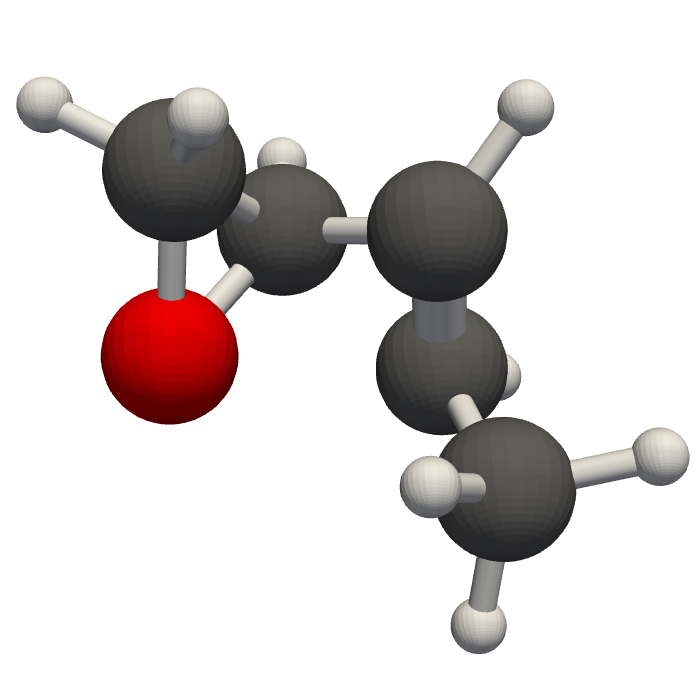}
      \includegraphics[width=.10\linewidth, keepaspectratio]{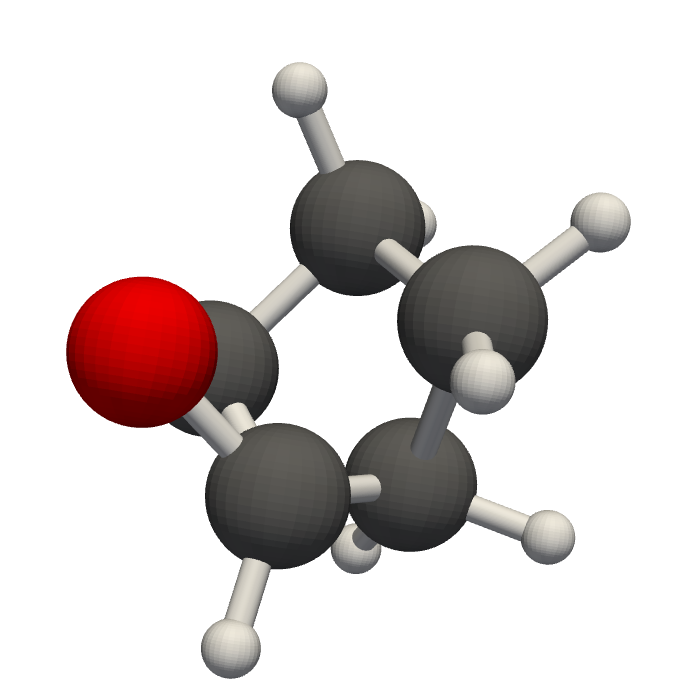}
      \includegraphics[width=.10\linewidth, keepaspectratio]{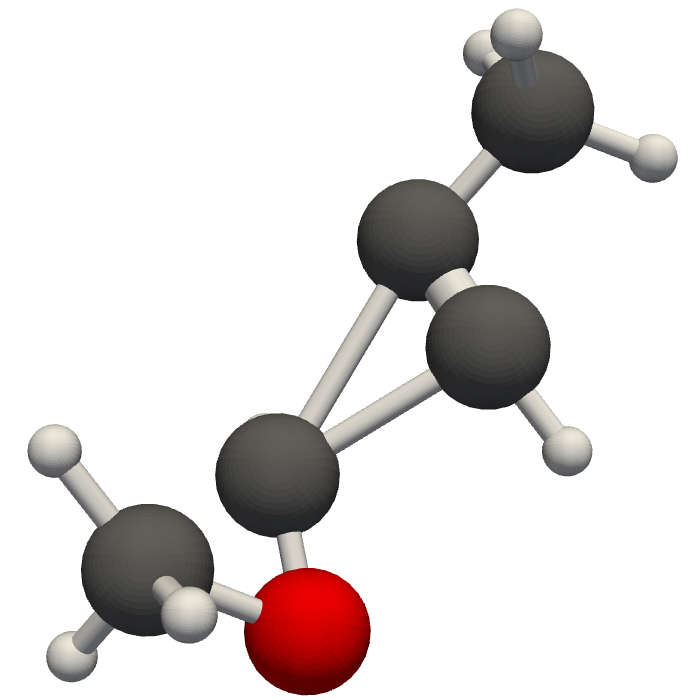}
      \includegraphics[width=.10\linewidth, keepaspectratio]{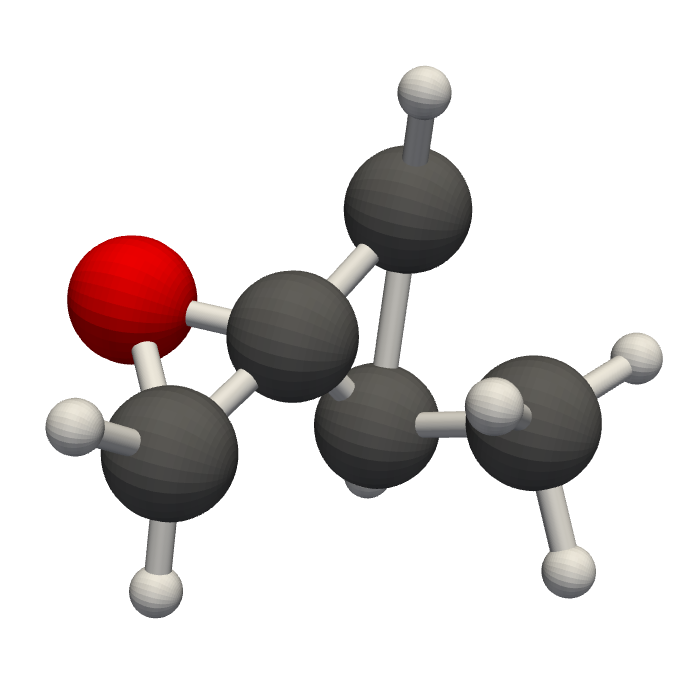}
      \includegraphics[width=.10\linewidth, keepaspectratio]{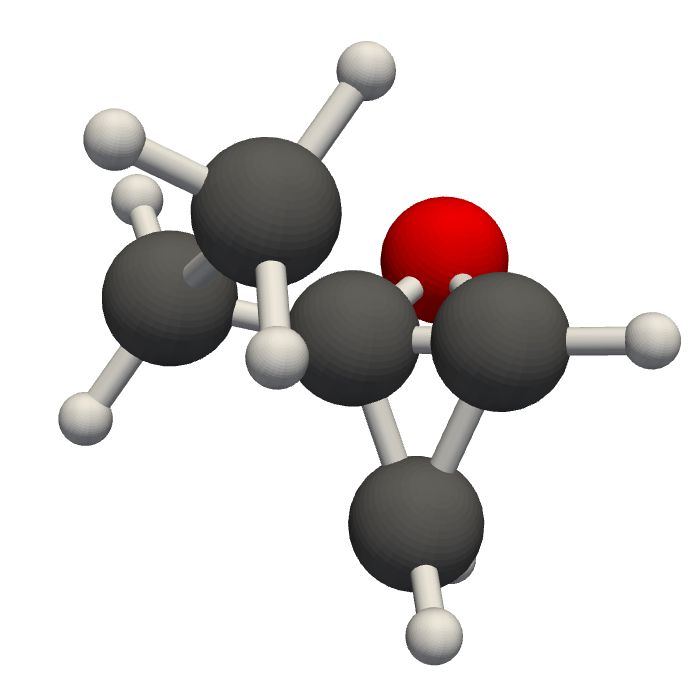}
      \newline
  \caption{Full batch of 64 molecules sampled from the trained EBM. Note that this number is larger then the total number of molecules trained on. }
  \label{fig:QM9-64}
\end{figure}
In \Cref{fig:QM9-64} we show a full batch of sampled molecules. We find that in addition to the ones presented in the main body of the paper, most other molecules are also anecdotally correct. However, we do also find some weird structures. These can be roughly categorized in two classes. The first class contains faulty molecules that result from the sampling procedure. If we label the rows by the letter A till H and the columns by the numbers 1 till 8, the molecule in C2 is one such example. The second class contains all molecules that show weird bonds. Molecule D7 and E3 are clear examples of this.

\subsection{Joint and Conditional Generation for FashionMNIST}
\label{app:fashion}

We use the FashionMNIST dataset \citep{xiao2017/online} for this experiment with no data augmentation \ie all the models are trained with images in their natural orientation. The test set is however preprocessed to contain images that have been randomly rotated using the $C_4$ symmetry group \ie rotation angles from the set $\{0^{\circ}, 90^{\circ}, 180^{\circ}, 270^{\circ}\}$. The equivariant energy based model is created using a $C_4$-steerable CNNs \citep{e2cnn} consisting for eight $C_4$-steerable convolutional layers followed by  a group pooling layer and a fully connected layer. The regular EBM consists of the same architecture wherein the individual layers are not equivariant \ie the steerable CNNs are replaced by normal CNN layers. We train all the models using the joint-energy model loss described in \Cref{eq:loss-JEM}. We train the models using both equivariant SVGD and vanilla SVGD resulting in three combinations of models, namely: Equivariant EBM trained with equivariant SVGD, equivariant EBM trained with vanilla SVGD and, regular EBM trained with vanilla SVGD. 

We trained each model for 300 epochs using the Adam optimizer \citep{kingma2014adam} using a batch size of 64. For equivariant SVGD, we evolve the dynamics for 1,000 time-steps per mini-batch but due to slow convergence we had to increase this to 3,000 time-steps for vanilla SVGD. We used an RBF kernel for SVGD with a bandwidth of 0.1 for vanilla SVGD and 0.005 for equivariant SVGD with a step size of 0.08. We also used persistence with a reset probability of 0.1.

\end{document}